\documentclass[numbers,webpdf,imanum]{ima-authoring-template}
\pdfoutput=1 

\usepackage{graphicx}%

\usepackage{tabularx}
\usepackage{multirow}
\usepackage{float}
\floatstyle{plaintop}
\restylefloat{table}

\usepackage{amsmath}
\usepackage{amssymb}
\usepackage{amsfonts}
\usepackage{amsthm}
\usepackage{mathrsfs}
\usepackage[title]{appendix}
\usepackage{xcolor}
\usepackage{textcomp}
\usepackage{manyfoot}
\usepackage{booktabs}

\usepackage{algorithm}%
\usepackage{algorithmicx}%
\usepackage{algpseudocode}%

\usepackage{listings}%

\usepackage[T1]{fontenc}

\usepackage{amsopn}
\DeclareMathOperator{\esssup}{{\rm ess\,sup}}

\usepackage{enumitem}
\setlist[enumerate]{leftmargin=0.3in}
\setlist[itemize]{leftmargin=0.3in}

\theoremstyle{thmstyleone}%
\newtheorem{theorem}{Theorem}[section]

\newtheorem{proposition}[theorem]{Proposition}%

\newtheorem{lemma}[theorem]{Lemma}%
\newtheorem{corollary}[theorem]{Corollary}%
\newtheorem{definition}[theorem]{Definition}%
\newtheorem{remark}[theorem]{Remark}%

\numberwithin{equation}{section}

\usepackage{braket}

\usepackage{array}

\usepackage{geometry}

\allowdisplaybreaks[4]

\begin{document}
\begin{sloppypar}

\DOI{DOI HERE}
\copyrightyear{2026}
\vol{00}
\pubyear{2026}
\access{Advance Access Publication Date: Day Month Year}
\appnotes{Paper}
\copyrightstatement{Published by Oxford University Press on behalf of the Institute of Mathematics and its Applications. All rights reserved.}
\firstpage{1}

\title[Deep-layer limit and stability of FBS-network(II)]{
	Deep-layer limit and stability analysis of the basic forward-backward-splitting induced network (II): learning problems
}

\author{\textsc{Xuan Lin}
	\address{
		\orgname{China Academy of Aerospace System and Innovation}, 
		\orgaddress{
			\postcode{100088},
			\state{Beijing}, 
			\country{China}
		}
	}
	\address{
		\orgdiv{School of Mathematical Sciences}, 
		\orgname{Nankai University}, 
		\orgaddress{
			\street{94 Weijin Road}, \postcode{300071}, \state{Tianjin}, \country{China}
		}
	}  
}
\author{\textsc{Chunlin Wu*}
	\address{
		\orgdiv{School of Mathematical Sciences}, 
		\orgname{Nankai University}, 
		\orgaddress{
			\street{94 Weijin Road}, \postcode{300071}, \state{Tianjin}, \country{China}
		}
	}
}

\authormark{Xuan Lin, Chunlin Wu}

\corresp[*]{Corresponding author. \href{email:1120200026@mail.nankai.edu.cn}{wucl@nankai.edu.cn}}

\received{3}{5}{2026}
\revised{Date}{0}{2026}
\accepted{Date}{0}{2026}

\abstract{
	Deep unfolding neural networks derived from iterative optimization schemes and numerical ordinary/partial differential equations (ODEs/PDEs) have attracted much attention in data science over the last decade.
	Therein, numerous important network architectures were constructed from the basic forward-backward-splitting (FBS) algorithm.
	In this paper, we continue our research on the most basic FBS-induced network, an architecture unrolled from the original FBS algorithm by incorporating direct parameter relaxations. 
	Following the difference/differential inclusion formulations in our previous forward system analyses, we here consider some theoretical aspects of corresponding learning problems.
	Under some mild assumptions, we establish a general convergence property of the training problem of the basic FBS-induced network to the learning problem of the deep-layer limit system, implying a $\Gamma$-convergence argument showing that any cluster point of the optimal learning parameters for the network is a solution to the learning problem of the deep-layer limit system.
	A qualitative analysis of perturbation stabilities of these learning problems is also presented. 
	A simple numerical experiment is conducted to validate our main general convergence result.
}


\maketitle

\section{Introduction}\label{sec-introduction}

In this paper, we continue our previous study in \cite{lin2025deep} on deep-layer limit analysis of the so-called basic FBS-network.
Here, "basic" means that the neural
network is unrolled from the original FBS algorithm with direct parameter relaxation.
Using our previous dynamical inclusion modeling of forward systems in \cite{lin2025deep}, we here consider the deep-layer limit of the learning problem and related stability aspects.

Indeed, there has been growing interest in deep-layer limit modeling and analysis for deep neural networks (DNNs) in data science in recent years, which helps to interpret the behaviors and mechanisms in DNNs.
In particular, the well-known ResNets can be characterized as an ODE \cite{weinan2017proposal, thorpe2023deep, chang2018reversible} or an ordinary differential inclusion (ODI) \cite{zhang2020forward} in the deep-layer limit.
Recently, Transformer architecture without layer normalization and its variants have been interpreted as a (coupled) ODE systems \cite{lu2020understanding, liang2026deep}. 
Moreover, densely connected network (DenseNet) with a general nonlocal framework has been mathematically modeled as a nonlinear integral equation in the deep-layer limit \cite{huang2025mathematical}. 
For networks induced from numerical ODEs \cite{gear1984ode, butcher1987numerical} or time-dependent PDEs \cite{larsson2003partial, de2015numerical}, such as trainable reaction diffusion \cite{chen2016trainable}, PDE-Net and its variant \cite{long2018pde, long2019pde}, stable/reversible networks \cite{haber2017stable, chang2018reversible, haber2019imexnet, ruthotto2020deep}, LM-ResNet \cite{lu2018beyond}, neural ODE \cite{chen2018neural}, networks with multigrid structures \cite{he2019mgnet, tai2024pottsmgnet}, and more from a continuous viewpoint \cite{weinan2020machine}, their deep-layer limits are naturally the systems from which they originated.
These studies help to indicate some properties such as stabilities of network architectures \cite{haber2017stable, chang2018reversible, ma2019model, lu2018beyond, haber2019imexnet, zhang2020forward, ruthotto2020deep}, or design some novel and interesting architectures \cite{haber2017stable, lu2018beyond, chang2018reversible, he2019mgnet, zhang2020forward, weinan2020machine, tai2024pottsmgnet}.

However, for those networks \cite{gregor2010learning, liu2019alista, sun2016deep, adler2018learned, kobler2017variational, hammernik2018learning} unrolled from iterative optimization algorithms \cite{beck2017first, glowinski2017splitting, glowinski1989augmented} which play important roles in imaging-related inverse problems, there is so far little study on deep-layer limit analysis \cite{lin2025deep, huang2024dynamical}.
Therein, \cite{huang2024dynamical} studied a dynamical system modeling based on ODE for learned primal-dual method, while \cite{lin2025deep} proposed a dynamical inclusion modeling for the most basic FBS-network.
As the FBS algorithm and its variants have induced several related significant networks like LISTA \cite{gregor2010learning}, ALISTA \cite{liu2019alista}, ISTA-Net \cite{zhang2018ista}, FISTA-Net \cite{xiang2021fista}, variational networks \cite{kobler2017variational}, PFBS-IR and PFBS-AIR \cite{ding2020low}, we now focus on this algorithm and continue our previous study on the most basic induced network structure in \cite{lin2025deep}. 
The main contributions of this paper can be summarized as follows
\begin{itemize}
	\item[$\bullet$ ] We show the existence of solutions to the learning problems of the basic FBS-network and its deep-layer limit system, under certain conditions.
	
	\item[$\bullet$ ] We establish a general convergence result from the learning problem of the basic FBS-network to that of its deep-layer limit system, indicating $\Gamma$-convergence and related consequences.
	
	\item[$\bullet$ ] We present perturbation stabilities of the learning problems for the basic FBS-network and its deep-layer limit system, in terms of the initial state, the observed data, and the label. 
\end{itemize}
These results demonstrate a deep consistency between the learning problems of the basic FBS-network and its deep-layer limit system, as well as their stabilities.
Our analysis method admits a simplified variant capable of analyzing the deep-layer limit of other networks unrolled from the FBS type of algorithms, such as LISTA and ALISTA.

The remainder of the paper is organized as follows.
In Section \ref{sec-formulation}, we formulate the learning problems of the FBS-network and its deep-layer limit under dynamical inclusion modeling.
We establish in Section \ref{sec-convergence-optimal-control} not only the existence of solutions to the learning problems, but also convergence properties from the network training problem to the learning problem of the deep-layer limit system.
In Section \ref{sec-stability}, we discuss the stabilities of both learning problems.
Our main convergence result is verified numerically in Section \ref{sec-experiment}.
Section \ref{sec-conclusion} concludes our study.

\section{Learning problems of FBS-induced networks and related deep-layer limit systems} \label{sec-formulation}

The FBS algorithm is fundamental and efficient for solving the following signal and image processing related linear inverse problem 
\begin{align*}
	\min_{x \in \mathbb{R}^{n}} \Big\{ \frac{1}{2} \|Ax - b\|_{2}^{2} + \lambda \mathcal{R}(x) \Big\},
\end{align*}
where $b \in \mathbb{R}^{m}$ is an observation, $A \in \mathbb{R}^{m \times n}$ (usually $m \leq n$) is a measurement or representation matrix, $\mathcal{R}$ is a smooth or nonsmooth regularization function, and $\lambda > 0$ is a trade-off parameter.
This classical algorithm can be represented as the following iterative scheme
\begin{align} \label{eq-FBS-algorithm}
	x^{k+1} = (I + \alpha^{k} \lambda \partial \mathcal{R})^{-1} (x^{k} - \alpha^{k} A^{\top} (A x^{k} - b) ), \quad k = 0, 1, 2, \cdots,
\end{align}
where $\alpha^{k} > 0$ is the iterative step size, and $\partial \mathcal{R}$ is the subdifferential of $\mathcal{R}$.
From Eq.\eqref{eq-FBS-algorithm} and related variants, numerous effective DNNs \cite{gregor2010learning, kobler2017variational, zhang2018ista, liu2019alista, ding2020low, xiang2021fista, lin2025deep} have been constructed in sparse coding and imaging applications.

\subsection{The basic FBS-network case} \label{subsec-formulation-FBS}

By directly unrolling the FBS algorithm and relaxing parameters, we presented the basic FBS-network in \cite{lin2025deep}, where dynamical inclusion modeling and some theoretical analyses for forward systems are also provided.
As in \cite{lin2025deep}, the forward systems of the basic FBS-network with $N$ layers and its related deep-layer limit system are as follows
\begin{align}
	\label{eq-system-inclusion-discrete} 
	& \left\{ \begin{aligned}
		& x^{N,k+1} - x^{N,k} + h_{N} \alpha^{N,k} (A^{N,k})^{\top} (A^{N,k} x^{N,k} - b) + h_{N} \alpha^{N,k} \lambda^{N,k} (\partial \mathcal{R})(x^{N,k+1})
		 \owns \vec{0}, \\
		&\hspace{55pt} k = 0, 1, \cdots, N-1, \\
		& x^{N,0} = x^{0},
	\end{aligned}\right. \\
	\label{eq-system-inclusion-continuous}
	& \left\{ \begin{aligned}
		& \dot{\pmb{x}}(t) + \pmb{\alpha}(t) (\pmb{A}(t))^{\top} (\pmb{A}(t) \pmb{x}(t) - b) + \pmb{\alpha}(t) \pmb{\lambda}(t) (\partial \mathcal{R})(\pmb{x}(t)) \owns \vec{0}, \,\, {\rm a.e. \,\,} t \in [0,T],  \\
		& \pmb{x}(0) = x^{0},
	\end{aligned}\right.
\end{align}
where $N \in \mathbb{N}_{+}$ is a predefined number of layers of the network, $h_{N} \equiv T/N$, $x^{0}$ is the given initial value, and $A^{N,:} := \{ A^{N,k} \}_{k=0}^{N-1} \subset \mathbb{R}^{m \times n}, \alpha^{N,:} \equiv \{ \alpha^{N,k} \}_{k=0}^{N-1} \subset \mathbb{R}_{+}, \lambda^{N,:} \equiv \{ \lambda^{N,k} \}_{k=0}^{N-1} \subset \mathbb{R}_{+}$, $\pmb{A}$, $\pmb{\alpha}$, $\pmb{\lambda}$ are learnable parameters that generalize the original $A$, $\alpha$, $\lambda$.

In supervised learning, we determine learnable parameters through minimizing the loss functions over given datasets.
Suppose we are given a training set $\{ (b_{j}, y_{j}) \}_{j=1}^{J}$ with $J$ samples, where $(b_{j}, y_{j}) \in \mathbb{R}^{m} \times \mathbb{R}^{n}$ is the $j$-th pair of the input data and the label. 
Let $\{ x^{0}_{j} \}_{j=1}^{J} \subset \mathbb{R}^{n}$ be the initial system states. 
We can then formulate the optimal control learning problem of the FBS-network as follows:
\begin{equation}\label{eq-discrete-optimal-control-problem-multi}
	\left\{\begin{aligned}
		& \min_{ (A^{N,:}, \alpha^{N,:}, \lambda^{N,:}) \in \mathscr{D}_{N} } \bigg\{ \mathcal{J}_{N}^{J}(A^{N,:}, \alpha^{N,:}, \lambda^{N,:}) = \frac{1}{J} \sum_{j=1}^{J} \mathcal{L}(x^{N,N}_{j}; y_{j}) 
		+ \beta_{1} \mathcal{H}_{N}^{(1)}(A^{N,:})  \\
		&\hspace{175pt}+ \beta_{2} \mathcal{H}_{N}^{(2)}(\alpha^{N,:})
		+ \beta_{3} \mathcal{H}_{N}^{(3)}(\lambda^{N,:}) 		
		\bigg\}  \\ 
		& {\rm s.t.} \,\, \vec{0} \in x^{N,k+1}_{j} - x^{N,k}_{j} + h_{N} \alpha^{N,k} (A^{N,k})^{\top} (A^{N,k} x^{N,k}_{j} - b_{j})
		+ h_{N} \alpha^{N,k} \lambda^{N,k} (\partial\mathcal{R})(x^{N,k+1}_{j}),  \\
		&\hspace{75pt}  
		k = 0, 1, \ldots, N-1,  \\
		&\hspace{15pt} x^{N,0}_{j} = x^{0}_{j}, \\
		&\hspace{15pt} j = 1, 2, \ldots, J,
	\end{aligned}\right.
\end{equation}
where $\mathscr{D}_{N} \subset (\mathbb{R}^{m \times n})^{N} \times \mathbb{R}^{N} \times \mathbb{R}^{N}$ is an admissible set for the control variables $A^{N,:}$, $\alpha^{N,:}$, $\lambda^{N,:}$ with $\alpha^{N,:}$, $\lambda^{N,:}$ being nonnegative, $\mathcal{L}: \mathbb{R}^{n} \times \mathbb{R}^{n} \to \mathbb{R}$ is a continuous, lower bounded loss function (e.g., $\mathcal{L}(x; y) := \frac{1}{2} \|x - y\|_{2}^{2}$) for a single sample, and $\mathcal{H}_{N}^{(1)}: (\mathbb{R}^{m \times n})^{N} \to \mathbb{R}_{+}$, $\mathcal{H}_{N}^{(2)}: \mathbb{R}^{N} \to \mathbb{R}_{+}$, $\mathcal{H}_{N}^{(3)}: \mathbb{R}^{N} \to \mathbb{R}_{+}$ are the regularization functions with weight parameters $\beta_{1} \geq 0$, $\beta_{2} \geq 0$, $\beta_{3} \geq 0$, respectively.
Similarly, we have an optimal control learning problem in the continuous-time setting as 
\begin{equation}\label{eq-continuous-optimal-control-problem-multi}
	\left\{\begin{aligned}
		& \inf_{(\pmb{A}, \pmb{\alpha}, \pmb{\lambda}) \in \pmb{\mathscr{D}}} \bigg\{ \mathcal{J}^{J}(\pmb{A}, \pmb{\alpha}, \pmb{\lambda}) = \frac{1}{J} \sum_{j=1}^{J} \mathcal{L}(\pmb{x}_{j}(T); y_{j}) 
		+ \beta_{1} \mathcal{H}^{(1)}(\pmb{A})
		+ \beta_{2} \mathcal{H}^{(2)}(\pmb{\alpha})
		+ \beta_{3} \mathcal{H}^{(3)}(\pmb{\lambda}) 	
		\bigg\}  \\
		& {\rm s.t.} \,\, \vec{0} \in \dot{\pmb{x}}_{j}(t) + \pmb{\alpha}(t) (\pmb{A}(t))^{\top} (\pmb{A}(t) \pmb{x}_{j}(t) - b_{j}) + \pmb{\alpha}(t) \pmb{\lambda}(t) (\partial\mathcal{R})(\pmb{x}_{j}(t)),  
		{\rm \,\, a.e. \,\, } t \in [0,T],  \\
		&\quad\,\,\,\, \pmb{x}_{j}(0) = x^{0}_{j}, \\ 
		&\quad\,\,\,\, j = 1, 2, \ldots, J,
	\end{aligned}\right.
\end{equation}
where $\pmb{\mathscr{D}} \subset L^{\infty}([0,T]; \mathbb{R}^{m \times n} \times \mathbb{R} \times \mathbb{R})$ is an admissible set for the control variables $\pmb{A}$, $\pmb{\alpha}$, $\pmb{\lambda}$ with $\pmb{\alpha}$, $\pmb{\lambda}$ being nonnegative, and $\mathcal{H}^{(1)}: L^{\infty}([0,T];\mathbb{R}^{m \times n}) \to \mathbb{R}_{+}$, $\mathcal{H}^{(2)}: L^{\infty}([0,T]) \to \mathbb{R}_{+}$, $\mathcal{H}^{(3)}: L^{\infty}([0,T]) \to \mathbb{R}_{+}$ are the regularization functions with weight parameters $\beta_{1} \geq 0$, $\beta_{2} \geq 0$, $\beta_{3} \geq 0$, respectively.

\subsection{Some other cases like LISTA} \label{subsec-extension-LISTA}

The above inclusion-based problems \eqref{eq-discrete-optimal-control-problem-multi} and  \eqref{eq-continuous-optimal-control-problem-multi} can be adapted to some other cases induced from the FBS algorithm.
As LISTA is the earliest one in this topic, we here take it as an example.
Recall from \cite{lin2025deep} that the forward systems of the LISTA network with $N$ layers ($N \in \mathbb{N}_{+}$) and its related deep-layer limit system are as follows
\begin{align}
	\label{eq-lfbs-discrete-2}
	&\left\{\begin{aligned}
		& x^{N,k+1} - x^{N,k} + h_{N} (W_{1} x^{N,k} - W_{2} b) 
		+ h_{N} \theta (\partial \|\cdot\|_{1})(x^{N,k+1})
		\owns \vec{0},  
		\,\, k = 0, 1, \ldots, N-1,  \\
		& x^{N,0} := x^{0}.
	\end{aligned}\right.  \\
	\label{eq-lfbs-continuous-2}
	&\left\{\begin{aligned}
		& \dot{\pmb{x}}(t) 
		+ (W_{1} \pmb{x}(t) - W_{2}b) 
		+ \theta (\partial \|\cdot\|_{1})(\pmb{x}(t))
		\owns \vec{0}, \quad {\rm a.e.\,\,} t \in [0,T],  \\
		& \pmb{x}(0) = x^{0}.
	\end{aligned}\right.
\end{align}
where $\theta := \alpha \lambda$, $W_{1} := \alpha A^{\top} A$, and $W_{2} := \alpha A^{\top}$ are all set to be learnable with consideration of the constraint $\theta > 0$.
Therefore, in supervised learning, the optimal control problem of the LISTA network \eqref{eq-lfbs-discrete-2} can be written as below:
\begin{equation}\label{eq-discrete-optimal-control-problem-lfbs}
	\left\{\begin{aligned}
		& \min_{ (W_{1}, W_{2}, \theta) \in \widetilde{\mathscr{D}} } \Big\{ \widetilde{\mathcal{J}}_{N}^{J} (W_{1}, W_{2}, \theta) := \frac{1}{J} \sum_{j=1}^{J} \mathcal{L}(x^{N,N}_{j}; y_{j})  
		+ \beta_{1} \widetilde{\mathcal{H}}^{(1)}(W_{1})  \\
		&\hspace{135pt}+ \beta_{2} \widetilde{\mathcal{H}}^{(2)}(W_{2})
		+ \beta_{3} \widetilde{\mathcal{H}}^{(3)}(\theta) 
		\Big\}  \\
		& {\rm s.t.} \,\, \vec{0} \in x^{N,k+1}_{j} - x^{N,k}_{j} + h_{N} (W_{1} x^{N,k}_{j} - W_{2}b) + h_{N} \theta (\partial \|\cdot\|_{1})(x^{N,k+1}_{j}),  
		\\
		&\hspace{115pt} k = 0, 1, \ldots, N-1,  \\
		& \quad\,\,\,\, x^{N,0}_{j} = x^{0}_{j},  \\
		& \quad\,\,\,\, j = 1, 2, \ldots, J,
	\end{aligned}\right.
\end{equation}
where $\widetilde{\mathscr{D}} \subset \mathbb{R}^{n \times n} \times \mathbb{R}^{n \times m} \times \mathbb{R}$ is an admissible set with $W_{1} \in \mathbb{R}^{n \times n}$, $W_{2} \in \mathbb{R}^{n \times m}$, $\theta \in \mathbb{R}$, and $\widetilde{\mathcal{H}}^{(1)}: \mathbb{R}^{n \times n} \to \mathbb{R}_{+}$, $\widetilde{\mathcal{H}}^{(2)}: \mathbb{R}^{n \times m} \to \mathbb{R}_{+}$, $\widetilde{\mathcal{H}}^{(3)}: \mathbb{R} \to \mathbb{R}_{+}$ are the regularization functions with weight parameters $\beta_{1} \geq 0$, $\beta_{2} \geq 0$, $\beta_{3} \geq 0$, respectively.
Similarly, we write the optimal control problem of the related deep-layer limit system \eqref{eq-lfbs-continuous-2} as
\begin{equation}\label{eq-continuous-optimal-control-problem-lfbs}
	\left\{\begin{aligned}
		& \inf_{ (W_{1}, W_{2}, \theta) \in \widetilde{\mathscr{D}} } \Bigg\{ \widetilde{\mathcal{J}}^{J} (W_{1}, W_{2}, \theta) := \frac{1}{J} \sum_{j=1}^{J} \mathcal{L}(\pmb{x}_{j}(T); y_{j}) 
		+ \beta_{1} \widetilde{\mathcal{H}}^{(1)}(W_{1})  \\
		&\hspace{135pt} + \beta_{2} \widetilde{\mathcal{H}}^{(2)}(W_{2})
		+ \beta_{3} \widetilde{\mathcal{H}}^{(3)}(\theta) 
		\Bigg\}  \\
		& {\rm s.t.} \,\, \vec{0} \in \dot{\pmb{x}}_{j}(t) + (W_{1} \pmb{x}_{j}(t) - W_{2} b) + \theta(t) (\partial \|\cdot\|_{1})(\pmb{x}_{j}(t)), \,\, {\rm a.e.} \,\, t \in [0,T],  \\
		& \quad\,\,\,\, \pmb{x}_{j}(0) = x^{0}_{j}, \\
		& \quad\,\,\,\, j = 1, 2, \ldots, J.
	\end{aligned}\right.
\end{equation}

\subsection{The learning problems to be analyzed and further notations}

In the following sections, we will study the convergence and stability properties of the learning problems for the basic FBS-network and its deep-layer limit system.
Since the number of samples $J$ in practical applications is always finite, we can, without loss of generality, consider the single-sample case of \eqref{eq-discrete-optimal-control-problem-multi} and \eqref{eq-continuous-optimal-control-problem-multi}, which we denote by $(\mathfrak{Q}_{N})$ and $(\mathfrak{Q})$, respectively, as follows
\begin{align}
	\label{eq-discrete-optimal-control-problem-single}
	&\left\{\begin{aligned}
		& \min_{ (A^{N,:}, \alpha^{N,:}, \lambda^{N,:}) \in \mathscr{D}_{N} } \Big\{ \mathcal{J}_{N}(A^{N,:}, \alpha^{N,:}, \lambda^{N,:}) = \mathcal{L}(x^{N,N}; y) 
		+ \beta_{1} \mathcal{H}_{N}^{(1)}(A^{N,:})  \\
		&\hspace{170pt}+ \beta_{2} \mathcal{H}_{N}^{(2)}(\alpha^{N,:})
		+ \beta_{3} \mathcal{H}_{N}^{(3)}(\lambda^{N,:}) 		
		\Big\}  \\
		& {\rm s.t.} \,\, \vec{0} \in x^{N,k+1} - x^{N,k} + h_{N} \alpha^{N,k} (A^{N,k})^{\top} (A^{N,k} x^{N,k} - b)
		+ h_{N} \alpha^{N,k} \lambda^{N,k} (\partial\mathcal{R})(x^{N,k+1}),  
		\\ 
		&\hspace{85pt} k = 0, 1, \ldots, N-1,  \\
		& \quad\,\,\,\, x^{N,0} = x^{0}, 
	\end{aligned}\right. \\
	\label{eq-continuous-optimal-control-problem-single}
	&\left\{\begin{aligned}
		& \inf_{(\pmb{A}, \pmb{\alpha}, \pmb{\lambda}) \in \pmb{\mathscr{D}}} \big\{ \mathcal{J}(\pmb{A}, \pmb{\alpha}, \pmb{\lambda}) = \mathcal{L}(\pmb{x}(T); y) 
		+ \beta_{1} \mathcal{H}^{(1)}(\pmb{A})
		+ \beta_{2} \mathcal{H}^{(2)}(\pmb{\alpha})
		+ \beta_{3} \mathcal{H}^{(3)}(\pmb{\lambda}) 	
		\big\}  \\
		& {\rm s.t.} \,\, \vec{0} \in \dot{\pmb{x}}(t) + \pmb{\alpha}(t) (\pmb{A}(t))^{\top} (\pmb{A}(t) \pmb{x}(t) - b) + \pmb{\alpha}(t) \pmb{\lambda}(t) (\partial\mathcal{R})(\pmb{x}(t)), 
		{\rm \,\, a.e. \,\, } t \in [0,T],  \\
		&\quad\,\,\,\, \pmb{x}(0) = x^{0}. 
	\end{aligned}\right.
\end{align}
There is no essential difference between convergence and stability analysis for problems \eqref{eq-discrete-optimal-control-problem-single} \eqref{eq-continuous-optimal-control-problem-single} and that for multi-sample problems \eqref{eq-discrete-optimal-control-problem-multi} \eqref{eq-continuous-optimal-control-problem-multi}.
Moreover, similar to the forward analyses in \cite{lin2025deep}, our analysis procedure can be simplified to discuss the learning problems in other cases like LISTA and its continuous-time analog.

For our analyses on learning problems of the basic FBS-network and the related limit system in later sections, we give some assumptions on the regularizer $\mathcal{R}$ and the single-sample loss function $\mathcal{L}$:

(A1) $\mathcal{R}: \mathbb{R}^{n} \to \mathbb{R}$ is convex on $\mathbb{R}^{n}$;

(A2) $\vec{0} \in (\partial \mathcal{R})(\vec{0})$;

(A3) there exists a constant $M > 0$, for each $x \in \mathbb{R}^{n}$, either  $\sup \{\|z\|_{2}: z \in (\partial \mathcal{R})(x)\} \leq M \|x\|_{2}$ (case 1), or $\sup \{\|z\|_{2}: z \in (\partial \mathcal{R})(x)\} \leq M $ (case 2); 

(A4) $\mathcal{L}: \mathbb{R}^{n} \times \mathbb{R}^{n} \to \mathbb{R}$ is continuous and nonnegative.  \\
We mention that (A1)-(A3) have been proposed in \cite{lin2025deep} for forward system analyses with numerous useful examples.
Examples for the loss function $\mathcal{L}$ satisfying (A4) include $\mathcal{L}(x; y) := \frac{1}{p} \|x - y\|_{p}^{p}$ with any given $p \in (0,+\infty)$. 
Note that when $p = 2$, the function $\mathcal{L}$ is exactly the widely-used empirical mean square error (MSE) in deep learning literature.

According to the forward stability results in \cite{lin2025deep}, we consider the regularization functions of the problem $(\mathfrak{Q}_{N})$ and $(\mathfrak{Q})$ as follows. 
For some $p \in [1,+\infty)$,
\begin{align*}
	& \mathcal{H}_{N}^{(1)}(A^{N,:}) := \psi\Big( \frac{1}{N} \sum_{k=0}^{N-1} \| A^{N,k} \|_{2}^{p} \Big), 
	\,\, \forall A^{N,k} \in \mathbb{R}^{m \times n}, 
	\,\, \forall k = 0, 1, \ldots, N-1,
	\,\, \forall N \in \mathbb{N}_{+},  \\
	& \mathcal{H}_{N}^{(2)}(\alpha^{N,:}) := \psi\Big( \frac{1}{N} \sum_{k=0}^{N-1} | \alpha^{N,k} |^{p} \Big), 
	\,\, \forall \alpha^{N,k} \in \mathbb{R}_{+}, 
	\,\, \forall k = 0, 1, \ldots, N-1,
	\,\, \forall N \in \mathbb{N}_{+},  \\ 
	& \mathcal{H}_{N}^{(3)}(\lambda^{N,:}) := \psi\Big( \frac{1}{N} \sum_{k=0}^{N-1} | \lambda^{N,k} |^{p} \Big), 
	\,\, \forall \lambda^{N,k} \in \mathbb{R}_{+}, 
	\,\, \forall k = 0, 1, \ldots, N-1,
	\,\, \forall N \in \mathbb{N}_{+},  \\ 
	& \mathcal{H}^{(1)}(\pmb{A}) := \psi\Big( \frac{1}{T} \int_{[0,T]} \| \pmb{A}(t) \|_{2}^{p} {\rm d}t \Big), 
	\,\, \forall \pmb{A} \in L^{p}([0,T];\mathbb{R}^{m \times n}),  \\
	& \mathcal{H}^{(2)}(\pmb{\alpha}) := \psi\Big( \frac{1}{T} \int_{[0,T]} | \pmb{\alpha}(t) |^{p} {\rm d}t \Big), 
	\,\, \forall \pmb{\alpha} \in L^{p}([0,T]),  \\
	& \mathcal{H}^{(3)}(\pmb{\lambda}) := \psi\Big( \frac{1}{T} \int_{[0,T]} | \pmb{\lambda}(t) |^{p} {\rm d}t \Big), 
	\,\, \forall \pmb{\lambda} \in L^{p}([0,T]),
\end{align*}
unless otherwise specified.
We use the following assumption for $\psi$ that

(A5) $\psi: \mathbb{R} \to \mathbb{R}_{+}$ is locally Lipschitz continuous.  \\
Note that if $\psi \equiv 0$, the problems  \eqref{eq-discrete-optimal-control-problem-single} \eqref{eq-continuous-optimal-control-problem-single} reduce to a special case without parameter regularization. 



We recall some notations from \cite{lin2025deep}.
We denote by $B_{n}(z;r)$ and $\overline{B}_{n}(z;r)$ in $\mathbb{R}^{n}$ the $n$-dimensional open ball centered at $z$ with radius $r > 0$ and its closure, respectively.
For $N \in \mathbb{N}_{+}$, the interval $[0,T]$ is partitioned into $N$ subintervals as follows:
\begin{align*}
	[0,T] := \bigcup_{k=0}^{N-2} \Big[ k \frac{T}{N}, (k+1) \frac{T}{N} \Big) \bigcup \Big[ (N-1) \frac{T}{N}, T \Big],
\end{align*}
where each subinterval is denoted by $\Omega_{k}^{N}$ for $k \in \{ 0, 1, \ldots, N-1 \}$.
Let $\mathfrak{X}$ represent a finite dimensional Banach space, and denote $\mathfrak{X}^{N} \equiv \underbrace{\mathfrak{X} \times \mathfrak{X} \times \cdots \times \mathfrak{X}}_{N}$.
We define a piecewise constant extension operator $\mathcal{I}_{N}: \mathfrak{X}^{N} \to L^{\infty}([0,T];\mathfrak{X})$ such that for every $V \in \mathfrak{X}^{N}$,
\begin{align*}
	(\mathcal{I}_{N} V)(t) 
	:= \sum_{k=0}^{N-1} V_{k} \chi_{\Omega_{k}^{N}}(t), \quad \forall t \in [0,T],
\end{align*}
with $\chi_{\mathcal{C}}(t) := \left\{\begin{aligned}
	& 1, & t \in \mathcal{C},  \\
	& 0, & t \notin \mathcal{C}.
\end{aligned}\right.$ being the characteristic function of the set $\mathcal{C}$.
Additionally, we employ a projection operator $\mathcal{P}_{N}: L^{1}([0,T]; \mathfrak{X}) \to \mathfrak{X}^{N}$, such that for every $\pmb{V} \in L^{1}([0,T];\mathfrak{X})$,
\begin{align*}
	(\mathcal{P}_{N} \pmb{V})_{k} := \frac{1}{|\Omega_{k}^{N}|} \int_{\Omega_{k}^{N}} \pmb{V}(s) {\rm d}s, \quad \forall k \in \{ 0, 1, \ldots, N-1 \}.
\end{align*}
We thus can connect the system \eqref{eq-discrete-optimal-control-problem-single} and the system \eqref{eq-continuous-optimal-control-problem-single} by introducing the functional $\mathcal{J}_{N} \circ \mathcal{P}_{N}$ as follows.
For any given $N \in \mathbb{N}_{+}$, we let the function $(\mathcal{J}_{N} \circ \mathcal{P}_{N})(\pmb{A}, \pmb{\alpha}, \pmb{\lambda}) := \mathcal{J}_{N}(\mathcal{P}_{N} \pmb{A}, \mathcal{P}_{N} \pmb{\alpha}, \mathcal{P}_{N} \pmb{\lambda})$ for any $(\pmb{A}, \pmb{\alpha}, \pmb{\lambda}) \in \pmb{\mathscr{D}}$.

We also introduce some further notations and conventions.
For any function $\pmb{U} \in L^{1}([0,T];\mathbb{R}^{m \times n})$, we extend its definition by setting $\pmb{U}(t) := 0$ for $t \in \mathbb{R} \backslash [0,T]$, while still denoting the extended function as $\pmb{U}$.
Hence, the shift operator $\tau_{h}$ is defined such that for any $h \in \mathbb{R}$ and any function $\pmb{U} \in L^{1}(\mathbb{R}; \mathbb{R}^{m \times n})$,
\begin{align*}
	(\tau_{h} \pmb{U})(t) := \pmb{U}(t + h), \quad \forall t \in \mathbb{R}.
\end{align*}
We endow the space $\mathbb{R}^{m \times n} \times \mathbb{R} \times \mathbb{R}$ with the norm
\begin{align*}
	\|(U, \eta, \zeta)\|_{\mathbb{R}^{m \times n} \times \mathbb{R} \times \mathbb{R}} 
	:= \|U\|_{2} + |\eta| + |\zeta|, 
	\quad \forall (U, \eta, \zeta) \in \mathbb{R}^{m \times n} \times \mathbb{R} \times \mathbb{R}.
\end{align*}
We use this to further equip norms for spaces $(\mathbb{R}^{m \times n})^{N} \times \mathbb{R}^{N} \times \mathbb{R}^{N}$ and $L^{p}([0,T]; \mathbb{R}^{m \times n} \times \mathbb{R} \times \mathbb{R})$.
In particular, the $\ell^{p} (p \in [1,+\infty])$ norm of $(U^{N,:}, \eta^{N,:}, \zeta^{N,:}) \in (\mathbb{R}^{m \times n})^{N} \times \mathbb{R}^{N} \times \mathbb{R}^{N}$ is as follows
\begin{align*}
    & \|(U^{N,:}, \eta^{N,:}, \zeta^{N,:})\|_{\ell^{p}((\mathbb{R}^{m \times n})^{N} \times \mathbb{R}^{N} \times \mathbb{R}^{N})}  \\
    :=& \left\{\begin{aligned}
        & \Big( \frac{T}{N} \sum_{k=0}^{N-1} \|(U^{N,k}, \eta^{N,k}, \zeta^{N,k})\|_{\mathbb{R}^{m \times n} \times \mathbb{R} \times \mathbb{R}}^{p} \Big)^{\frac{1}{p}}, \, & p& \in [1,+\infty),  \\
        & \max_{0 \leq k \leq N-1} \|(U^{N,k}, \eta^{N,k}, \zeta^{N,k})\|_{\mathbb{R}^{m \times n} \times \mathbb{R} \times \mathbb{R}}, \, & p& = +\infty.
    \end{aligned}\right. 
\end{align*}
For $p \in [1,+\infty]$, the norm in $L^{p}([0,T]; \mathbb{R}^{m \times n} \times \mathbb{R} \times \mathbb{R})$ naturally reads
\begin{align*}
	& \|(\pmb{U}, \pmb{\eta}, \pmb{\zeta})\|_{L^{p}([0,T]; \mathbb{R}^{m \times n} \times \mathbb{R} \times \mathbb{R})} 
	:= 
	\left\{\begin{aligned}
		& \big( \int_{[0,T]} \|(\pmb{U}(s), \pmb{\eta}(s), \pmb{\zeta}(s))\|_{\mathbb{R}^{m \times n} \times \mathbb{R} \times \mathbb{R}}^{p} {\rm d}s \big)^{\frac{1}{p}}, 
		\, & p& \in [1,+\infty),  \\
		& \esssup_{s \in [0,T]} \|(\pmb{U}(s), \pmb{\eta}(s), \pmb{\zeta}(s))\|_{2}, 
		\, & p& = +\infty.
	\end{aligned}\right. 
\end{align*}
for any $(\pmb{U}, \pmb{\eta}, \pmb{\zeta}) \in L^{p}([0,T]; \mathbb{R}^{m \times n} \times \mathbb{R} \times \mathbb{R})$.
Note that the above factor $\frac{T}{N}$ is introduced for consistency between $(\mathbb{R}^{m \times n})^{N} \times \mathbb{R}^{N} \times \mathbb{R}^{N}$ and $L^{p}([0,T]; \mathbb{R}^{m \times n} \times \mathbb{R} \times \mathbb{R})$.

\section{Convergence properties from the learning problem of the basic FBS-network to that of its deep-layer limit system} 
\label{sec-convergence-optimal-control}

We start from the existence of solutions to \eqref{eq-discrete-optimal-control-problem-single} and \eqref{eq-continuous-optimal-control-problem-single}.

\subsection{Existence of solutions to the learning problems of the basic FBS-network and its deep-layer limit system}

\subsubsection{Existence of solutions to the learning problem \eqref{eq-discrete-optimal-control-problem-single}}

The existence of solutions to the learning problem \eqref{eq-discrete-optimal-control-problem-single} is straightforward by noting the following lemma, which is included here for completeness.

\begin{lemma}\label{lemma-perturb-of-prox}
	Suppose the assumptions (A1)(A2) hold and a vector $y \in \mathbb{R}^{n}$ be given. 
	Denote the minimizer of the optimization problem $\min_{x \in \mathbb{R}^{n}} \{\mathcal{R}(x) + \frac{1}{2 \rho} \|x - y\|_{2}^{2} \}$ with $\rho > 0$ as $x^{\rho} = {\rm prox}_{\rho \mathcal{R}}(y)$. 
	If $\Delta \rho \to 0$, then $x^{\rho + \Delta \rho} \to x^{\rho}$. 
\end{lemma}

\begin{proof}[Proof. ]
	By denoting $f_{\rho}(x) := \mathcal{R}(x) + \frac{1}{2 \rho} \|x - y\|_{2}^{2}$, we first show 
	\begin{align}\label{eq-lemma3.1-1}
		f_{\rho}(x^{\rho}) = \lim_{\Delta \rho \to 0} f_{\rho}(x^{\rho + \Delta \rho}),
	\end{align}
	i.e.,
	\begin{align*} 
		\mathcal{R}(x^{\rho}) + \frac{1}{2 \rho} \|x^{\rho} - y\|_{2}^{2}  
		= \lim_{\Delta \rho \to 0} \left\{ \mathcal{R}(x^{\rho + \Delta \rho}) + \frac{1}{2 \rho} \|x^{\rho + \Delta \rho} - y\|_{2}^{2} \right\}.
	\end{align*}	
	By the definition, we have
	\begin{align}
		\label{ineq-1} \mathcal{R}(x^{\rho}) + \frac{1}{2 \rho} \|x^{\rho} - y\|_{2}^{2}
		\leq& \mathcal{R}(x^{\rho + \Delta \rho}) + \frac{1}{2 \rho} \|x^{\rho + \Delta \rho} - y\|_{2}^{2},  \\
		\label{ineq-2} \mathcal{R}(x^{\rho + \Delta \rho}) + \frac{1}{2(\rho + \Delta \rho)} \|x^{\rho + \Delta \rho} - y\|_{2}^{2}
		\leq& \mathcal{R}(x^{\rho}) + \frac{1}{2(\rho + \Delta \rho)} \|x^{\rho} - y\|_{2}^{2}.
	\end{align}
	It then follows that
	\begin{align}
		\label{ineq-3} \mathcal{R}(x^{\rho}) + \frac{1}{2 \rho} \|x^{\rho} - y\|_{2}^{2}
		\leq& \mathop{\lim\inf}_{\Delta \rho \to 0} \left\{ \mathcal{R}(x^{\rho + \Delta \rho}) + \frac{1}{2 \rho} \|x^{\rho + \Delta \rho} - y\|_{2}^{2} \right\},  \\
		\label{ineq-4} \mathop{\lim\sup}_{\Delta \rho \to 0} \left\{ \mathcal{R}(x^{\rho + \Delta \rho}) + \frac{1}{2(\rho + \Delta \rho)} \|x^{\rho + \Delta \rho} - y\|_{2}^{2} \right\}
		\leq& \mathop{\lim\sup}_{\Delta \rho \to 0} \left\{ \mathcal{R}(x^{\rho}) + \frac{1}{2(\rho + \Delta \rho)} \|x^{\rho} - y\|_{2}^{2} \right\}.
	\end{align}
	Since the nonexpansiveness of the proximal operator and ${\rm prox}_{\rho \mathcal{R}}(\vec{0}) = \vec{0}$ indicate $\|x^{\rho}\|_{2} \leq \|y\|_{2}$ for every $\rho > 0$, we obtain
	\begin{align*}
		& \mathop{\lim\sup}_{\Delta \rho \to 0} \left\{ \mathcal{R}(x^{\rho + \Delta \rho}) + \frac{1}{2 \rho} \|x^{\rho + \Delta \rho} - y\|_{2}^{2} \right\}  \\
		=& \mathop{\lim\sup}_{\Delta \rho \to 0} \left\{ \mathcal{R}(x^{\rho + \Delta \rho}) + \frac{1}{2(\rho + \Delta \rho)} \|x^{\rho + \Delta \rho} - y\|_{2}^{2} \right\}   \\
		\leq& \mathop{\lim\sup}_{\Delta \rho \to 0} \left\{ \mathcal{R}(x^{\rho}) + \frac{1}{2(\rho + \Delta \rho)} \|x^{\rho} - y\|_{2}^{2} \right\} \,\,({\rm by \,\, Eq.} \eqref{ineq-4})  \\
		=& \mathcal{R}(x^{\rho}) + \frac{1}{2 \rho} \|x^{\rho} - y\|_{2}^{2}  \\
		\leq& \mathop{\lim\inf}_{\Delta \rho \to 0} \left\{ \mathcal{R}(x^{\rho + \Delta \rho}) + \frac{1}{2 \rho} \|x^{\rho + \Delta \rho} - y\|_{2}^{2} \right\} \,\,({\rm by \,\, Eq.} \eqref{ineq-3}),
	\end{align*}
	which proves Eq.\eqref{eq-lemma3.1-1}.

	We then show the lemma by using Eq.\eqref{eq-lemma3.1-1}.
	Since $f_{\rho}(x)$ is strongly convex with the constant $\frac{1}{\rho} > 0$, then for every $g \in \partial f(x^{\rho})$,
	\begin{align*}
		\frac{1}{2\rho} \|x^{\rho + \Delta \rho} - x^{\rho}\|_{2}^{2}
		\leq f_{\rho}(x^{\rho + \Delta \rho}) - f_{\rho}(x^{\rho}) - g^{\top} \cdot (x^{\rho + \Delta \rho} - x^{\rho}). 
	\end{align*}
	As $x^{\rho}$ is the minimizer of $f_{\rho}$, we can choose $g = \vec{0}$ and complete the proof.
\end{proof}                       

Now we show that the learning problem $(\mathfrak{Q}_{N})$ can attain its minimum.

\begin{theorem}\label{thm-existence-of-solution-to-optimal-control-discrete}
	Suppose the assumptions (A1)(A2)(A4)(A5) hold. 
	Let the observed data $b \in \mathbb{R}^{m}$, the initial value $x^{0} \in \mathbb{R}^{n}$, and the label $y \in \mathbb{R}^{n}$ be given. Consider the optimal control problem $(\mathfrak{Q}_{N})$. If the admissible set $\mathscr{D}_{N}$ is closed and bounded, then $(\mathfrak{Q}_{N})$ has a solution.	
\end{theorem}

\begin{proof}[Proof. ]
	By \textbf{Lemma} \ref{lemma-perturb-of-prox}, it is not difficult to see that for every $k = 0, 1, \ldots, N-1$, the state $x^{N,k+1}$ is continuous with respect to (w.r.t.) $(A^{N,:}, \alpha^{N,:}, \lambda^{N,:}, x^{N,k})$, and thus $x^{N,N}$ is continuous w.r.t. $(A^{N,:}, \alpha^{N,:}, \lambda^{N,:})$.
	Therefore, by the assumption of $\mathcal{L}$, we see that the objective functional $\mathcal{J}_{N}$ is continuous w.r.t. $(A^{N,:}, \alpha^{N,:}, \lambda^{N,:})$.
	If the admissible set $\mathscr{D}_{N}$ is closed and bounded, the optimal control problem $(\mathfrak{Q}_{N})$ can attain its minimum over the feasible set of learnable parameters.
\end{proof}

In the following discussion, we denote the solution set of the problem $(\mathfrak{Q}_{N})$ as $S_{N}$.

\subsubsection{Existence of solutions to the learning problem \eqref{eq-continuous-optimal-control-problem-single}}
We now provide an existence theorem of solutions to the optimal control problem to the deep-layer limit system as the continuous-time analog of that of the basic FBS-network. 
We first give two lemmas, which indicate a Bochner version of \textbf{Kolmogorov}-\textbf{Ri{\'e}sz}-\textbf{Fr{\'e}chet} \textbf{Theorem} \cite[Thm 4.26]{brezis2010functional} and the sufficiency of the upper semicontinuity (u.s.c.) of a set-valued mapping to its outer semicontinuity (o.s.c.).

\begin{lemma}\label{lemma-Kolmogorov-Riesz-bochner}
	(Matrix-valued version of Kolmogorov-Ri{\'e}sz-Fr{\'e}chet Theorem)
	Let $\mathcal{F}$ be a bounded set in $L^{p}(\mathbb{R}; \mathbb{R}^{m \times n})$ with $p \in [1,+\infty)$. 
	Assume that
	\begin{align*}
		\lim_{|h| \to 0} \|\tau_{h} \pmb{U} - \pmb{U}\|_{L^{p}(\mathbb{R}; \mathbb{R}^{m \times n})}  = 0 \,\, {\rm uniformly \,\, in} \,\, \pmb{U} \in \mathcal{F},
	\end{align*}
	i.e., $\forall \varepsilon > 0$, $\exists \delta > 0$ such that $\forall h \in \mathbb{R}$ with $|h| < \delta$, the inequality $\|\tau_{h} \pmb{U} - \pmb{U}\|_{L^{p}(\mathbb{R}; \mathbb{R}^{m \times n})} < \varepsilon$ holds for all $\pmb{U} \in \mathcal{F}$. 
	Then the closure of $\mathcal{F} |_{\Omega}$ in $L^{p}(\mathbb{R}; \mathbb{R}^{m \times n})$ is compact for any measurable set $\Omega \subset \mathbb{R}$	with finite measure. 
	Here $\mathcal{F} |_{\Omega}$ denotes the restrictions to $\Omega$ of the functions in $\mathcal{F}$.
\end{lemma}

\begin{proof}[Proof. ]
	By noting $|\pmb{U}_{i,j}(t)| \leq \|\pmb{U}(t)\|_{2}$ for a.e. $t \in \Omega$, any $i \in \{1, 2, \cdots, m\}$, and any $j \in\{1, 2, \cdots, n\}$, we can apply the classical \textbf{Kolmogorov}-\textbf{Ri{\'e}sz}-\textbf{Fr{\'e}chet} \textbf{Theorem} \ref{thm-Kolmogorov-Riesz} to an entry $\mathcal{F}_{i,j} \equiv \{ \pmb{V}_{i,j}: \pmb{V} \in \mathcal{F} \}$ to obtain its compactness.
	By progressively extracting $mn$ (a finite number) times of subsequences and $\|\pmb{U}(t)\|_{2}^{p} \leq C_{m,n,p} \sum_{i=1}^{m} \sum_{j=1}^{n} |\pmb{U}_{i,j}(t)|^{p}$ for a.e. $t \in \Omega$ with some certain constant $C_{m,n,p} > 0$, we derive the compactness of $\mathcal{F} |_{\Omega}$.
\end{proof}

Clearly, we can extend \textbf{Lemma} \ref{lemma-Kolmogorov-Riesz-bochner} to the case for $\mathcal{F}$ being a bounded set in $L^{p}(\mathbb{R};\Pi_{i \in \Theta} \mathbb{R}^{m_{i} \times n_{i}})$ with a finite index set $\Theta$, which can be seen as a Bochner version of \textbf{Lemma} \ref{lemma-Kolmogorov-Riesz-bochner}.

\begin{lemma}\label{lemma-usc-imply-osc}
	(U.s.c. implies o.s.c)
	Let the set-valued mapping $F: \mathbb{R}^{n} \rightrightarrows \mathbb{R}^{m}$ be closed-valued. If $F$ is u.s.c., then $F$ is also o.s.c..
\end{lemma}

\begin{proof}[Proof. ]
	Consider any given $y \in \mathop{\lim\sup}_{x \to \overline{x}} F(x)$. 
	Thus there exist $x^{\nu} \to \overline{x}$ and $y^{\nu} \in F(x^{\nu})$ such that $y^{\nu} \to y$. 
	Since $F$ is u.s.c., then we have ${\rm dist}(y^{\nu}, F(\overline{x})) \to 0$. 
	Note that $F$ is closed-valued, then if $y^{\nu} \to y$, we can derive $y \in F(\overline{x})$. 
	Hence, $\mathop{\lim\sup}_{x \to \overline{x}} F(x) \subseteq F(\overline{x})$, which completes the proof.
\end{proof}

Now, we give the existence of the solution to the problem $(\mathfrak{Q})$.
\begin{theorem} \label{thm-existence-of-solution-to-optimal-control-continuous}
	Suppose the assumptions (A1)-(A5) hold and $p \in [1,+\infty)$. 
	Let the observed data $b \in \mathbb{R}^{m}$, the initial value $x^{0} \in \mathbb{R}^{n}$, and the label $y \in \mathbb{R}^{n}$.
	Recall $\pmb{\mathscr{D}}$ as a set of $(\pmb{A}, \pmb{\alpha}, \pmb{\lambda})$ with $\pmb{\alpha}$, $\pmb{\lambda}$ both being nonnegative.
	Consider the optimal control problem $(\mathfrak{Q})$ \eqref{eq-continuous-optimal-control-problem-single} of the related deep-layer limit system. 
	If the following conditions 
	\begin{itemize}			
		\item[(1) ] (Uniform boundedness in $L^{\infty}$ space) $\pmb{\mathscr{D}}$ is bounded in $L^{\infty}([0,T]; \mathbb{R}^{m \times n} \times \mathbb{R} \times \mathbb{R})$;
		
		\item[(2) ] (Closedness in $L^{p}$ space) $\pmb{\mathscr{D}}$ is closed in $L^{p}([0,T]; \mathbb{R}^{m \times n} \times \mathbb{R} \times \mathbb{R})$;
		
		\item[(3) ] (Equi-continuity in $L^{p}$ sense)
		$\pmb{\mathscr{D}}$ satisfies  
		\begin{align*}
			\lim\limits_{|h| \to 0} \|\tau_{h}(\pmb{A}, \pmb{\alpha}, \pmb{\lambda}) - (\pmb{A}, \pmb{\alpha}, \pmb{\lambda})\|_{L^{p}(\mathbb{R};\mathbb{R}^{m \times n} \times \mathbb{R} \times \mathbb{R})} = 0, \quad {\rm uniformly \,\, for \,\, } (\pmb{A}, \pmb{\alpha}, \pmb{\lambda}) \in \pmb{\mathscr{D}},
		\end{align*} 
	\end{itemize} 
	hold, then the minimization problem $(\mathfrak{Q})$ has at least a solution.
\end{theorem}

\begin{proof}[Proof. ]
	We divide the proof into four steps. 
	
	\textbf{Step 1:} 
	By the \textbf{Existence}  
	\textbf{and} \textbf{Uniqueness} \textbf{Theorem} \cite[Thm.3.2]{lin2025deep},
	we see that for each $(\pmb{A}, \pmb{\alpha}, \pmb{\lambda}) \in \pmb{\mathscr{D}}$, the state inclusion \eqref{eq-system-inclusion-continuous}
	has a unique absolute continuous solution $\pmb{x}$, and the objecctive function $\mathcal{J}$ is well defined.
	For later bound estimations, we denote three bound constants $M_{\pmb{A}} > 0$, $M_{\pmb{\alpha}} > 0$, and $M_{\pmb{\lambda}} > 0$ such that for every $(\pmb{A}, \pmb{\alpha}, \pmb{\lambda}) \in \pmb{\mathscr{D}}$,  
	\begin{equation*}
		\|\pmb{A}\|_{L^{\infty}([0,T];\mathbb{R}^{m \times n})} < M_{\pmb{A}}, \quad \|\pmb{\alpha}\|_{L^{\infty}([0,T])} < M_{\pmb{\alpha}}, \quad  \|\pmb{\lambda}\|_{L^{\infty}([0,T])} < M_{\pmb{\lambda}},
	\end{equation*}
	by the uniform boundedness given in assumption (1).
	
	Suppose $\{ (\pmb{A}^{(j)}, \pmb{\alpha}^{(j)}, \pmb{\lambda}^{(j)}) \}_{j=1}^{+\infty} \subset \pmb{\mathscr{D}}$ be a minimizing sequence of the problem $(\mathfrak{Q})$ in \eqref{eq-continuous-optimal-control-problem-single}, and $\pmb{x}^{(j)}: [0,T] \to \mathbb{R}^{n}$ be the unique associated state to the state inclusion \eqref{eq-system-inclusion-continuous} with the tuple $( \pmb{A}^{(j)}, \pmb{\alpha}^{(j)}, \pmb{\lambda}^{(j)} )$, $j = 1, 2, \cdots$. 
	Hence, $\pmb{x}^{(j)}$ satisfies 
	\begin{align}\label{eq-system-inclusion-continuous-j}
		\left\{\begin{aligned}
			& \dot{\pmb{x}}^{(j)}(t) + \pmb{\alpha}^{(j)}(t) (\pmb{A}^{(j)}(t))^{\top} (\pmb{A}^{(j)}(t) \pmb{x}^{(j)}(t) - b) + \pmb{\alpha}^{(j)}(t) \pmb{\lambda}^{(j)}(t) (\partial\mathcal{R})(\pmb{x}^{(j)}(t))
			\owns \vec{0}, {\rm \,\,a.e.\,\,} t \in [0,T],  \\
			& \pmb{x}^{(j)}(0) = x^{0}.
		\end{aligned}\right.
	\end{align}
	
	Due to the assumptions (1)-(3) and a Bochner version of \textbf{Lemma} \ref{lemma-Kolmogorov-Riesz-bochner}, there exist a subsequence of $\{ (\pmb{A}^{(j)}, \pmb{\alpha}^{(j)}, \pmb{\lambda}^{(j)}) \}_{j=1}^{+\infty} \subset \pmb{\mathscr{D}}$
	(not relabeled) and a limit point $(\pmb{A}^{*}, \pmb{\alpha}^{*}, \pmb{\lambda}^{*}) \in \pmb{\mathscr{D}}$ (with $\pmb{\alpha}^{*}$, $\pmb{\lambda}^{*}$ both being nonnegative) such that
	\begin{align*}
		(\pmb{A}^{(j)}, \pmb{\alpha}^{(j)}, \pmb{\lambda}^{(j)}) \overset{L^{p}}{\to} (\pmb{A}^{*}, \pmb{\alpha}^{*}, \pmb{\lambda}^{*}), \quad j \to +\infty.
	\end{align*}
	By \textbf{Theorem} \ref{thm-Lp-imply-a.e.pointwise}, this sequence has a subsequence (not relabeled) that converges a.e. on $[0,T]$.
	
	\textbf{Step 2:} We will show the uniform boundedness of $\{ \pmb{x}^{(j)} \}_{j=1}^{+\infty}$ in $H^{1}((0,T);\mathbb{R}^{n})$ and thus the existence of a weak cluster point $\pmb{x}^{*}$.
	By \textbf{Bound} \textbf{Estimation} \textbf{Theorem}  \cite[Thm.3.2]{lin2025deep},
	one has, for each $j \in \mathbb{N}_{+}$, 
	\begin{align*}
		\|\pmb{x}^{(j)}(t)\|_{2} 
		\leq& \|x^{0}\|_{2} \cdot {\rm exp}  \Big( \int_{[0,t]} |\pmb{\alpha}^{(j)}(s)| \|\pmb{A}^{(j)}(s)\|_{2}^{2} {\rm d}s \Big)  \\
		&+ \int_{[0,t]} |\pmb{\alpha}^{(j)}(s)| \|\pmb{A}^{(j)}(s)\|_{2} \|b\|_{2} \cdot {\rm exp} \Big( \int_{[s,t]} |\pmb{\alpha}^{(j)}(r)| \|\pmb{A}^{(j)}(r)\|_{2}^{2} {\rm d}r \Big) {\rm d}s  \\
		\leq& \|x^{0}\|_{2} \cdot {\rm exp} (M_{\pmb{\alpha}} M_{\pmb{A}}^{2} T ) + M_{\pmb{\alpha}}  M_{\pmb{A}} T \|b\|_{2} \cdot {\rm exp} (M_{\pmb{\alpha}} M_{\pmb{A}}^{2} T) =: M_{0}, \quad \forall t \in [0,T],
	\end{align*}
	where $M_{0} := M_{0}(M_{\pmb{A}}, M_{\pmb{\alpha}}, b, x^{0}, T) \geq 0$ is clearly independent of $t \in [0,T]$ and $j \in \mathbb{N}_{+}$. 
	This leads to 
	\begin{align*}
		\|\pmb{x}^{(j)}\|_{L^{2}([0,T];\mathbb{R}^{n})} = \Big( \int_{[0,T]} \|\pmb{x}^{(j)}(t)\|_{2}^{2} {\rm d}t \Big)^{\frac{1}{2}} \leq T^{\frac{1}{2}} M_{0}, 
		\quad \forall j \in \mathbb{N}_{+}.
	\end{align*}
	Moreover, 
	for every $j \in\mathbb{N}_{+}$ and every $t \in [0,T]$,
	\begin{align*}
		\sup \{ \|z\|_{2}: z \in (\partial \mathcal{R}) (\pmb{x}^{(j)}(t)) \} 
		\leq \sup \{ \|z\|_{2}: z \in (\partial \mathcal{R}) (\overline{B}_{n}(\vec{0};M_{0} + 1))\} =: M_{1},
	\end{align*}
	where the constant $M_{1} := M_{1}(M_{\pmb{A}}, M_{\pmb{\alpha}}, b, x^{0}, T) > 0$ is independent of $j \in \mathbb{N}_{+}$.
	Using these, we derive from the state inclusion \eqref{eq-system-inclusion-continuous-j} that,
	\begin{align*}
		\|\dot{\pmb{x}}^{(j)}(t)\|_{2} 
		\leq& |\pmb{\alpha}^{(j)}(t)| \|\pmb{A}^{(j)}(t)\|_{2} (\|\pmb{A}^{(j)}(t)\|_{2} \|\pmb{x}^{(j)}(t)\|_{2} + \|b\|_{2})  \\
		&+ |\pmb{\alpha}^{(j)}(t)| |\pmb{\lambda}^{(j)}(t)| \sup\{ \|z\|_{2}: z \in (\partial\mathcal{R})(\pmb{x}^{(j)}(t)) \}  \\
		\leq& |\pmb{\alpha}^{(j)}(t)| \|\pmb{A}^{(j)}(t)\|_{2}^{2} M_{0} + |\pmb{\alpha}^{(j)}(t)| \|\pmb{A}^{(j)}(t)\|_{2}\|b\|_{2} + |\pmb{\alpha}^{(j)}(t)| |\pmb{\lambda}^{(j)}(t)| M_{1}  \\
		\leq& M_{\pmb{\alpha}} M_{\pmb{A}}^{2} M_{0} + M_{\pmb{\alpha}} M_{\pmb{A}}\|b\|_{2} + M_{\pmb{\alpha}} M_{\pmb{\lambda}} M_{1} =: M_{2}, \quad {\rm a.e. \,\,} t \in [0,T], 
	\end{align*}
	where the constant $M_{2} := M_{2}(M_{1}, M_{\pmb{A}}, M_{\pmb{\alpha}}, M_{\pmb{\lambda}}, b, x^{0}) > 0$ is independent of $j \in \mathbb{N}_{+}$.
	Hence, 
	\begin{equation*}
		\|\dot{\pmb{x}}^{(j)}\|_{L^{2}([0,T];\mathbb{R}^{n})} 
		= \Big( \int_{[0,T]} \|\dot{\pmb{x}}^{(j)}(t)\|_{2}^{2} {\rm d}t \Big)^{\frac{1}{2}} 
		\leq T^{\frac{1}{2}} M_{2},
		\quad \forall j \in \mathbb{N}_{+}.
	\end{equation*}
	
	The boundedness of $\{ \pmb{x}^{(j)} \}_{j=1}^{+\infty}$ and $\{ \dot{\pmb{x}}^{(j)} \}_{j=1}^{+\infty}$ in $L^{2}([0,T];\mathbb{R}^{n})$ imply the boundedness of $\{ \pmb{x}^{(j)} \}_{j=1}^{+\infty}$ in $H^{1}((0,T);\mathbb{R}^{n})$, since $\{\pmb{x}^{(j)}\}_{j=1}^{+\infty}$ are univariate functions.
	By \textbf{Embedding} \textbf{Theorem} 
	\cite[Thm 10.13]{alt2016linear}
	and noting the absolute continuity of $\pmb{x}^{(j)}$, we derive that $\pmb{x}^{(j)} \to \pmb{x}^{*}$ in $C^{0}([0,T];\mathbb{R}^{n})$ uniformly up to a subsequence (and without relabeled).
	Besides, there exists a subsequence (not relabeled) such that for $j \to +\infty$, $\pmb{x}^{(j)} \rightharpoonup \pmb{x}^{*}$ in $H^{1}((0,T);\mathbb{R}^{n})$, 
	indicating by 
	\cite[E.g. 8.4(3), page 230]{alt2016linear}
	that 
	\begin{align}\label{eq-derivative-weak-convergence}
		\dot{\pmb{x}}^{(j)} \rightharpoonup \dot{\pmb{x}}^{*} {\rm \,\,in \,\,} L^{2}((0,T);\mathbb{R}^{n}).
	\end{align}
	
	\textbf{Step 3:} Now we prove that $\pmb{x}^{*}$ is the unique solution to the state inclusion \eqref{eq-system-inclusion-continuous} with parameters $(\pmb{A}^{*}, \pmb{\alpha}^{*}, \pmb{\lambda}^{*})$. 
	Here we follow the proof framework in \cite[Proof of Thm 2.1]{elliott1985convergence}, but we are under weaker assumptions on the differential inclusion and need to construct the subgradient function by ourselves carefully from the subdifferentiation term. 
	
	\textbf{Step 3-1:} In this substep, for each $j \in \mathbb{N}_{+}$, since $\pmb{\alpha}^{(j)}(t) \pmb{\lambda}^{(j)}(t)$ possibly equals to zero, we wish to construct a measurable mapping $\pmb{g}^{(j)}: [0,T] \to \mathbb{R}^{n}$ such that $\pmb{g}^{(j)}(t) \in (\partial\mathcal{R})(\pmb{x}^{(j)}(t))$ for each $t \in [0,T]$, and satisfies 
	\begin{align}\label{eq-system-equation-continuous-j}
		\dot{\pmb{x}}^{(j)}(t) + \pmb{\alpha}^{(j)}(t) (\pmb{A}^{(j)}(t))^{\top} (\pmb{A}^{(j)}(t) \pmb{x}^{(j)}(t) - b) + \pmb{\alpha}^{(j)}(t) \pmb{\lambda}^{(j)}(t) \pmb{g}^{(j)}(t) = \vec{0}, \quad {\rm a.e.\,\,} t \in [0,T],
	\end{align}
	and moreover, $\{\pmb{g}^{(j)}\}_{j=1}^{+\infty}$ has a weak cluster point in $L^{2}([0,T];\mathbb{R}^{n})$.
	
	According to the state inclusion \eqref{eq-system-inclusion-continuous-j}, for each $j \in \mathbb{N}_{+}$, there always exists a selection function $\widetilde{\pmb{g}}^{(j)}$ with $\widetilde{\pmb{g}}^{(j)}(t) \in (\partial\mathcal{R})(\pmb{x}^{(j)}(t))$ for every $t \in [0,T]$ satisfying 
	\begin{align}\label{eq-system-equation-continuous-j-tilde}
		\dot{\pmb{x}}^{(j)}(t) + \pmb{\alpha}^{(j)}(t) (\pmb{A}^{(j)}(t))^{\top} (\pmb{A}^{(j)}(t) \pmb{x}^{(j)}(t) - b) + \pmb{\alpha}^{(j)}(t) \pmb{\lambda}^{(j)}(t) \widetilde{\pmb{g}}^{(j)}(t) = \vec{0}, \quad {\rm a.e.\,\,} t \in [0,T].
	\end{align}
	Note that $\widetilde{\pmb{g}}^{(j)}$ is not necessarily measurable on $[0,T]$.
	For each $j \in \mathbb{N}_{+}$, we define two measurable sets 
	\begin{equation*}
		E_{1}^{(j)} := \{ t \in [0,T]: \dot{\pmb{x}}^{(j)}(t) {\rm \,\, does \,\, not \,\, exist} \}, \quad 
		E_{2}^{(j)} := \{ t \in [0,T]: \pmb{\alpha}^{(j)}(t) \pmb{\lambda}^{(j)}(t) \neq 0 \},
	\end{equation*}
	where $\mu(E_{1}^{(j)}) = 0$.
	We try to construct $\pmb{g}^{(j)}$ on two disjoint sets: $E_{2}^{(j)} \backslash E_{1}^{(j)}$ and $[0,T] \backslash (E_{2}^{(j)} \backslash E_{1}^{(j)})$.  
	On the one hand, it is easy to see that on $E_{2}^{(j)} \backslash E_{1}^{(j)}$,  
	\begin{align*}
		\widetilde{\pmb{g}}^{(j)}(t) :=
		- \frac{1}{\pmb{\alpha}^{(j)}(t) \pmb{\lambda}^{(j)}(t)} ( \dot{\pmb{x}}^{(j)}(t) + \pmb{\alpha}^{(j)}(t) (\pmb{A}^{(j)}(t))^{\top} (\pmb{A}^{(j)}(t) \pmb{x}^{(j)}(t) - b) ),
	\end{align*}
	is measurable.
	On the other hand, let us consider $[0,T] \backslash (E_{2}^{(j)} \backslash E_{1}^{(j)})$.
	Note that $\partial\mathcal{R}$ is closed-valued and u.s.c. on $\mathbb{R}^{n}$ due to the assumptions (A1)-(A3) and 
	\cite[Prop 6.1.1]{borwein2010convex}.
	Then by \textbf{Lemma} \ref{lemma-usc-imply-osc}, we derive that $\partial\mathcal{R}$ is o.s.c. on $\mathbb{R}^{n}$. 
	Since $\pmb{x}^{(j)} \in C^{0}([0,T];\mathbb{R}^{n})$, then for each $y \in \mathop{\lim\sup}_{t \to \overline{t}} (\partial\mathcal{R} \circ \pmb{x}^{(j)})(t)$, there exist $t^{\nu} \to \overline{t}$ (hence $\pmb{x}^{(j)}(t^{\nu}) \to \pmb{x}^{(j)}(\overline{t})$) and $y^{\nu} \in (\partial\mathcal{R})(\pmb{x}^{(j)}(t^{\nu})) \equiv (\partial\mathcal{R} \circ \pmb{x}^{(j)})(t^{\nu})$ such that $y^{\nu} \to y$. Thus, $\partial\mathcal{R} \circ \pmb{x}^{(j)}$ is also o.s.c. on $[0,T]$. By \cite[E.g. 14.9]{rockafellar2009variational}
	(closed-valued + o.s.c. derive measurability) and \textbf{Measurable} \textbf{Selection} \textbf{Theorem} 
	\cite[Coro.14.6]{rockafellar2009variational},
	there exists a measurable selection $\overline{\pmb{g}}^{(j)}: [0,T] \to \mathbb{R}^{n}$ such that $\overline{\pmb{g}}^{(j)}(t) \in (\partial\mathcal{R} \circ \pmb{x}^{(j)})(t)$ for each $t \in [0,T]$. 
	Therefore, for each $j \in \mathbb{N}_{+}$, we define the mapping $\pmb{g}^{(j)}: [0,T] \to \mathbb{R}^{n}$ as
	\begin{equation*}
		\pmb{g}^{(j)}(t) := \left\{\begin{aligned}
			& \widetilde{\pmb{g}}^{(j)}(t), \quad t \in E_{2}^{(j)} \backslash E_{1}^{(j)},  \\
			& \overline{\pmb{g}}^{(j)}(t), \quad t \in [0,T] \backslash (E_{2}^{(j)} \backslash E_{1}^{(j)}),
		\end{aligned}\right.
	\end{equation*}
	which is measurable on $[0,T]$ and satisfies the state equation \eqref{eq-system-equation-continuous-j}. 
		
	It is straightforward that $\pmb{g}^{(j)}(t) \in (\partial\mathcal{R})(\pmb{x}^{(j)}(t)) \subseteq \overline{B}_{n}(\vec{0}, M_{1})$ for a.e. $t \in [0,T]$ and every $j \in \mathbb{N}_{+}$, which indicates 
	$\int_{[0,T]} \|\pmb{g}^{(j)}(t)\|_{2}^{2} {\rm d}t \leq T M_{1}^{2}$ for every $j \in \mathbb{N}_{+}$. 
	Here, $M_{1}$ is a constant introduced in \textbf{Step 2}. 
	This indicates the existence of a subsequence of $\{ \pmb{g}^{(j)} \}_{j=1}^{+\infty}$ (not relabeled) satisfying $\pmb{g}^{(j)} \rightharpoonup \pmb{g}^{*}$ in $L^{2}([0,T];\mathbb{R}^{n})$ as $j \to +\infty$.
	
	\textbf{Step 3-2:} We will prove in this substep that in $L^{2}((0,T);\mathbb{R}^{n})$, as $j \to +\infty$,
	\begin{align}\label{eq-j-sum-weak-converge}
		\dot{\pmb{x}}^{(j)} + \pmb{\alpha}^{(j)} (\pmb{A}^{(j)})^{\top} (\pmb{A}^{(j)} \pmb{x}^{(j)} - b) + \pmb{\alpha}^{(j)} \pmb{\lambda}^{(j)} \pmb{g}^{(j)}  
		\rightharpoonup \dot{\pmb{x}}^{*} 
		+ \pmb{\alpha}^{*} (\pmb{A}^{*})^{\top} (\pmb{A}^{*} \pmb{x}^{*} - b) 
		+ \pmb{\alpha}^{*} \pmb{\lambda}^{*} \pmb{g}^{*}.
	\end{align}
	Since $\dot{\pmb{x}}^{(j)} \rightharpoonup \dot{\pmb{x}}^{*}$ in $L^{2}((0,T);\mathbb{R}^{n})$ is known from Eq.\eqref{eq-derivative-weak-convergence}, we next show 
	\begin{align}
		\label{eq-weak-converge-1}& \pmb{\alpha}^{(j)} (\pmb{A}^{(j)})^{\top} (\pmb{A}^{(j)} \pmb{x}^{(j)} - b) \to \pmb{\alpha}^{*} (\pmb{A}^{*})^{\top} (\pmb{A}^{*} \pmb{x}^{*} - b),  \\
		\label{eq-weak-converge-2}& \pmb{\alpha}^{(j)} \pmb{\lambda}^{(j)} \pmb{g}^{(j)} \rightharpoonup \pmb{\alpha}^{*} \pmb{\lambda}^{*} \pmb{g}^{*},
	\end{align}
	as $j \to +\infty$ in $L^{2}([0,T];\mathbb{R}^{n})$.
	Using the inequality $(\sum_{i=1}^{n} a_{i})^{p} \leq n^{p-1} \sum_{i=1}^{n} a_{i}^{p}$,
	Eq.\eqref{eq-weak-converge-1} is deduced by 
	\begin{align*}
		& \|\pmb{\alpha}^{(j)} (\pmb{A}^{(j)})^{\top} (\pmb{A}^{(j)} \pmb{x}^{(j)} - b) - \pmb{\alpha}^{*} (\pmb{A}^{*})^{\top} (\pmb{A}^{*} \pmb{x}^{*} - b)\|_{L^{2}([0,T];\mathbb{R}^{n})}^{2}  \\
		=& \int_{[0,T]} \|\pmb{\alpha}^{(j)}(t) (\pmb{A}^{(j)}(t))^{\top} (\pmb{A}^{(j)}(t) \pmb{x}^{(j)}(t) - b) - \pmb{\alpha}^{*}(t) (\pmb{A}^{*}(t))^{\top} (\pmb{A}^{*}(t) \pmb{x}^{*}(t) - b)\|_{2}^{2} {\rm d}t  \\
		\leq& 4 \int_{[0,T]} \|\pmb{\alpha}^{(j)}(t) (\pmb{A}^{(j)}(t))^{\top} (\pmb{A}^{(j)}(t) \pmb{x}^{(j)}(t) - b) 
		- \pmb{\alpha}^{*}(t) (\pmb{A}^{(j)}(t))^{\top} (\pmb{A}^{(j)}(t) \pmb{x}^{(j)}(t) - b)\|_{2}^{2} {\rm d}t   \\
		&+ 4 \int_{[0,T]} \|\pmb{\alpha}^{*}(t) (\pmb{A}^{(j)}(t))^{\top} (\pmb{A}^{(j)}(t) \pmb{x}^{(j)}(t) - b) 
		- \pmb{\alpha}^{*}(t) (\pmb{A}^{*}(t))^{\top} (\pmb{A}^{(j)}(t) \pmb{x}^{(j)}(t) - b)\|_{2}^{2} {\rm d}t   \\
		&+ 4 \int_{[0,T]} \|\pmb{\alpha}^{*}(t) (\pmb{A}^{*}(t))^{\top} \pmb{A}^{(j)}(t) \pmb{x}^{(j)}(t) 
		- \pmb{\alpha}^{*}(t) (\pmb{A}^{*}(t))^{\top} \pmb{A}^{*}(t) \pmb{x}^{(j)}(t)\|_{2}^{2} {\rm d}t   \\
		&+ 4 \int_{[0,T]} \|\pmb{\alpha}^{*}(t) (\pmb{A}^{*}(t))^{\top} \pmb{A}^{*}(t) \pmb{x}^{(j)}(t) 
		- \pmb{\alpha}^{*}(t) (\pmb{A}^{*}(t))^{\top} \pmb{A}^{*}(t) \pmb{x}^{*}(t)\|_{2}^{2} {\rm d}t   \\
		\leq& 4 \int_{[0,T]} \|\pmb{\alpha}^{(j)}(t) - \pmb{\alpha}^{*}(t)\|_{2} \cdot 1 {\rm d}t \cdot 2 M_{\pmb{\alpha}} \cdot M_{\pmb{A}}^{2} (M_{\pmb{A}} M_{0} + \|b\|)^{2}   \\
		&+ 4 \int_{[0,T]} \|\pmb{A}^{(j)}(t) - \pmb{A}^{*}(t)\|_{2} \cdot 1 {\rm d}t \cdot 2 M_{\pmb{A}} \cdot M_{\pmb{\alpha}}^{2} (M_{\pmb{A}} M_{0} + \|b\|)^{2}   \\
		&+ 4 \int_{[0,T]} \|\pmb{A}^{(j)}(t) - \pmb{A}^{*}(t)\|_{2} \cdot 1 {\rm d}t \cdot 2 M_{\pmb{A}} \cdot  M_{\pmb{\alpha}}^{2} M_{\pmb{A}}^{2}  M_{0}^{2}   \\
		&+ 4 \int_{[0,T]} \|\pmb{x}^{(j)}(t) - \pmb{x}^{*}(t)\|_{2} \cdot 1 {\rm d}t \cdot 2 M_{0} \cdot M_{\pmb{\alpha}}^{2} M_{\pmb{A}}^{4}   \\
		\leq& 4 \|\pmb{\alpha}^{(j)} - \pmb{\alpha}^{*}\|_{L^{p}([0,T])} \cdot T^{\frac{1}{q}} \cdot 2M_{\pmb{\alpha}} \cdot M_{\pmb{A}}^{2} (M_{\pmb{A}} M_{0} + \|b\|)^{2}   \\
		&+ 4 \|\pmb{A}^{(j)} - \pmb{A}^{*}\|_{L^{p}([0,T]; \mathbb{R}^{m \times n})} \cdot T^{\frac{1}{q}} \cdot 2M_{\pmb{A}} \cdot M_{\pmb{\alpha}}^{2} (M_{\pmb{A}} M_{0} + \|b\|)^{2}   \\
		&+ 4 \|\pmb{A}^{(j)} - \pmb{A}^{*}\|_{L^{p}([0,T]; \mathbb{R}^{m \times n})} \cdot T^{\frac{1}{q}} \cdot 2M_{\pmb{A}} \cdot M_{\pmb{\alpha}}^{2} M_{\pmb{A}}^{2}  M_{0}^{2}   \\
		&+ 4 \|\pmb{x}^{(j)} - \pmb{x}^{*}\|_{L^{p}([0,T]; \mathbb{R}^{n})} \cdot T^{\frac{1}{q}} \cdot 2M_{0} \cdot M_{\pmb{\alpha}}^{2} M_{\pmb{A}}^{4}   \\
		\to& 0 + 0 + 0 + 0 = 0, \quad j \to +\infty.
	\end{align*}		
	In addition, Eq.\eqref{eq-weak-converge-2} is due to that for every $\pmb{v} \in L^{2}([0,T];\mathbb{R}^{n})$,
	\begin{align*}
		& |\langle \pmb{\alpha}^{(j)} \pmb{\lambda}^{(j)} \pmb{g}^{(j)} - \pmb{\alpha}^{*} \pmb{\lambda}^{*} \pmb{g}^{*}, \pmb{v} \rangle|   
		= \Big| \int_{[0,T]} \langle \pmb{\alpha}^{(j)}(t) \pmb{\lambda}^{(j)}(t) \pmb{g}^{(j)}(t) - \pmb{\alpha}^{*}(t) \pmb{\lambda}^{*}(t) \pmb{g}^{*}(t), \pmb{v}(t) \rangle {\rm d}t \Big|  \\
		\leq&  \int_{[0,T]} \left| \langle \pmb{\alpha}^{(j)}(t) \pmb{\lambda}^{(j)}(t) \pmb{g}^{(j)}(t) - \pmb{\alpha}^{*}(t) \pmb{\lambda}^{(j)}(t) \pmb{g}^{(j)}(t), \pmb{v}(t) \rangle \right| {\rm d}t   \\
		&+ \int_{[0,T]} \left| \langle \pmb{\alpha}^{*}(t) \pmb{\lambda}^{(j)}(t) \pmb{g}^{(j)}(t) - \pmb{\alpha}^{*}(t) \pmb{\lambda}^{*}(t) \pmb{g}^{(j)}(t), \pmb{v}(t) \rangle \right| {\rm d}t   \\
		&+ \Big| \int_{[0,T]} \langle \pmb{\alpha}^{*}(t) \pmb{\lambda}^{*}(t) \pmb{g}^{(j)}(t) - \pmb{\alpha}^{*}(t) \pmb{\lambda}^{*}(t) \pmb{g}^{*}(t), \pmb{v}(t) \rangle {\rm d}t \Big|  \\
		\leq& \|\pmb{\alpha}^{(j)} - \pmb{\alpha}^{*}\|_{L^{p}([0,T])}^{\frac{1}{2}} \cdot T^{\frac{1}{2q}} \cdot 2^{\frac{1}{2}} M_{\pmb{\alpha}}^{\frac{1}{2}} \cdot M_{\pmb{\lambda}} M_{1} \|\pmb{v}\|_{L^{2}([0,T];\mathbb{R}^{n})}   \\
		&+ \|\pmb{\lambda}^{(j)} - \pmb{\lambda}^{*}\|_{L^{p}([0,T])}^{\frac{1}{2}} \cdot T^{\frac{1}{2q}} \cdot 2^{\frac{1}{2}} M_{\pmb{\lambda}}^{\frac{1}{2}} \cdot M_{\pmb{\alpha}} M_{1} \|\pmb{v}\|_{L^{2}([0,T];\mathbb{R}^{n})}   \\
		&+ \Big| \int_{[0,T]} \langle \pmb{g}^{(j)}(t) - \pmb{g}^{*}(t), \pmb{\alpha}^{*}(t) \pmb{\lambda}^{*}(t) \pmb{v}(t) \rangle {\rm d}t \Big|  \\
		\to& 0 + 0 + 0 = 0, \quad j \to +\infty,
	\end{align*}
	where we used \textbf{H{\"o}lder} inequality and $\pmb{g}^{(j)} \rightharpoonup \pmb{g}^{*}$ in $L^{2}([0,T];\mathbb{R}^{n})$.
	
	One can then follow the last part of the proof framework of \textbf{Theorem} 2.1 in \cite[pp.10-11]{elliott1985convergence} to check that $\pmb{x}^{*}$ satisfies
	\begin{equation}\label{eq-gamma-*}
		\dot{\pmb{x}}^{*}(t) + \pmb{\alpha}^{*}(t) (\pmb{A}^{*}(t))^{\top} (\pmb{A}^{*}(t) \pmb{x}^{*}(t) - b) + \pmb{\alpha}^{*}(t) \pmb{\lambda}^{*}(t) \pmb{g}^{*}(t) = \vec{0}, \quad {\rm a.e. \,\,} t \in [0,T],
	\end{equation}
	and $\pmb{g}^{*}$ is a.e. a selection of $\partial \mathcal{R} \circ \pmb{x}^{*}$, that is, $\pmb{g}^{*}(t) \in (\partial \mathcal{R}) (\pmb{x}^{*}(t))$ for a.e. $t \in [0,T]$.
	
	\textbf{Step 4:} Finally, we check that $\lim_{j \to +\infty} \mathcal{J}(\pmb{A}^{(j)}, \pmb{\alpha}^{(j)}, \pmb{\lambda}^{(j)}) = \mathcal{J}(\pmb{A}^{*}, \pmb{\alpha}^{*}, \pmb{\lambda}^{*})$.
	Note that
	\begin{align*}
		& \Big| \mathcal{J}(\pmb{A}^{(j)}, \pmb{\alpha}^{(j)}, \pmb{\lambda}^{(j)}) - \mathcal{J}(\pmb{A}^{*}, \pmb{\alpha}^{*}, \pmb{\lambda}^{*}) \Big|   \\
		\leq& \Big| \mathcal{L}(\pmb{x}^{(j)}(T); y) - \mathcal{L}(\pmb{x}^{*}(T); y) \Big| 
		+ \beta_{1} \Big| \mathcal{H}^{(1)}(\pmb{A}^{(j)}) - \mathcal{H}^{(1)}(\pmb{A}^{*}) \Big|  \\
		&+ \beta_{2} \Big| \mathcal{H}^{(2)}(\pmb{\alpha}^{(j)}) - \mathcal{H}^{(2)}(\pmb{\alpha}^{*}) \Big| 
		+ \beta_{3} \Big| \mathcal{H}^{(3)}(\pmb{\lambda}^{(j)}) - \mathcal{H}^{(3)}(\pmb{\lambda}^{*}) \Big|.
	\end{align*}
	By the assumption (A4) and $\lim_{j \to +\infty} \pmb{x}^{(j)}(T) = \pmb{x}^{*}(T)$, one has $\lim_{j \to +\infty} |\mathcal{L}(\pmb{x}^{(j)}(T); y) - \mathcal{L}(\pmb{x}^{*}(T); y)| = 0$.
	Besides, since 
	\begin{align*}
		\Big| \|\pmb{A}^{(j)}\|_{L^{p}([0,T];\mathbb{R}^{m \times n})} - \|\pmb{A}^{*}\|_{L^{p}([0,T];\mathbb{R}^{m \times n})} \Big|
		\leq \|\pmb{A}^{(j)} - \pmb{A}^{*}\|_{L^{p}([0,T];\mathbb{R}^{m \times n})}
		\to 0, \quad j \to +\infty,
	\end{align*}
	one has, by the assumption (A5), that, as $j \to +\infty$, 
	\begin{align*}
		\Big| \mathcal{H}^{(1)}(\pmb{A}^{(j)}) - \mathcal{H}^{(1)}(\pmb{A}^{*}) \Big|
		= \Big| \psi\Big( \frac{1}{T} \|\pmb{A}^{(j)}\|_{L^{p}([0,T];\mathbb{R}^{m \times n})}^{p} \Big) - \psi\Big( \frac{1}{T} \|\pmb{A}^{*}\|_{L^{p}([0,T];\mathbb{R}^{m \times n})}^{p} \Big) \Big|
		\to 0,
	\end{align*}
	and similarly $\lim_{j \to +\infty} |\mathcal{H}^{(2)}(\pmb{\alpha}^{(j)}) - \mathcal{H}^{(2)}(\pmb{\alpha}^{*})| = 0$, $\lim_{j \to +\infty} |\mathcal{H}^{(3)}(\pmb{\lambda}^{(j)}) - \mathcal{H}^{(3)}(\pmb{\lambda}^{*})| = 0$.
	Therefore, 
	\begin{align*}
		& \Big| \lim_{j \to +\infty} \mathcal{J}(\pmb{A}^{(j)}, \pmb{\alpha}^{(j)}, \pmb{\lambda}^{(j)}) - \mathcal{J}(\pmb{A}^{*}, \pmb{\alpha}^{*}, \pmb{\lambda}^{*}) \Big|   \\
		\leq& \lim_{j \to +\infty} \Big |\mathcal{J}(\pmb{A}^{(j)}, \pmb{\alpha}^{(j)}, \pmb{\lambda}^{(j)}) - \mathcal{J}(\pmb{A}^{*}, \pmb{\alpha}^{*}, \pmb{\lambda}^{*}) \Big|   \\
		\leq& \lim_{j \to +\infty} \Big| \mathcal{L}(\pmb{x}^{(j)}(T); y) - \mathcal{L}(\pmb{x}^{*}(T); y) \Big| 
		+ \lim_{j \to +\infty} \beta_{1} \Big| \mathcal{H}^{(1)}(\pmb{A}^{(j)}) - \mathcal{H}^{(1)}(\pmb{A}^{*}) \Big|  \\
		&+ \lim_{j \to +\infty} \beta_{2} \Big| \mathcal{H}^{(2)}(\pmb{\alpha}^{(j)}) - \mathcal{H}^{(2)}(\pmb{\alpha}^{*}) \Big| 
		+ \lim_{j \to +\infty} \beta_{3} \Big| \mathcal{H}^{(3)}(\pmb{\lambda}^{(j)}) - \mathcal{H}^{(3)}(\pmb{\lambda}^{*}) \Big|  \\
		=& 0 + 0 + 0 + 0 = 0,
	\end{align*}
	which completes the proof.
\end{proof}

\begin{remark}
	Of course, one can assume "compactness in $L^{\infty}$ topology" to show the same result, which seems simplifying the proof.
	However, this assumption is too restrictive.
	Consider a sequence $\{ \mathcal{I}_{3^{N}} \mathcal{P}_{3^{N}} \pmb{\alpha}^{\#} \}_{N=1}^{+\infty}$ with $\pmb{\alpha}^{\#}(t) := \pmb{1}_{[\frac{T}{2},T]}(t)$ for every $t \in [0,T]$.
	We see that it has no convergent subsequence in $L^{\infty}$ topology by \cite[Chapter IV, Section 8, Thm.18, pp.297]{dunford1988linear}.
\end{remark}

In the following discussion, we denote the solution set of $(\mathfrak{Q})$ as $S$, and by \textbf{Theorem} \ref{thm-existence-of-solution-to-optimal-control-continuous}, $S \neq \emptyset$.

\subsection{Convergence properties of the learning problem of the basic FBS-network to that of its deep-layer limit system}

In this subsection, we establish our main general convergence result of the objective functional of the optimal control problem of the FBS-network to that of the deep-layer limit analog, which implies the $\Gamma$-convergence, and thus the convergence property of minimizers of learning problems.

\begin{theorem}\label{thm-convergence-objective-functions}
	(Convergence property of objective functionals of learning problems)
	Suppose the assumptions (A1)-(A5) hold, and $p \in [1,+\infty)$. 
	Let the observed data $b \in \mathbb{R}^{m}$, the initial value $x^{0} \in \mathbb{R}^{n}$, and the label $y \in \mathbb{R}^{n}$.
	If $\pmb{\mathscr{F}} \subset L^{\infty}([0,T]; \mathbb{R}^{m \times n} \times \mathbb{R} \times \mathbb{R})$ is nonempty and bounded,
	then for any point $(\widetilde{\pmb{A}}, \widetilde{\pmb{\alpha}}, \widetilde{\pmb{\lambda}}) \in \pmb{\mathscr{F}}$ and any sequence $\{ (\pmb{A}^{(N)}, \pmb{\alpha}^{(N)}, \pmb{\lambda}^{(N)}) \}_{N=1}^{+\infty} \subset \pmb{\mathscr{F}}$ satisfying $(\pmb{A}^{(N)}, \pmb{\alpha}^{(N)}, \pmb{\lambda}^{(N)}) \overset{L^{p}}{\to} (\widetilde{\pmb{A}}, \widetilde{\pmb{\alpha}}, \widetilde{\pmb{\lambda}})$ as $N \to +\infty$, $\mathcal{J}(\widetilde{\pmb{A}}, \widetilde{\pmb{\alpha}}, \widetilde{\pmb{\lambda}}) = \lim\limits_{N \to +\infty} (\mathcal{J}_{N} \circ \mathcal{P}_{N})(\pmb{A}^{(N)}, \pmb{\alpha}^{(N)}, \pmb{\lambda}^{(N)})$ holds.
\end{theorem}

\begin{proof}[Proof. ]
	For a point $(\pmb{A}^{(N)}, \pmb{\alpha}^{(N)}, \pmb{\lambda}^{(N)}) \in \pmb{\mathscr{F}}$, we denote $\{ x^{N,k} \}_{k=0}^{N}$ as the network state of the FBS-network \eqref{eq-system-inclusion-discrete} with parameters $(\mathcal{P}_{N} \pmb{A}^{(N)}, \mathcal{P}_{N} \pmb{\alpha}^{(N)}, \mathcal{P}_{N} \pmb{\lambda}^{(N)})$. 
	We then define the piecewise linear mapping $\pmb{x}^{N}$ as
	\begin{align*}
		\pmb{x}^{N}(t) := x^{N,k} + \frac{N}{T} (x^{N,k+1} - x^{N,k})(t - t^{N,k}), \,\, \forall t \in [t^{N,k},t^{N,k+1}], \,\, \forall k \in \{0, 1, \ldots, N-1\}.
	\end{align*}
	
	For any point $(\widetilde{\pmb{A}}, \widetilde{\pmb{\alpha}}, \widetilde{\pmb{\lambda}}) \in \pmb{\mathscr{F}}$ and any sequence $\{ (\pmb{A}^{(N)}, \pmb{\alpha}^{(N)}, \pmb{\lambda}^{(N)}) \}_{N=1}^{+\infty} \subset \pmb{\mathscr{F}}$ satisfying $(\pmb{A}^{(N)}, \pmb{\alpha}^{(N)}, \pmb{\lambda}^{(N)}) \overset{L^{p}}{\to} (\widetilde{\pmb{A}}, \widetilde{\pmb{\alpha}}, \widetilde{\pmb{\lambda}})$ as $N \to +\infty$, we first derive 
	\begin{align*}
		& \left| \|\mathcal{I}_{N} \mathcal{P}_{N} \pmb{A}^{(N)}\|_{L^{p}([0,T];\mathbb{R}^{m \times n})} - \|\widetilde{\pmb{A}}\|_{L^{p}([0,T];\mathbb{R}^{m \times n})} \right|  
		\leq \|\mathcal{I}_{N} \mathcal{P}_{N} \pmb{A}^{(N)} - \widetilde{\pmb{A}}\|_{L^{p}([0,T];\mathbb{R}^{m \times n})}   \\
		\leq& \|\mathcal{I}_{N} \mathcal{P}_{N} \pmb{A}^{(N)} - \mathcal{I}_{N} \mathcal{P}_{N} \widetilde{\pmb{A}}\|_{L^{p}([0,T];\mathbb{R}^{m \times n})}
		+ \|\mathcal{I}_{N} \mathcal{P}_{N} \widetilde{\pmb{A}} - \widetilde{\pmb{A}}\|_{L^{p}([0,T];\mathbb{R}^{m \times n})}  \\
		\leq& \|\pmb{A}^{(N)} - \widetilde{\pmb{A}}\|_{L^{p}([0,T];\mathbb{R}^{m \times n})}
		+ \|\mathcal{I}_{N} \mathcal{P}_{N} \widetilde{\pmb{A}} - \widetilde{\pmb{A}}\|_{L^{p}([0,T];\mathbb{R}^{m \times n})}  \\
		\to& 0 + 0 = 0, \quad N \to +\infty,
	\end{align*}
	where the last two steps use \cite[Prop.3.7]{lin2025deep}.	 
	Similarly, we obtain $\lim_{N \to +\infty} \|\mathcal{I}_{N} \mathcal{P}_{N} \pmb{\alpha}^{(N)} - \widetilde{\pmb{\alpha}}\|_{L^{p}([0,T])} = 0$, $\lim_{N \to +\infty} \|\mathcal{I}_{N} \mathcal{P}_{N} \pmb{\lambda}^{(N)} - \widetilde{\pmb{\lambda}}\|_{L^{p}([0,T])} = 0$.	
	Noting that 
	\begin{align*}
		(\mathcal{H}_{N}^{(1)} \circ \mathcal{P}_{N}) (\pmb{A}^{(N)})
		=& \psi\Big( \frac{1}{N} \sum_{k=0}^{N-1} \| (\mathcal{P}_{N} \pmb{A}^{(N)})_{k} \|_{2}^{p} \Big) 
		= \psi\Big( \frac{1}{T} \sum_{k=0}^{N-1} h_{N} \| (\mathcal{P}_{N} \pmb{A}^{(N)})_{k} \|_{2}^{p} \Big)  \\
		=& \psi\Big( \frac{1}{T} \|\mathcal{I}_{N} \mathcal{P}_{N} \pmb{A}^{(N)}\|_{L^{p}([0,T];\mathbb{R}^{m \times n})}^{p} \Big),
	\end{align*}
	we see that as $N \to +\infty$, 
	\begin{align*}
		(\mathcal{H}_{N}^{(1)} \circ \mathcal{P}_{N}) (\pmb{A}^{(N)})
		= \psi\Big( \frac{1}{T} \|\mathcal{I}_{N} \mathcal{P}_{N} \pmb{A}^{(N)}\|_{L^{p}([0,T];\mathbb{R}^{m \times n})}^{p} \Big) 
		\to \psi\Big( \frac{1}{T} \|\widetilde{\pmb{A}}\|_{L^{p}([0,T];\mathbb{R}^{m \times n})}^{p} \Big)
		= \mathcal{H}^{(1)}(\widetilde{\pmb{A}}),
	\end{align*}
	according to the assumption (A5).
	Similarly, $(\mathcal{H}_{N}^{(2)} \circ \mathcal{P}_{N}) (\pmb{\alpha}^{(N)}) \to \mathcal{H}^{(2)}(\widetilde{\pmb{\alpha}})$, $(\mathcal{H}_{N}^{(3)} \circ \mathcal{P}_{N}) (\pmb{\lambda}^{(N)}) \to \mathcal{H}^{(3)}(\widetilde{\pmb{\lambda}})$ as $N \to +\infty$.

	Besides, since $\{ \pmb{A}^{(N)} \}_{N=1}^{+\infty}$ is bounded in $L^{\infty}$ space and $\pmb{A}^{(N)} \overset{L^{p}}{\to} \widetilde{\pmb{A}}$,  $\pmb{\alpha}^{(N)} \overset{L^{p}}{\to} \widetilde{\pmb{\alpha}}$, $\pmb{\lambda}^{(N)} \overset{L^{p}}{\to} \widetilde{\pmb{\lambda}}$, we obtain by \cite[Cor.3.9]{lin2025deep} that $\pmb{x}^{N}$, which is constructed by $\{ x^{N,k} \}_{k=0}^{N}$ satisfying the basic FBS-network \eqref{eq-system-inclusion-discrete} with parameters $(\mathcal{P}_{N} \pmb{A}^{(N)}, \mathcal{P}_{N} \pmb{\alpha}^{(N)}, \mathcal{P}_{N} \pmb{\lambda}^{(N)})$, converges uniformly to the unique solution $\widetilde{\pmb{x}}$ of the differential inclusion \eqref{eq-system-inclusion-continuous} with parameters $(\widetilde{\pmb{A}}, \widetilde{\pmb{\alpha}}, \widetilde{\pmb{\lambda}})$. 
	It follows that $\pmb{x}^{N}(T) \to \widetilde{\pmb{x}}(T)$ as $N \to +\infty$.
	Hence, by the assumption (A4), one has
	\begin{align*}
		\mathcal{L}(x^{N,N}; y) 
		= \mathcal{L}(\pmb{x}^{N}(T); y)
		\to \mathcal{L}(\widetilde{\pmb{x}}(T); y), \quad N \to +\infty.
	\end{align*}

	Therefore, we obtain
	\begin{align*}
		& \lim_{N \to +\infty} (\mathcal{J}_{N} \circ \mathcal{P}_{N}) ( \pmb{A}^{(N)}, \pmb{\alpha}^{(N)}, \pmb{\lambda}^{(N)})   \\
		=& \lim_{N \to +\infty} \Big[ \mathcal{L}(\pmb{x}^{N}(T); y)
		+ \beta_{1} (\mathcal{H}_{N}^{(1)} \circ \mathcal{P}_{N}) (\pmb{A}^{(N)}) 
		+ \beta_{2} (\mathcal{H}_{N}^{(2)} \circ \mathcal{P}_{N}) (\pmb{\alpha}^{(N)}) + \beta_{3} (\mathcal{H}_{N}^{(3)} \circ \mathcal{P}_{N}) (\pmb{\lambda}^{(N)})
		\Big]  \\
		=& \mathcal{L}(\widetilde{\pmb{x}}(T); y)
		+ \beta_{1} \mathcal{H}^{(1)}(\widetilde{\pmb{A}}) 
		+ \beta_{2} \mathcal{H}^{(2)}(\widetilde{\pmb{\alpha}}) 
		+ \beta_{3} \mathcal{H}^{(3)}(\widetilde{\pmb{\lambda}})  \\
		=& \mathcal{J}(\widetilde{\pmb{A}}, \widetilde{\pmb{\alpha}}, \widetilde{\pmb{\lambda}}),
	\end{align*}
	which completes the proof.
\end{proof}

\begin{remark}
	This theorem gives a general convergence of objective functional of the learning problems from the discrete- to continuous-time setting, over a relatively more general domain.
	However, we cannot borrow the usual techniques in optimization to derive the convergence of the minimizers directly from this result, because the form of objective functional keeps changing along $N$, and each of the functional is in general not strongly convex.
\end{remark}

The following corollary, although straightforward from \textbf{Theorem} \ref{thm-convergence-objective-functions}, presents the  $\Gamma$-convergence from the objective functional of the FBS-network training problem to that of its continuous-time analog.

\begin{corollary}\label{cor-gamma-convergence-extended-version}
	($\Gamma$-convergence of objective functionals of learning problems)
	Suppose the assumptions (A1)-(A5) hold, and $p \in [1,+\infty)$. 
	Let the observed data $b \in \mathbb{R}^{m}$, the initial value $x^{0} \in \mathbb{R}^{n}$, and the label $y \in \mathbb{R}^{n}$.
	If $\pmb{\mathscr{F}} \subset L^{\infty}([0,T]; \mathbb{R}^{m \times n} \times \mathbb{R} \times \mathbb{R})$ is nonempty and bounded, and it is closed in $L^{p}$ sense, then $\overline{\mathcal{J}}(\cdot, \cdot, \cdot) = \Gamma$-$\lim\limits_{N \to +\infty} \overline{\mathcal{J}_{N} \circ \mathcal{P}_{N}}(\cdot, \cdot, \cdot)$ over $L^{p}([0,T]; \mathbb{R}^{m \times n} \times \mathbb{R} \times \mathbb{R})$, where
	\begin{align*}
		& \overline{\mathcal{J}}(\pmb{A}, \pmb{\alpha}, \pmb{\lambda}) := \left\{\begin{aligned}
			& \mathcal{J}(\pmb{A}, \pmb{\alpha}, \pmb{\lambda}), & (\pmb{A}, \pmb{\alpha}, \pmb{\lambda})& \in \pmb{\mathscr{F}}, \\
			& +\infty, & (\pmb{A}, \pmb{\alpha}, \pmb{\lambda})& \in L^{p}([0,T]; \mathbb{R}^{m \times n} \times \mathbb{R} \times \mathbb{R}) \backslash \pmb{\mathscr{F}},
		\end{aligned}\right.  \\
		& \overline{\mathcal{J}_{N} \circ \mathcal{P}_{N}}(\pmb{A}, \pmb{\alpha}, \pmb{\lambda}) := \left\{\begin{aligned}
			& \mathcal{J}_{N} \circ \mathcal{P}_{N}(\pmb{A}, \pmb{\alpha}, \pmb{\lambda}), & (\pmb{A}, \pmb{\alpha}, \pmb{\lambda})& \in \pmb{\mathscr{F}}, \\
			& +\infty, & (\pmb{A}, \pmb{\alpha}, \pmb{\lambda})& \in L^{p}([0,T]; \mathbb{R}^{m \times n} \times \mathbb{R} \times \mathbb{R}) \backslash \pmb{\mathscr{F}}.
		\end{aligned}\right.
	\end{align*}
\end{corollary}

\begin{proof}[Proof. ]		
	Let us check the "liminf" and "limsup" conditions; see Appendix A.3.
	We will use the closedness of $\pmb{\mathscr{F}}$ in $L^{p}([0,T]; \mathbb{R}^{m \times n} \times \mathbb{R} \times \mathbb{R})$ and \textbf{Theorem} \ref{thm-convergence-objective-functions}.   

    For the "liminf" condition, we arbitrarily take a $(\widetilde{\pmb{A}}, \widetilde{\pmb{\alpha}}, \widetilde{\pmb{\lambda}})\in L^{p}([0,T]; \mathbb{R}^{m \times n} \times \mathbb{R} \times \mathbb{R})$ and a sequence  $\{ (\pmb{A}^{(N)}, \pmb{\alpha}^{(N)}, \pmb{\lambda}^{(N)}) \}_{N=1}^{+\infty} \subset L^{p}([0,T]; \mathbb{R}^{m \times n} \times \mathbb{R} \times \mathbb{R})$ satisfying $(\pmb{A}^{(N)}, \pmb{\alpha}^{(N)}, \pmb{\lambda}^{(N)}) \overset{L^{p}}{\to} (\widetilde{\pmb{A}}, \widetilde{\pmb{\alpha}}, \widetilde{\pmb{\lambda}})$ as $N \to +\infty$. If $(\widetilde{\pmb{A}}, \widetilde{\pmb{\alpha}}, \widetilde{\pmb{\lambda}}) \in L^{p}([0,T]; \mathbb{R}^{m \times n} \times \mathbb{R} \times \mathbb{R}) \backslash \pmb{\mathscr{F}}$, then the "liminf" condition holds trivially.
    If $(\widetilde{\pmb{A}}, \widetilde{\pmb{\alpha}}, \widetilde{\pmb{\lambda}}) \in \pmb{\mathscr{F}}$, it is enough to consider the case with $\{ (\pmb{A}^{(N)}, \pmb{\alpha}^{(N)}, \pmb{\lambda}^{(N)}) \}_{N=1}^{+\infty} \subseteq \pmb{\mathscr{F}}$ by the definition of $\overline{\mathcal{J}_{N} \circ \mathcal{P}_{N}}$, and the "liminf" condition also holds due to \textbf{Theorem} \ref{thm-convergence-objective-functions}.
	
    For the "limsup" condition, we arbitrarily take a $(\widetilde{\pmb{A}}, \widetilde{\pmb{\alpha}}, \widetilde{\pmb{\lambda}})\in L^{p}([0,T]; \mathbb{R}^{m \times n} \times \mathbb{R} \times \mathbb{R})$. If $(\widetilde{\pmb{A}}, \widetilde{\pmb{\alpha}}, \widetilde{\pmb{\lambda}}) \in L^{p}([0,T]; \mathbb{R}^{m \times n} \times \mathbb{R} \times \mathbb{R}) \backslash \pmb{\mathscr{F}}$, then again the ``limsup" condition holds trivially. If $(\widetilde{\pmb{A}}, \widetilde{\pmb{\alpha}}, \widetilde{\pmb{\lambda}}) \in\pmb{\mathscr{F}}$, we can choose a sequence  $\{ (\pmb{A}^{(N)}, \pmb{\alpha}^{(N)}, \pmb{\lambda}^{(N)}) \}_{N=1}^{+\infty} \subset \pmb{\mathscr{F}}$ satisfying $(\pmb{A}^{(N)}, \pmb{\alpha}^{(N)}, \pmb{\lambda}^{(N)}) \overset{L^{p}}{\to} (\widetilde{\pmb{A}}, \widetilde{\pmb{\alpha}}, \widetilde{\pmb{\lambda}})$ as $N \to +\infty$, and the ``limsup" condition holds again by using \textbf{Theorem} \ref{thm-convergence-objective-functions}.
\end{proof}

So far, the convergence of the objective function of the learning problem associated with the basic FBS-network to that of its corresponding deep limit system has been established. 
In what follows, we build connections between the minimizers of $\mathcal{J}_{N}$ and $\mathcal{J}_{N} \circ \mathcal{P}_{N}$, and combine them with the above $\Gamma$-convergence to derive the convergence properties of the learning problem minimizers of the basic FBS-network to those of its continuous-time counterpart.

For convenience of description, we denote 
\begin{align*}
	\mathcal{P}_{N} \pmb{\mathscr{F}} := \{ (\mathcal{P}_{N} \pmb{A}, \mathcal{P}_{N} \pmb{\alpha}, \mathcal{P}_{N} \pmb{\lambda}): (\pmb{A}, \pmb{\alpha}, \pmb{\lambda}) \in \pmb{\mathscr{F}} \},
\end{align*}
for $\pmb{\mathscr{F}} \subseteq L^{\infty}([0,T]; \mathbb{R}^{m \times n} \times \mathbb{R} \times \mathbb{R})$, 
and 
\begin{align*}
	& \mathcal{I}_{N} \mathscr{F}_{N} := \{ (\mathcal{I}_{N} A^{N,:}, \mathcal{I}_{N} \alpha^{N,:}, \mathcal{I}_{N} \lambda^{N,:}): (A^{N,:}, \alpha^{N,:}, \lambda^{N,:}) \in \mathscr{F}_{N} \},  \\
	& [\mathscr{F}_{N}] := \{ (\pmb{A}, \pmb{\alpha}, \pmb{\lambda}): (\mathcal{P}_{N} \pmb{A}, \mathcal{P}_{N} \pmb{\alpha}, \mathcal{P}_{N} \pmb{\lambda}) \in \mathscr{F}_{N} \},
\end{align*}
for $\mathscr{F}_{N} \subseteq (\mathbb{R}^{m \times n})^{N} \times \mathbb{R}^{N} \times \mathbb{R}^{N}$, $N \in \mathbb{N}_{+}$.
The following proposition and corollary build connections between the minimizers of $\mathcal{J}_{N}$ and $\mathcal{J}_{N} \circ \mathcal{P}_{N}$.


\begin{proposition}\label{prop-minimizer-extension}
	Suppose $\pmb{\mathscr{F}} \subseteq L^{\infty}([0,T]; \mathbb{R}^{m \times n} \times \mathbb{R} \times \mathbb{R})$.
	Let $[\mathcal{P}_{N} \pmb{\mathscr{F}}] \subseteq \pmb{\mathscr{F}}$ for any $N \in \mathbb{N}_{+}$.
	Then the point $((A^{N,:})^{*}, (\alpha^{N,:})^{*}, (\lambda^{N,:})^{*}) \in (\mathbb{R}^{m \times n})^{N} \times \mathbb{R}^{N} \times \mathbb{R}^{N}$ is a minimizer of $\mathcal{J}_{N}(\cdot, \cdot, \cdot)$ in $\mathcal{P}_{N} \pmb{\mathscr{F}}$ if and only if the point $(\pmb{A}^{*,N}, \pmb{\alpha}^{*,N}, \pmb{\lambda}^{*,N}) \in L^{\infty}([0,T]; \mathbb{R}^{m \times n} \times \mathbb{R} \times \mathbb{R})$ satisfying $\mathcal{P}_{N} (\pmb{A}^{*,N}, \pmb{\alpha}^{*,N}, \pmb{\lambda}^{*,N}) = ((A^{N,:})^{*}, (\alpha^{N,:})^{*}, (\lambda^{N,:})^{*})$ is a minimizer of $(\mathcal{J}_{N} \circ \mathcal{P}_{N}) (\cdot, \cdot, \cdot)$ in $\pmb{\mathscr{F}}$.
\end{proposition}

\begin{proof}[Proof. ]
	Our proof is motivated by \cite{huang2024dynamical}. 
	
	"$\Longrightarrow$": We assume by contradiction that the point $((A^{N,:})^{*}, (\alpha^{N,:})^{*}, (\lambda^{N,:})^{*}) \in (\mathbb{R}^{m \times n})^{N} \times \mathbb{R}^{N} \times \mathbb{R}^{N}$ is a minimizer of $\mathcal{J}_{N}(\cdot, \cdot, \cdot)$ in $\mathcal{P}_{N} \pmb{\mathscr{F}}$, but the point $(\pmb{A}^{*,N}, \pmb{\alpha}^{*,N}, \pmb{\lambda}^{*,N}) \in L^{\infty}([0,T]; \mathbb{R}^{m \times n} \times \mathbb{R} \times \mathbb{R})$ satisfying $\mathcal{P}_{N} (\pmb{A}^{*,N}, \pmb{\alpha}^{*,N}, \pmb{\lambda}^{*,N}) = ((A^{N,:})^{*}, (\alpha^{N,:})^{*}, (\lambda^{N,:})^{*})$ is not a minimizer of $(\mathcal{J}_{N} \circ \mathcal{P}_{N}) (\cdot, \cdot, \cdot)$ over $\pmb{\mathscr{F}}$.
	Let $(\pmb{A}^{\#,N}, \pmb{\alpha}^{\#,N}, \pmb{\lambda}^{\#,N})$ be a point in $\pmb{\mathscr{F}}$ satisfying 
	\begin{align*}
		(\mathcal{J}_{N} \circ \mathcal{P}_{N}) (\pmb{A}^{\#,N}, \pmb{\alpha}^{\#,N}, \pmb{\lambda}^{\#,N}) < (\mathcal{J}_{N} \circ \mathcal{P}_{N}) (\pmb{A}^{*,N}, \pmb{\alpha}^{*,N}, \pmb{\lambda}^{*,N}).
	\end{align*}
	We see 
	$(\mathcal{P}_{N} \pmb{A}^{\#,N}, \mathcal{P}_{N} \pmb{\alpha}^{\#,N}, \mathcal{P}_{N} \pmb{\lambda}^{\#,N}) \in \mathcal{P}_{N} \pmb{\mathscr{F}}$, and
	\begin{align*}
		(\mathcal{J}_{N} \circ \mathcal{P}_{N}) (\pmb{A}^{\#,N}, \pmb{\alpha}^{\#,N}, \pmb{\lambda}^{\#,N})
		= \mathcal{J}_{N} (\mathcal{P}_{N} \pmb{A}^{\#,N}, \mathcal{P}_{N} \pmb{\alpha}^{\#,N}, \mathcal{P}_{N} \pmb{\lambda}^{\#,N}).
	\end{align*} 
	According to $\mathcal{P}_{N} (\pmb{A}^{*,N}, \pmb{\alpha}^{*,N}, \pmb{\lambda}^{*,N}) = ((A^{N,:})^{*}, (\alpha^{N,:})^{*}, (\lambda^{N,:})^{*}) \in \mathcal{P}_{N} \pmb{\mathscr{F}}$, one has $(\pmb{A}^{*,N}, \pmb{\alpha}^{*,N}, \pmb{\lambda}^{*,N}) \in [\mathcal{P}_{N} \pmb{\mathscr{F}}] \subseteq \pmb{\mathscr{F}}$, and thus
	\begin{align*}
		(\mathcal{J}_{N} \circ \mathcal{P}_{N}) (\pmb{A}^{\#,N}, \pmb{\alpha}^{\#,N}, \pmb{\lambda}^{\#,N})  
		<& (\mathcal{J}_{N} \circ \mathcal{P}_{N}) (\pmb{A}^{*,N}, \pmb{\alpha}^{*,N}, \pmb{\lambda}^{*,N})  \\
		=& \mathcal{J}_{N}(\mathcal{P}_{N} \pmb{A}^{*,N}, \mathcal{P}_{N} \pmb{\alpha}^{*,N}, \mathcal{P}_{N} \pmb{\lambda}^{*,N})  \\ 
		=& \mathcal{J}_{N}((A^{N,:})^{*}, (\alpha^{N,:})^{*}, (\lambda^{N,:})^{*}),
	\end{align*}
	which means $\mathcal{J}_{N} (\mathcal{P}_{N} \pmb{A}^{\#,N}, \mathcal{P}_{N} \pmb{\alpha}^{\#,N}, \mathcal{P}_{N} \pmb{\lambda}^{\#,N})
	< \mathcal{J}_{N}((A^{N,:})^{*}, (\alpha^{N,:})^{*}, (\lambda^{N,:})^{*})$, indicating a contradiction.
	
	"$\Longleftarrow$": We assume by contradiction that the point $(\pmb{A}^{*,N}, \pmb{\alpha}^{*,N}, \pmb{\lambda}^{*,N}) \in L^{\infty}([0,T]; \mathbb{R}^{m \times n} \times \mathbb{R} \times \mathbb{R})$ satisfying $\mathcal{P}_{N} (\pmb{A}^{*,N}, \pmb{\alpha}^{*,N}, \pmb{\lambda}^{*,N}) = ((A^{N,:})^{*}, (\alpha^{N,:})^{*}, (\lambda^{N,:})^{*})$ is a minimizer of $(\mathcal{J}_{N} \circ \mathcal{P}_{N}) (\cdot, \cdot, \cdot)$ in $\pmb{\mathscr{F}}$, but the point $((A^{N,:})^{*}, (\alpha^{N,:})^{*}, (\lambda^{N,:})^{*}) \in (\mathbb{R}^{m \times n})^{N} \times \mathbb{R}^{N} \times \mathbb{R}^{N}$ is not a minimizer of $\mathcal{J}_{N}(\cdot, \cdot, \cdot)$ in $\mathcal{P}_{N} \pmb{\mathscr{F}}$.
	Let $((A^{N,:})^{\#}, (\alpha^{N,:})^{\#}, (\lambda^{N,:})^{\#})$ be a point in $\mathcal{P}_{N} \pmb{\mathscr{F}}$ satisfying
	\begin{align*}
		\mathcal{J}_{N}((A^{N,:})^{\#}, (\alpha^{N,:})^{\#}, (\lambda^{N,:})^{\#}) < \mathcal{J}_{N}((A^{N,:})^{*}, (\alpha^{N,:})^{*}, (\lambda^{N,:})^{*}).
	\end{align*}
	We see 
	$(\mathcal{I}_{N} (A^{N,:})^{\#}, \mathcal{I}_{N} (\alpha^{N,:})^{\#}, \mathcal{I}_{N} (\lambda^{N,:})^{\#}) \in [\mathcal{P}_{N} \pmb{\mathscr{F}}] \subseteq \pmb{\mathscr{F}}$, and thus
	\begin{align*}
		\mathcal{J}_{N}((A^{N,:})^{\#}, (\alpha^{N,:})^{\#}, (\lambda^{N,:})^{\#})
		=& \mathcal{J}_{N}(\mathcal{P}_{N} \mathcal{I}_{N} (A^{N,:})^{\#}, \mathcal{P}_{N} \mathcal{I}_{N} (\alpha^{N,:})^{\#}, \mathcal{P}_{N} \mathcal{I}_{N} (\lambda^{N,:})^{\#})  \\
		=& (\mathcal{J}_{N} \circ \mathcal{P}_{N}) (\mathcal{I}_{N} (A^{N,:})^{\#}, \mathcal{I}_{N} (\alpha^{N,:})^{\#}, \mathcal{I}_{N} (\lambda^{N,:})^{\#}).
	\end{align*}
	According to $((A^{N,:})^{*}, (\alpha^{N,:})^{*}, (\lambda^{N,:})^{*}) = \mathcal{P}_{N} (\pmb{A}^{*,N}, \pmb{\alpha}^{*,N}, \pmb{\lambda}^{*,N}) \in \mathcal{P}_{N} \pmb{\mathscr{F}}$ and $(\pmb{A}^{*,N}, \pmb{\alpha}^{*,N}, \pmb{\lambda}^{*,N}) \in \pmb{\mathscr{F}}$, one has
	\begin{align*}
		\mathcal{J}_{N}((A^{N,:})^{\#}, (\alpha^{N,:})^{\#}, (\lambda^{N,:})^{\#})  
		<& \mathcal{J}_{N}((A^{N,:})^{*}, (\alpha^{N,:})^{*}, (\lambda^{N,:})^{*})  \\
		=& \mathcal{J}_{N}(\mathcal{P}_{N} \pmb{A}^{*,N}, \mathcal{P}_{N} \pmb{\alpha}^{*,N}, \mathcal{P}_{N} \pmb{\lambda}^{*,N})  \\
		=& (\mathcal{J}_{N} \circ \mathcal{P}_{N}) (\pmb{A}^{*,N}, \pmb{\alpha}^{*,N}, \pmb{\lambda}^{*,N}),
	\end{align*}
	which means $(\mathcal{J}_{N} \circ \mathcal{P}_{N}) (\mathcal{I}_{N} (A^{N,:})^{\#}, \mathcal{I}_{N} (\alpha^{N,:})^{\#}, \mathcal{I}_{N} (\lambda^{N,:})^{\#}) < (\mathcal{J}_{N} \circ \mathcal{P}_{N}) (\pmb{A}^{*,N}, \pmb{\alpha}^{*,N}, \pmb{\lambda}^{*,N})$, leading to a contradiction.
\end{proof}

\begin{corollary}\label{cor-discrete-minimizer-extended-is-composite-minimizer}
	Suppose $\pmb{\mathscr{F}} \subseteq L^{\infty}([0,T]; \mathbb{R}^{m \times n} \times \mathbb{R} \times \mathbb{R})$.
	Let $\mathcal{I}_{N} \mathcal{P}_{N} \pmb{\mathscr{F}} \subseteq \pmb{\mathscr{F}}$ for any $N \in \mathbb{N}_{+}$.
	Then the point $((A^{N,:})^{*}, (\alpha^{N,:})^{*}, (\lambda^{N,:})^{*}) \in (\mathbb{R}^{m \times n})^{N} \times \mathbb{R}^{N} \times \mathbb{R}^{N}$ is a minimizer of $\mathcal{J}_{N}(\cdot, \cdot, \cdot)$ in $\mathcal{P}_{N} \pmb{\mathscr{F}}$ if and only if the point $(\mathcal{I}_{N} (A^{N,:})^{*}, \mathcal{I}_{N} (\alpha^{N,:})^{*}, \mathcal{I}_{N} (\lambda^{N,:})^{*}) \in L^{\infty}([0,T]; \mathbb{R}^{m \times n} \times \mathbb{R} \times \mathbb{R})$ is a minimizer of $(\mathcal{J}_{N} \circ \mathcal{P}_{N}) (\cdot, \cdot, \cdot)$ over $\pmb{\mathscr{F}}$.
\end{corollary}

\begin{proof}[Proof. ]
	"$\Longrightarrow$": We assume by contradiction that the point $((A^{N,:})^{*}, (\alpha^{N,:})^{*}, (\lambda^{N,:})^{*}) \in (\mathbb{R}^{m \times n})^{N} \times \mathbb{R}^{N} \times \mathbb{R}^{N}$ is a minimizer of $\mathcal{J}_{N}(\cdot, \cdot, \cdot)$ in $\mathcal{P}_{N} \pmb{\mathscr{F}}$, but the point $(\mathcal{I}_{N} (A^{N,:})^{*}, \mathcal{I}_{N} (\alpha^{N,:})^{*}, \mathcal{I}_{N} (\lambda^{N,:})^{*}) \in L^{\infty}([0,T]; \mathbb{R}^{m \times n} \times \mathbb{R} \times \mathbb{R})$ is not a minimizer of $(\mathcal{J}_{N} \circ \mathcal{P}_{N}) (\cdot, \cdot, \cdot)$ in $\pmb{\mathscr{F}}$.
	Let $(\pmb{A}^{\#,N}, \pmb{\alpha}^{\#,N}, \pmb{\lambda}^{\#,N})$ in $\pmb{\mathscr{F}}$ satisfying 
	\begin{align*}
		(\mathcal{J}_{N} \circ \mathcal{P}_{N}) (\pmb{A}^{\#,N}, \pmb{\alpha}^{\#,N}, \pmb{\lambda}^{\#,N})  
		< (\mathcal{J}_{N} \circ \mathcal{P}_{N}) (\mathcal{I}_{N} (A^{N,:})^{*}, \mathcal{I}_{N} (\alpha^{N,:})^{*}, \mathcal{I}_{N} (\lambda^{N,:})^{*}).
	\end{align*}
	We see 
	$(\mathcal{P}_{N} \pmb{A}^{\#,N}, \mathcal{P}_{N} \pmb{\alpha}^{\#,N}, \mathcal{P}_{N} \pmb{\lambda}^{\#,N}) \in \mathcal{P}_{N} \pmb{\mathscr{F}}$, and
	\begin{align*}
		(\mathcal{J}_{N} \circ \mathcal{P}_{N}) (\pmb{A}^{\#,N}, \pmb{\alpha}^{\#,N}, \pmb{\lambda}^{\#,N})
		= \mathcal{J}_{N} (\mathcal{P}_{N} \pmb{A}^{\#,N}, \mathcal{P}_{N} \pmb{\alpha}^{\#,N}, \mathcal{P}_{N} \pmb{\lambda}^{\#,N}).
	\end{align*} 
	According to $((A^{N,:})^{*}, (\alpha^{N,:})^{*}, (\lambda^{N,:})^{*}) \in \mathcal{P}_{N} \pmb{\mathscr{F}}$, one has $( \mathcal{I}_{N} (A^{N,:})^{*}, \mathcal{I}_{N} (\alpha^{N,:})^{*}, \mathcal{I}_{N} (\lambda^{N,:})^{*} ) \in \mathcal{I}_{N} \mathcal{P}_{N} \pmb{\mathscr{F}} \subseteq \pmb{\mathscr{F}}$, and thus
	\begin{align*}
		(\mathcal{J}_{N} \circ \mathcal{P}_{N}) (\pmb{A}^{\#,N}, \pmb{\alpha}^{\#,N}, \pmb{\lambda}^{\#,N})  
		<& (\mathcal{J}_{N} \circ \mathcal{P}_{N}) (\mathcal{I}_{N} (A^{N,:})^{*}, \mathcal{I}_{N} (\alpha^{N,:})^{*}, \mathcal{I}_{N} (\lambda^{N,:})^{*})  \\
		=& \mathcal{J}_{N}(\mathcal{P}_{N} \mathcal{I}_{N} (A^{N,:})^{*}, \mathcal{P}_{N} \mathcal{I}_{N} (\alpha^{N,:})^{*}, \mathcal{P}_{N} \mathcal{I}_{N} (\lambda^{N,:})^{*})  \\ 
		=& \mathcal{J}_{N}((A^{N,:})^{*}, (\alpha^{N,:})^{*}, (\lambda^{N,:})^{*}),
	\end{align*}
	which means $\mathcal{J}_{N} (\mathcal{P}_{N} \pmb{A}^{\#,N}, \mathcal{P}_{N} \pmb{\alpha}^{\#,N}, \mathcal{P}_{N} \pmb{\lambda}^{\#,N})
	< \mathcal{J}_{N}((A^{N,:})^{*}, (\alpha^{N,:})^{*}, (\lambda^{N,:})^{*})$, indicating a contradiction.
	
	"$\Longleftarrow$": We assume by contradiction that the point $(\mathcal{I}_{N} (A^{N,:})^{*}, \mathcal{I}_{N} (\alpha^{N,:})^{*}, \mathcal{I}_{N} (\lambda^{N,:})^{*}) \in L^{\infty}([0,T]; \mathbb{R}^{m \times n} \times \mathbb{R} \times \mathbb{R})$ is a minimizer of $(\mathcal{J}_{N} \circ \mathcal{P}_{N}) (\cdot, \cdot, \cdot)$ in $\pmb{\mathscr{F}}$, but the point $((A^{N,:})^{*}, (\alpha^{N,:})^{*}, (\lambda^{N,:})^{*}) \in (\mathbb{R}^{m \times n})^{N} \times \mathbb{R}^{N} \times \mathbb{R}^{N}$ is not a minimizer of $\mathcal{J}_{N}(\cdot, \cdot, \cdot)$ in $\mathcal{P}_{N} \pmb{\mathscr{F}}$.
	Let $((A^{N,:})^{\#}, (\alpha^{N,:})^{\#}, (\lambda^{N,:})^{\#})$ in $\mathcal{P}_{N} \pmb{\mathscr{F}}$ satisfying
	\begin{align*}
		\mathcal{J}_{N}((A^{N,:})^{\#}, (\alpha^{N,:})^{\#}, (\lambda^{N,:})^{\#}) 
		< \mathcal{J}_{N}((A^{N,:})^{*}, (\alpha^{N,:})^{*}, (\lambda^{N,:})^{*}).
	\end{align*}
	We see 
	$(\mathcal{I}_{N} (A^{N,:})^{\#}, \mathcal{I}_{N} (\alpha^{N,:})^{\#}, \mathcal{I}_{N} (\lambda^{N,:})^{\#}) \in \mathcal{I}_{N} \mathcal{P}_{N} \pmb{\mathscr{F}} \subseteq \pmb{\mathscr{F}}$, and thus
	\begin{align*}
		\mathcal{J}_{N}((A^{N,:})^{\#}, (\alpha^{N,:})^{\#}, (\lambda^{N,:})^{\#})
		=& \mathcal{J}_{N}(\mathcal{P}_{N} \mathcal{I}_{N} (A^{N,:})^{\#}, \mathcal{P}_{N} \mathcal{I}_{N} (\alpha^{N,:})^{\#}, \mathcal{P}_{N} \mathcal{I}_{N} (\lambda^{N,:})^{\#})  \\
		=& (\mathcal{J}_{N} \circ \mathcal{P}_{N}) (\mathcal{I}_{N} (A^{N,:})^{\#}, \mathcal{I}_{N} (\alpha^{N,:})^{\#}, \mathcal{I}_{N} (\lambda^{N,:})^{\#}).
	\end{align*}
	According to $(\mathcal{I}_{N} (A^{N,:})^{*}, \mathcal{I}_{N} (\alpha^{N,:})^{*}, \mathcal{I}_{N} (\lambda^{N,:})^{*}) \in \pmb{\mathscr{F}}$ and $((A^{N,:})^{*}, (\alpha^{N,:})^{*}, (\lambda^{N,:})^{*}) = (\mathcal{P}_{N} \mathcal{I}_{N} (A^{N,:})^{*}, \mathcal{P}_{N} \mathcal{I}_{N} (\alpha^{N,:})^{*}, \mathcal{P}_{N} \mathcal{I}_{N} (\lambda^{N,:})^{*}) \in \mathcal{P}_{N} \pmb{\mathscr{F}}$, one has
	\begin{align*}
		\mathcal{J}_{N}((A^{N,:})^{\#}, (\alpha^{N,:})^{\#}, (\lambda^{N,:})^{\#})  
		<& \mathcal{J}_{N}((A^{N,:})^{*}, (\alpha^{N,:})^{*}, (\lambda^{N,:})^{*})  \\
		=& \mathcal{J}_{N}(\mathcal{P}_{N} \mathcal{I}_{N} (A^{N,:})^{*}, \mathcal{P}_{N} \mathcal{I}_{N} (\alpha^{N,:})^{*}, \mathcal{P}_{N} \mathcal{I}_{N} (\lambda^{N,:})^{*})  \\
		=& (\mathcal{J}_{N} \circ \mathcal{P}_{N}) (\mathcal{I}_{N} (A^{N,:})^{*}, \mathcal{I}_{N} (\alpha^{N,:})^{*}, \mathcal{I}_{N} (\lambda^{N,:})^{*}),
	\end{align*}
	which means $(\mathcal{J}_{N} \circ \mathcal{P}_{N}) (\mathcal{I}_{N} (A^{N,:})^{\#}, \mathcal{I}_{N} (\alpha^{N,:})^{\#}, \mathcal{I}_{N} (\lambda^{N,:})^{\#}) < (\mathcal{J}_{N} \circ \mathcal{P}_{N}) (\mathcal{I}_{N} (A^{N,:})^{*}, \mathcal{I}_{N} (\alpha^{N,:})^{*}, \mathcal{I}_{N} (\lambda^{N,:})^{*})$, leading to a contradiction.
\end{proof}

\begin{remark}
	In fact, one cannot directly derive \textbf{Corollary} \ref{cor-discrete-minimizer-extended-is-composite-minimizer} from \textbf{Proposition} \ref{prop-minimizer-extension}, but the proof framework of \textbf{Corollary} \ref{cor-discrete-minimizer-extended-is-composite-minimizer} is similar to that of \textbf{Proposition} \ref{prop-minimizer-extension}.
\end{remark}

We are now at the position to prove the existence of cluster points of solutions of the problems $\{ (\mathfrak{Q}_{N})  \}_{N=1}^{+\infty}$, each of which is exactly a solution of the problem $(\mathfrak{Q})$ \eqref{eq-continuous-optimal-control-problem-single}.

\begin{theorem} \label{thm-existence-of-cluster-of-minimizers-new}
	(Convergence property of minimizers of learning problems)
	Suppose that the assumptions (A1)-(A5) hold, $p \in [1,+\infty)$, $\pmb{\mathscr{D}}$ is bounded in $L^{\infty}([0,T];\mathbb{R}^{m \times n} \times \mathbb{R} \times \mathbb{R})$ and is closed, equi-continuous in $L^{p}$ sense. 
	Assume $\mathcal{I}_{N} \mathcal{P}_{N} \pmb{\mathscr{D}} \subseteq \pmb{\mathscr{D}}$ for any $N \in \mathbb{N}_{+}$.
	Denote $((A^{N,:})^{*}, (\alpha^{N,:})^{*}, (\lambda^{N,:})^{*}) \in \mathcal{P}_{N} \pmb{\mathscr{D}}$ as a solution of the problem $(\mathfrak{Q}_{N})$ \eqref{eq-discrete-optimal-control-problem-single} in $\mathcal{P}_{N} \pmb{\mathscr{D}}$ for any $N \in \mathbb{N}_{+}$.
	Then there exist a subsequence $\{ ((A^{N_{i},:})^{*}, (\alpha^{N_{i},:})^{*}, (\lambda^{N_{i},:})^{*}) \}_{i=1}^{+\infty}$ and a limit point $(\pmb{A}^{*}, \pmb{\alpha}^{*}, \pmb{\lambda}^{*}) \in \pmb{\mathscr{D}}$ such that
	\begin{align*}
		( \mathcal{I}_{N_{i}} (A^{N_{i},:})^{*}, \mathcal{I}_{N_{i}} (\alpha^{N_{i},:})^{*}, \mathcal{I}_{N_{i}} (\lambda^{N_{i},:})^{*} ) \overset{L^{p}}{\to} ( \pmb{A}^{*}, \pmb{\alpha}^{*}, \pmb{\lambda}^{*} ), 
		\quad {\rm as \,\,} i \to +\infty.
	\end{align*}
	with $(\pmb{A}^{*}, \pmb{\alpha}^{*}, \pmb{\lambda}^{*})$  being a solution to the problem $(\mathfrak{Q})$ \eqref{eq-continuous-optimal-control-problem-single} in $\pmb{\mathscr{D}}$.
\end{theorem}

\begin{proof}[Proof. ]
	We define
	\begin{align*}
		& \widehat{\mathcal{J}}(\pmb{A}, \pmb{\alpha}, \pmb{\lambda}) := \left\{\begin{aligned}
			& \mathcal{J}(\pmb{A}, \pmb{\alpha}, \pmb{\lambda}), & (\pmb{A}, \pmb{\alpha}, \pmb{\lambda})& \in \pmb{\mathscr{D}}, \\
			& +\infty, & (\pmb{A}, \pmb{\alpha}, \pmb{\lambda})& \in L^{p}([0,T]; \mathbb{R}^{m \times n} \times \mathbb{R} \times \mathbb{R}) \backslash \pmb{\mathscr{D}},
		\end{aligned}\right.  \\
		& \widehat{\mathcal{J}_{N} \circ \mathcal{P}_{N}}(\pmb{A}, \pmb{\alpha}, \pmb{\lambda}) := \left\{\begin{aligned}
			& \mathcal{J}_{N} \circ \mathcal{P}_{N}(\pmb{A}, \pmb{\alpha}, \pmb{\lambda}), & (\pmb{A}, \pmb{\alpha}, \pmb{\lambda})& \in \pmb{\mathscr{D}}, \\
			& +\infty, & (\pmb{A}, \pmb{\alpha}, \pmb{\lambda})& \in L^{p}([0,T]; \mathbb{R}^{m \times n} \times \mathbb{R} \times \mathbb{R}) \backslash \pmb{\mathscr{D}},
		\end{aligned}\right.
	\end{align*}
	for our proof.	
	Since $( (A^{N,:})^{*}, (\alpha^{N,:})^{*}, (\lambda^{N,:})^{*} ) \in \mathcal{P}_{N} \pmb{\mathscr{D}}$ is a minimizer of the problem $(\mathfrak{Q}_{N})$ \eqref{eq-discrete-optimal-control-problem-single} over $\mathcal{P}_{N} \pmb{\mathscr{D}}$, one has, by \textbf{Corollary} \ref{cor-discrete-minimizer-extended-is-composite-minimizer}, that $( \mathcal{I}_{N} (A^{N,:})^{*}, \mathcal{I}_{N} (\alpha^{N,:})^{*}, \mathcal{I}_{N} (\lambda^{N,:})^{*} )$ is a minimizer of $(\mathcal{J}_{N} \circ \mathcal{P}_{N})(\cdot, \cdot, \cdot)$ over $\pmb{\mathscr{D}}$, i.e., a minimizer of $\widehat{\mathcal{J}_{N} \circ \mathcal{P}_{N}}(\cdot, \cdot, \cdot)$ over $L^{p}([0,T];\mathbb{R}^{m\times n} \times \mathbb{R} \times \mathbb{R})$.
    Moreover, by $\mathcal{I}_{N} \mathcal{P}_{N} \pmb{\mathscr{D}} \subseteq \pmb{\mathscr{D}}$, a Bochner version of \textbf{Lemma} \ref{lemma-Kolmogorov-Riesz-bochner} implies that $\{( \mathcal{I}_{N} (A^{N,:})^{*}, \mathcal{I}_{N} (\alpha^{N,:})^{*}, \mathcal{I}_{N} (\lambda^{N,:})^{*} )\}_{N=1}^{+\infty}$ has a subsequence $\{ ( \mathcal{I}_{N_{i}} (A^{N_{i},:})^{*}, \mathcal{I}_{N_{i}} (\alpha^{N_{i},:})^{*}, \mathcal{I}_{N_{i}} (\lambda^{N_{i},:})^{*} ) \}_{i=1}^{+\infty}$ converging to some point $(\pmb{A}^{*}, \pmb{\alpha}^{*}, \pmb{\lambda}^{*})\in\pmb{\mathscr{D}}$, i.e.,
	$
		( \mathcal{I}_{N_{i}} (A^{N_{i},:})^{*}, \mathcal{I}_{N_{i}} (\alpha^{N_{i},:})^{*}, \mathcal{I}_{N_{i}} (\lambda^{N_{i},:})^{*} ) \overset{L^{p}}{\to} ( \pmb{A}^{*}, \pmb{\alpha}^{*}, \pmb{\lambda}^{*} ), 
	$ as $N \to +\infty$.
	
	Using the above facts and $\widehat{\mathcal{J}}(\cdot, \cdot, \cdot) = \Gamma$-$\lim_{N \to +\infty} \widehat{\mathcal{J}_{N} \circ \mathcal{P}_{N}} (\cdot, \cdot, \cdot)$ over $L^{p}([0,T];\mathbb{R}^{m \times n} \times \mathbb{R} \times \mathbb{R})$ similarly in \textbf{Corollary} \ref{cor-gamma-convergence-extended-version}, we then apply the \textbf{Fundamental} \textbf{Theorem} \textbf{of} $\Gamma$-\textbf{Convergence} \ref{thm-fundamental-theorem-of-Gamma-convergence} to see that $(\pmb{A}^{*}, \pmb{\alpha}^{*}, \pmb{\lambda}^{*})$ is a minimizer of $\widehat{\mathcal{J}}(\cdot, \cdot, \cdot)$, i.e., a solution to the problem $(\mathfrak{Q})$ \eqref{eq-continuous-optimal-control-problem-single}.
\end{proof}

\section{Stability properties of learning problems of the basic FBS-network and its deep-layer limit system} \label{sec-stability}

In this section, we discuss the stability properties of the learning problems \eqref{eq-discrete-optimal-control-problem-single} and \eqref{eq-continuous-optimal-control-problem-single}, i.e., the sensitivity of their optimums and optimal solutions w.r.t $x^{0}$, $b$, $y$.

\subsection{The stability of the learning problem \eqref{eq-discrete-optimal-control-problem-single} of the basic FBS-network}

The following theorem describes the stability of the FBS-network in terms of the given initial value $x^{0}$, observed data $b$, and label $y$.
Its proof is inspired by \cite{frecon2018bilevel} for a bilevel optimization problem, and we put it in the appendix for completeness.

\begin{theorem}\label{thm-stability-control-problem-discrete-all}
	(Stability of the learning problem of the basic FBS-network w.r.t. $(x^{0}, b, y)$)
	Suppose the assumptions of the problem $(\mathfrak{Q}_{N})$ in \textbf{Theorem} \ref{thm-existence-of-solution-to-optimal-control-discrete} hold and $p \in [1,+\infty]$.
	Let $\{ ( (x^{0})^{(r)}, b^{(r)}, y^{(r)}) \}_{r=1}^{+\infty} \subset \mathbb{R}^{n} \times \mathbb{R}^{m} \times \mathbb{R}^{n}$ converge to $(x^{0}, b, y) \in \mathbb{R}^{n} \times \mathbb{R}^{m} \times \mathbb{R}^{n}$. 
	For any given $r \in \mathbb{N}_{+}$, we consider the following perturbed problem
	$(\widetilde{\mathfrak{Q}}_{N}^{(r)})$:
	\begin{equation}\label{eq-discrete-optimal-control-problem-all-r}
		\left\{\begin{aligned}
			& \min_{ (A^{N,:}, \alpha^{N,:}, \lambda^{N,:}) \in \mathscr{D}_{N} } \Big\{ \widetilde{\mathcal{J}}_{N}^{(r)}(A^{N,:}, \alpha^{N,:}, \lambda^{N,:}) := \mathcal{L}(x^{N,N}; y^{(r)}) 
			+ \beta_{1} \mathcal{H}_{N}^{(1)}(A^{N,:}) \\
			&\hspace{120pt}+ \beta_{2} \mathcal{H}_{N}^{(2)}(\alpha^{N,:})
			+ \beta_{3} \mathcal{H}_{N}^{(3)}(\lambda^{N,:})
			\Big\}  \\
			& {\rm s.t.} \,\, \vec{0} \in x^{N,k+1} - x^{N,k} + h_{N} \alpha^{N,k} (A^{N,k})^{\top} (A^{N,k}x^{N,k} - b^{(r)}) + h_{N} \alpha^{N,k} \lambda^{N,k} (\partial\mathcal{R})(x^{N,k+1}),  \\ 
			&\hspace{80pt} k = 0, 1, \ldots, N-1,  \\
			& \quad\,\,\,\, x^{N,0} = (x^{0})^{(r)},
		\end{aligned}\right.
	\end{equation}
	and denote its solution set as $\widetilde{S}_{N}^{(r)}$. 
	Let the sequence $\left\{ \big( (A^{N,:})^{*,(r)}, (\alpha^{N,:})^{*,(r)}, (\lambda^{N,:})^{*,(r)} \big) \right\}_{r=1}^{+\infty}$ satisfy $((A^{N,:})^{*,(r)}, (\alpha^{N,:})^{*,(r)}, (\lambda^{N,:})^{*,(r)}) \in \widetilde{S}_{N}^{(r)}$, $\forall r \in \mathbb{N}_{+}$.
	Then,
	\begin{itemize}
		\item[(1) ] $\{ ((A^{N,:})^{*,(r)}, (\alpha^{N,:})^{*,(r)}, (\lambda^{N,:})^{*,(r)}) \}_{r=1}^{+\infty}$ has a convergent subsequence, and 
		its all cluster points belong to $S_{N}$, i.e., the solution set of the problem $(\mathfrak{Q}_{N})$;
			
		\item[(2) ] as $r \to +\infty$,
		\begin{align}\label{eq-dist-converge-zero}
			& \inf_{ ( (A^{N,:})^{*}, (\alpha^{N,:})^{*}, (\lambda^{N,:})^{*} ) \in S_{N}} \| ( (A^{N,:})^{*,(r)}, (\alpha^{N,:})^{*,(r)}, (\lambda^{N,:})^{*,(r)} )  \notag \\
			&\hspace{100pt} - ( (A^{N,:})^{*}, (\alpha^{N,:})^{*}, (\lambda^{N,:})^{*} ) \|_{\ell^{p}((\mathbb{R}^{m \times n})^{N} \times \mathbb{R}^{N} \times \mathbb{R}^{N})} \to 0.
		\end{align}
		
		\item[(3) ] the optimal value of $(\widetilde{\mathfrak{Q}}_{N}^{(r)})$ converges to that of $(\mathfrak{Q}_{N})$ as $r \to +\infty$, that is, 
		\begin{align*}
			\inf_{(A^{N,:}, \alpha^{N,:}, \lambda^{N,:}) \in \mathscr{D}_{N}} \widetilde{\mathcal{J}}_{N}^{(r)}(A^{N,:}, \alpha^{N,:}, \lambda^{N,:}) 
			\to	\inf_{(A^{N,:}, \alpha^{N,:}, \lambda^{N,:}) \in \mathscr{D}_{N}} \mathcal{J}_{N}(A^{N,:}, \alpha^{N,:}, \lambda^{N,:}), \quad r \to +\infty.
		\end{align*}
	\end{itemize}
\end{theorem}

\subsection{The stability of the learning problem \eqref{eq-continuous-optimal-control-problem-single} of the related deep-layer limit system}

The following theorem gives the stability of the deep-layer limit system in terms of the given initial value $x^{0}$, observed data $b$, and label $y$.
Its proof is an infinite dimensional extension of that of \textbf{Theorem} 2.1 in \cite{frecon2018bilevel}.

\begin{theorem}\label{thm-stability-control-problem-continuous-all}
	(Stability of the learning problem of the deep-layer limit system w.r.t. $(x^{0}, b, y)$)
	Suppose the assumptions of the problem $(\mathfrak{Q})$ in \textbf{Theorem} \ref{thm-existence-of-solution-to-optimal-control-continuous} hold and $p \in [1,+\infty)$. 
	Let $\{ ( (x^{0})^{(r)}, b^{(r)}, y^{(r)} ) \}_{r=1}^{+\infty} \subset \mathbb{R}^{n} \times \mathbb{R}^{m} \times \mathbb{R}^{n}$ converge to $(x^{0}, b, y) \in \mathbb{R}^{n} \times \mathbb{R}^{m} \times \mathbb{R}^{n}$. 
	For any given $r \in \mathbb{N}_{+}$, we consider the following perturbed problem $(\widetilde{\mathfrak{Q}}^{(r)})$:
	\begin{equation}\label{eq-continuous-optimal-control-problem-all-r}
		\left\{\begin{aligned}
			& \min_{(\pmb{A}, \pmb{\alpha}, \pmb{\lambda}) \in \pmb{\mathscr{D}}} \Big\{ \widetilde{\mathcal{J}}^{(r)}(\pmb{A}, \pmb{\alpha}, \pmb{\lambda}) := \mathcal{L}((\pmb{x}(T))(\pmb{A}, \pmb{\alpha}, \pmb{\lambda}; (x^{0})^{(r)}, b^{(r)}); y^{(r)}) 
			+ \beta_{1} \mathcal{H}^{(1)}(\pmb{A})  \notag \\
			&\hspace{120pt}+ \beta_{2} \mathcal{H}^{(2)}(\pmb{\alpha})
			+ \beta_{3} \mathcal{H}^{(3)}(\pmb{\lambda}) 
			\Big\}  \\
			& {\rm s.t.} \,\, \vec{0} \in \dot{\pmb{x}}(t) + \pmb{\alpha}(t) (\pmb{A}(t))^{\top} (\pmb{A}(t) \pmb{x}(t) - b^{(r)}) + \pmb{\alpha}(t) \pmb{\lambda}(t) (\partial\mathcal{R})(\pmb{x}(t)), \,\, {\rm a.e.} \,\, t \in [0,T];  \\
			&\quad\,\,\,\, \pmb{x}(0) = (x^{0})^{(r)};
		\end{aligned}\right.
	\end{equation}
	and denote its solution set as $\widetilde{S}^{(r)}$. 
	Let the sequence $\{ (\pmb{A}^{*,(r)}, \pmb{\alpha}^{*,(r)}, \pmb{\lambda}^{*,(r)}) \}_{r=1}^{+\infty}$ satisfy $(\pmb{A}^{*,(r)}, \pmb{\alpha}^{*,(r)}, \pmb{\lambda}^{*,(r)}) \in \widetilde{S}^{(r)}$, $\forall r \in \mathbb{N}_{+}$.
	Then,
	\begin{itemize}
		\item[(1) ] $\{ (\pmb{A}^{*,(r)}, \pmb{\alpha}^{*,(r)}, \pmb{\lambda}^{*,(r)}) \}_{r=1}^{+\infty}$ has a convergent subsequence in $L^{p}$ topology, and its all cluster points of $\{ (\pmb{A}^{*,(r)}, \pmb{\alpha}^{*,(r)}, \pmb{\lambda}^{*,(r)}) \}_{r=1}^{+\infty}$ belong to $S$, i.e., the solution set of the problem $(\mathfrak{Q})$;
			
		\item[(2) ] $\lim_{r \to +\infty} \inf_{(\pmb{A}^{*}, \pmb{\alpha}^{*}, \pmb{\lambda}^{*}) \in S} \|(\pmb{A}^{*,(r)}, \pmb{\alpha}^{*,(r)}, \pmb{\lambda}^{*,(r)}) - (\pmb{A}^{*}, \pmb{\alpha}^{*}, \pmb{\lambda}^{*})\|_{L^{p}([0,T]; \mathbb{R}^ {m \times n} \times \mathbb{R} \times \mathbb{R})} = 0$;
		
		\item[(3) ] the optimal value of $(\widetilde{\mathfrak{Q}}^{(r)})$ converges to that of $(\mathfrak{Q})$ as $r \to +\infty$, that is,
		\begin{align*}
			\inf_{(\pmb{A}, \pmb{\alpha}, \pmb{\lambda}) \in \pmb{\mathscr{D}}} \widetilde{\mathcal{J}}^{(r)} ( \pmb{A}, \pmb{\alpha}, \pmb{\lambda} ) 
			\to \inf_{(\pmb{A}, \pmb{\alpha}, \pmb{\lambda}) \in \pmb{\mathscr{D}}} \mathcal{J}( \pmb{A}, \pmb{\alpha}, \pmb{\lambda} ),  \quad r \to +\infty.
		\end{align*}
	\end{itemize}
\end{theorem}

\begin{proof}[Proof. ]
	\textbf{Step 1:} We first prove the following two preliminary results:
	\begin{itemize}
		\item[(i) ] $\lim_{r \to +\infty} \sup_{(\pmb{A}, \pmb{\alpha}, \pmb{\lambda}) \in \pmb{\mathscr{D}}} |\widetilde{\mathcal{J}}^{(r)} (\pmb{A}, \pmb{\alpha}, \pmb{\lambda}) - \mathcal{J}(\pmb{A}, \pmb{\alpha}, \pmb{\lambda}) | = 0$.
		
		\item[(ii) ] If the sequence $\{ (\pmb{A}^{(l)}, \pmb{\alpha}^{(l)}, \pmb{\lambda}^{(l)}) \}_{l=1}^{+\infty} \subset \pmb{\mathscr{D}}$ satisfies $\|(\pmb{A}^{(l)}, \pmb{\alpha}^{(l)}, \pmb{\lambda}^{(l)}) - (\widetilde{\pmb{A}}, \widetilde{\pmb{\alpha}}, \widetilde{\pmb{\lambda}})\|_{L^{p}([0,T]; \mathbb{R}^{m \times n} \times \mathbb{R} \times \mathbb{R})} \to 0$ as $l \to +\infty$, then $\lim_{l \to +\infty} \mathcal{J}(\pmb{A}^{(l)}, \pmb{\alpha}^{(l)}, \pmb{\lambda}^{(l)}) = \mathcal{J}(\widetilde{\pmb{A}}, \widetilde{\pmb{\alpha}}, \widetilde{\pmb{\lambda}})$.
	\end{itemize}   
	
	For clarity, we here denote $(\pmb{x}(T))(\pmb{A}, \pmb{\alpha}, \pmb{\lambda}; x^{0}, b)$ as the value at $T$ of the unique solution $\pmb{x}$ of the differential inclusion \eqref{eq-system-inclusion-continuous} determined by $(\pmb{A}, \pmb{\alpha}, \pmb{\lambda})$ with the initial value $x^{0}$ and the observed  data $b$. 
	
	For (i), since $\pmb{\mathscr{D}}$ is bounded in $L^{\infty}([0,T];\mathbb{R}^{m \times n} \times \mathbb{R} \times \mathbb{R})$, 
	\cite[Thm.4.2]{lin2025deep} with $p = +\infty$ provides that 
	\begin{align*}
		\cup_{r=1}^{+\infty} \{ (\pmb{x}(T))(\pmb{A}, \pmb{\alpha}, \pmb{\lambda}; (x^{0})^{(r)}, b^{(r)}): (\pmb{A}, \pmb{\alpha}, \pmb{\lambda}) \in \pmb{\mathscr{D}} \} 
		\bigcup \{ (\pmb{x}(T))(\pmb{A}, \pmb{\alpha}, \pmb{\lambda}; x^{0}, b): (\pmb{A}, \pmb{\alpha}, \pmb{\lambda}) \in \pmb{\mathscr{D}} \}
	\end{align*} 
	is a subset of a compact set (exactly a closed bounded ball) of $\mathbb{R}^{n}$. 
	Moreover, 
	\cite[Thm.4.2]{lin2025deep} 
	and the boundedness of $\pmb{\mathscr{D}}$ indicate that, there exist two constants $\widetilde{c}_{1} \geq 0, \widetilde{c}_{2} \geq 0$ both independent of $r$, $\pmb{A}$, $\pmb{\alpha}$, $\pmb{\lambda}$ such that for any $(\pmb{A}, \pmb{\alpha}, \pmb{\lambda}) \in \pmb{\mathscr{D}}$,
	\begin{align*}
		\|(\pmb{x}(T))(\pmb{A}, \pmb{\alpha}, \pmb{\lambda}; (x^{0})^{(r)}, b^{(r)}) - (\pmb{x}(T))(\pmb{A}, \pmb{\alpha}, \pmb{\lambda}; x^{0}, b)\|_{2}  
		\leq \widetilde{c}_{1} \|(x^{0})^{(r)} - x^{0}\|_{2}
		+ \widetilde{c}_{2} \|b^{(r)} - b\|_{2}.  
	\end{align*}
	Hence, for any $\delta > 0$, there exists $r_{1,\delta} \in \mathbb{N}_{+}$ such that for every integer $r \geq r_{1,\delta}$, 
	\begin{align*}
		\sup_{(\pmb{A}, \pmb{\alpha}, \pmb{\lambda}) \in \pmb{\mathscr{D}}} \|(\pmb{x}(T))(\pmb{A}, \pmb{\alpha}, \pmb{\lambda}; (x^{0})^{(r)}, b^{(r)}) - (\pmb{x}(T))(\pmb{A}, \pmb{\alpha}, \pmb{\lambda}; x^{0}, b)\|_{2} 
		\leq \frac{1}{3} \delta.
	\end{align*}
	Similarly, $\lim_{r \to +\infty} y^{(r)} = y$ indicates that for any $\delta > 0$, there exists $r_{2,\delta} \in \mathbb{N}_{+}$ such that for every integer $r \geq r_{2,\delta}$, we have
	\begin{align*}
		\|y^{(r)} - y\|_{2} \leq \frac{1}{3} \delta.
	\end{align*}
	These two results indicate that for any $\delta > 0$, there exists $r_{3,\delta} := \max\{ r_{1,\delta}, r_{2,\delta} \} \in \mathbb{N}_{+}$ such that for every integer $r \geq r_{3,\delta}$, 
	\begin{align*}
		\sup_{(\pmb{A}, \pmb{\alpha}, \pmb{\lambda}) \in \pmb{\mathscr{D}}} \|(\pmb{x}(T))(\pmb{A}, \pmb{\alpha}, \pmb{\lambda}; (x^{0})^{(r)}, b^{(r)}) - (\pmb{x}(T))(\pmb{A}, \pmb{\alpha}, \pmb{\lambda}; x^{0}, b)\|_{2} + \|y^{(r)} - y\|_{2} 
		\leq \frac{2}{3} \delta
		< \delta.
	\end{align*}
	By further noting the continuity assumption of $\mathcal{L}$ from (A4) and thus the uniform continuity of $\mathcal{L}$ over compact sets, we obtain that for any $\varepsilon > 0$, there exists $r_{3,\delta(\varepsilon)} \equiv \max\{ r_{1,\delta(\varepsilon)}, r_{2,\delta(\varepsilon)} \} \in \mathbb{N}_{+}$ such that for every integer $r \geq r_{3,\delta(\varepsilon)}$,
	\begin{align*}
		& \sup_{(\pmb{A}, \pmb{\alpha}, \pmb{\lambda}) \in \pmb{\mathscr{D}}} |\widetilde{\mathcal{J}}^{(r)}(\pmb{A}, \pmb{\alpha}, \pmb{\lambda}) - \mathcal{J}(\pmb{A}, \pmb{\alpha}, \pmb{\lambda})|  \\
		=& \sup_{(\pmb{A}, \pmb{\alpha}, \pmb{\lambda}) \in \pmb{\mathscr{D}}} \left| \mathcal{L}\left( (\pmb{x}(T))(\pmb{A}, \pmb{\alpha}, \pmb{\lambda}; (x^{0})^{(r)}, (b)^{(r)}); (y)^{(r)} \right)
		+ \beta_{1} \mathcal{H}^{(1)}(\pmb{A})		
		+ \beta_{2} \mathcal{H}^{(2)}(\pmb{\alpha})
		+ \beta_{3} \mathcal{H}^{(3)}(\pmb{\lambda})  
		\right. \\
		&\hspace{50pt}\left.- \mathcal{L}\left( (\pmb{x}(T))(\pmb{A}, \pmb{\alpha}, \pmb{\lambda}; x^{0}, b); y \right) 
		- \beta_{1} \mathcal{H}^{(1)}(\pmb{A})		
		- \beta_{2} \mathcal{H}^{(2)}(\pmb{\alpha})
		- \beta_{3} \mathcal{H}^{(3)}(\pmb{\lambda})
		\right|  \\
		=& \sup_{(\pmb{A}, \pmb{\alpha}, \pmb{\lambda}) \in \pmb{\mathscr{D}}} \left| \mathcal{L}\left( (\pmb{x}(T))(\pmb{A}, \pmb{\alpha}, \pmb{\lambda}; (x^{0})^{(r)}, (b)^{(r)}); (y)^{(r)} \right) - \mathcal{L}\left( (\pmb{x}(T))(\pmb{A}, \pmb{\alpha}, \pmb{\lambda}; x^{0}, b); y \right) \right|  \\
		<& \varepsilon,
	\end{align*}  
	which proves (i).
	
	The proof of (ii) is similar to that of (i). 
	By 
	the assumption (A5) of $\psi$, we see that 
	\begin{align*}
		|\mathcal{H}^{(1)}(\pmb{A}^{(l)}) - \mathcal{H}^{(1)}(\widetilde{\pmb{A}})|  
		\leq& C_{\pmb{\mathscr{D}}} \|\pmb{A}^{(l)} - \widetilde{\pmb{A}}\|_{L^{p}([0,T]; \mathbb{R}^{m \times n})}
		\to 0, \quad l \to +\infty,
	\end{align*}
	where 
	$C_{\pmb{\mathscr{D}}} > 0$ is independent of $\pmb{A}^{(l)}$ and $\widetilde{\pmb{A}}$.
	Similarly, $|\mathcal{H}^{(2)}(\pmb{\alpha}^{(l)}) - \mathcal{H}^{(2)}(\widetilde{\pmb{\alpha}})| \leq C_{\pmb{\mathscr{D}}} \|\pmb{\alpha}^{(l)} - \widetilde{\pmb{\alpha}}\|_{L^{p}([0,T])}$ and $|\mathcal{H}^{(3)}(\pmb{\lambda}^{(l)}) - \mathcal{H}^{(3)}(\widetilde{\pmb{\lambda}})| \leq C_{\pmb{\mathscr{D}}} \|\pmb{\lambda}^{(l)} - \widetilde{\pmb{\lambda}}\|_{L^{p}([0,T])}$.
	Hence, for any $\varepsilon > 0$, there exists $l_{1,\varepsilon} \in \mathbb{N}_{+}$ such that for every $l \geq l_{1,\varepsilon}$,
	\begin{align*}
		\beta_{1} |\mathcal{H}^{(1)}(\pmb{A}^{(l)}) - \mathcal{H}^{(1)}(\widetilde{\pmb{A}})|  
		+ \beta_{2} |\mathcal{H}^{(2)}(\pmb{\alpha}^{(l)}) - \mathcal{H}^{(2)}(\widetilde{\pmb{\alpha}})|  
		+ \beta_{3} |\mathcal{H}^{(3)}(\pmb{\lambda}^{(l)}) - \mathcal{H}^{(3)}(\widetilde{\pmb{\lambda}})|
		< \frac{1}{2} \varepsilon.
	\end{align*}
	Besides, by \cite[Thm.4.2]{lin2025deep}
	and the boundedness of $\pmb{\mathscr{D}}$, we know
	\begin{align*}
		& \|(\pmb{x}(T))(\pmb{A}^{(l)}, \pmb{\alpha}^{(l)}, \pmb{\lambda}^{(l)}; x^{0}, b) - (\pmb{x}(T))(\pmb{A}, \pmb{\alpha}, \pmb{\lambda}; x^{0}, b)\|_{2}   \\ 
		\leq& \overline{c} (\|\pmb{A}^{(l)} - \widetilde{\pmb{A}}\|_{L^{p}([0,T];\mathbb{R}^{m \times n})} 
		+ \|\pmb{\alpha}^{(l)} - \widetilde{\pmb{\alpha}}\|_{L^{p}([0,T])} 
		+ \|\pmb{\lambda}^{(l)} - \widetilde{\pmb{\lambda}}\|_{L^{p}([0,T])}), 
	\end{align*} 
	where $\overline{c} > 0$ is a constant independent of $l$, $\pmb{A}^{(l)}$, $\widetilde{\pmb{A}}$, $\pmb{\alpha}^{(l)}$, $\widetilde{\pmb{\alpha}}$, $\pmb{\lambda}^{(l)}$, $\widetilde{\pmb{\lambda}}$. 
	Hence, for any $\delta > 0$, there exists $l_{2,\delta} \in \mathbb{N}_{+}$ such that for every integer $l \geq l_{2,\delta}$, 
	\begin{align*}
		\|(\pmb{x}(T))(\pmb{A}^{(l)}, \pmb{\alpha}^{(l)}, \pmb{\lambda}^{(l)}; x^{0}, b) - (\pmb{x}(T))(\widetilde{\pmb{A}}, \widetilde{\pmb{\alpha}}, \widetilde{\pmb{\lambda}}; x^{0}, b)\|_{2} 
		< \delta.
	\end{align*}
	By using the continuity assumption of $\mathcal{L}$ in (A4) and the above result, we obtain that for any $\varepsilon > 0$, there exists $l_{3,\varepsilon} := \max\{ l_{1,\varepsilon}, l_{2,\delta(\varepsilon)} \} \in \mathbb{N}_{+}$, such that for every $l \geq l_{3,\varepsilon}$,
	\begin{align*}
		& |\mathcal{J}(\pmb{A}^{(l)}, \pmb{\alpha}^{(l)}, \pmb{\lambda}^{(l)}) - \mathcal{J}(\widetilde{\pmb{A}}, \widetilde{\pmb{\alpha}}, \widetilde{\pmb{\lambda}})|  \\
		=& \left| \mathcal{L}\left( (\pmb{x}(T))(\pmb{A}^{(l)}, \pmb{\alpha}^{(l)}, \pmb{\lambda}^{(l)}; x^{0}, b); y \right) 
		+ \beta_{1} \mathcal{H}^{(1)}(\pmb{A}^{(l)})
		+ \beta_{2} \mathcal{H}^{(2)}(\pmb{\alpha}^{(l)})
		+ \beta_{3} \mathcal{H}^{(3)}(\pmb{\lambda}^{(l)})
		\right. \\
		&\left.- \mathcal{L}\left( (\pmb{x}(T))(\widetilde{\pmb{A}}, \widetilde{\pmb{\alpha}}, \widetilde{\pmb{\lambda}}; x^{0}, b); y \right) 
		- \beta_{1} \mathcal{H}^{(1)}(\widetilde{\pmb{A}})
		- \beta_{2} \mathcal{H}^{(2)}(\widetilde{\pmb{\alpha}})
		- \beta_{3} \mathcal{H}^{(3)}(\widetilde{\pmb{\lambda}})
		\right|  \\
		\leq& \left| \mathcal{L}\left( (\pmb{x}(T))(\pmb{A}^{(l)}, \pmb{\alpha}^{(l)}, \pmb{\lambda}^{(l)}; x^{0}, b); y \right) - \mathcal{L}\left( (\pmb{x}(T))(\widetilde{\pmb{A}}, \widetilde{\pmb{\alpha}}, \widetilde{\pmb{\lambda}}; x^{0}, b); y \right) \right|  \\
		&+ \beta_{1} |\mathcal{H}^{(1)}(\pmb{A}^{(l)}) - \mathcal{H}^{(1)}(\widetilde{\pmb{A}})|  
		+ \beta_{2} |\mathcal{H}^{(2)}(\pmb{\alpha}^{(l)}) - \mathcal{H}^{(2)}(\widetilde{\pmb{\alpha}})|  
		+ \beta_{3} |\mathcal{H}^{(3)}(\pmb{\lambda}^{(l)}) - \mathcal{H}^{(3)}(\widetilde{\pmb{\lambda}})|  \\
		<& \frac{1}{2} \varepsilon + \frac{1}{2} \varepsilon
		= \varepsilon,
	\end{align*} 
	which proves (ii).
	
	\textbf{Step 2:} In this step, we will prove (1).
	The existence of the convergent subsequence of $\{ (\pmb{A}^{*,(r)}, \pmb{\alpha}^{*,(r)}, \pmb{\lambda}^{*,(r)}) \}_{r=1}^{+\infty}$ in $L^{p}$ topology is obvious by the uniform boundedness of $\pmb{\mathscr{D}}$, the equi-continuity of $\pmb{\mathscr{D}}$ in $L^{p}$ sense, and a Bochner version of \textbf{Lemma} \ref{lemma-Kolmogorov-Riesz-bochner}. 
	 
	We next show that, in $L^{p}$ topology, every cluster point $(\pmb{A}^{*,*}, \pmb{\alpha}^{*,*}, \pmb{\lambda}^{*,*})$ of $\{ (\pmb{A}^{*,(r)}, \pmb{\alpha}^{*,(r)}, \pmb{\lambda}^{*,(r)}) \}_{r=1}^{+\infty}$ belongs to $S$.
	Assume that $(\pmb{A}^{*,(r_{i})}, \pmb{\alpha}^{*,(r_{i})}, \pmb{\lambda}^{*,(r_{i})}) \overset{L^{p}}{\to} (\pmb{A}^{*,*}, \pmb{\alpha}^{*,*}, \pmb{\lambda}^{*,*})$ as $i \to +\infty$. 
	We see that $(\pmb{A}^{*,*}, \pmb{\alpha}^{*,*}, \pmb{\lambda}^{*,*}) \in \pmb{\mathscr{D}}$.
	Using again the continuity of $\mathcal{L}$ in (A4), we obtain
	\begin{align*}
		& \mathcal{J}(\pmb{A}^{*,*}, \pmb{\alpha}^{*,*}, \pmb{\lambda}^{*,*}) \notag \\
		=& \mathcal{L}\Big( (\pmb{x}(T))(\pmb{A}^{*,*}, \pmb{\alpha}^{*,*}, \pmb{\lambda}^{*,*}; x^{0}, b); y \Big)
		+ \beta_{1} \mathcal{H}^{(1)}(\pmb{A}^{*,*})
		+ \beta_{2} \mathcal{H}^{(2)}(\pmb{\alpha}^{*,*})
		+ \beta_{3} \mathcal{H}^{(3)}(\pmb{\lambda}^{*,*})
		\notag \\
		=& \lim\limits_{i \to +\infty} \Big[ \mathcal{L}\Big( (\pmb{x}(T))(\pmb{A}^{*,(r_{i})}, \pmb{\alpha}^{*,(r_{i})}, \pmb{\lambda}^{*,(r_{i})}; x^{0}, b); y \Big)  
		+ \beta_{1} \mathcal{H}^{(1)}(\pmb{A}^{*,(r_{i})})  \notag \\
		&\hspace{25pt} + \beta_{2} \mathcal{H}^{(2)}(\pmb{\alpha}^{*,(r_{i})})
		+ \beta_{3} \mathcal{H}^{(3)}(\pmb{\lambda}^{*,(r_{i})})
		\Big]
		\notag \\ 
		=& \lim\limits_{i \to +\infty} \Big[ \mathcal{L}((\pmb{x}(T))(\pmb{A}^{*,(r_{i})}, \pmb{\alpha}^{*,(r_{i})}, \pmb{\lambda}^{*,(r_{i})}; (x^{0})^{(r_{i})}, (b)^{(r_{i})}); (y)^{(r_{i})})  
		+ \beta_{1} \mathcal{H}^{(1)}(\pmb{A}^{*,(r_{i})})  \notag \\
		&\hspace{25pt} + \beta_{2} \mathcal{H}^{(2)}(\pmb{\alpha}^{*,(r_{i})})
		+ \beta_{3} \mathcal{H}^{(3)}(\pmb{\lambda}^{*,(r_{i})})
		\Big]
		\notag \\
		\leq& \lim\limits_{i \to +\infty} \Big[ \mathcal{L}((\pmb{x}(T))(\pmb{A}, \pmb{\alpha}, \pmb{\lambda}; (x^{0})^{(r_{i})}, (b)^{(r_{i})}); (y)^{(r_{i})})  
		+ \beta_{1} \mathcal{H}^{(1)}(\pmb{A})
		+ \beta_{2} \mathcal{H}^{(2)}(\pmb{\alpha})
		+ \beta_{3} \mathcal{H}^{(3)}(\pmb{\lambda})
		\Big],
		\notag \\
		&\hspace{20pt} \forall (\pmb{A}, \pmb{\alpha}, \pmb{\lambda}) \in \pmb{\mathscr{D}},
		\notag \\
		=& \mathcal{L}((\pmb{x}(T))(\pmb{A}, \pmb{\alpha}, \pmb{\lambda}; x^{0}, b); y)  
		+ \beta_{1} \mathcal{H}^{(1)}(\pmb{A})
		+ \beta_{2} \mathcal{H}^{(2)}(\pmb{\alpha})
		+ \beta_{3} \mathcal{H}^{(3)}(\pmb{\lambda}), 
		\quad
		\forall (\pmb{A}, \pmb{\alpha}, \pmb{\lambda}) \in \pmb{\mathscr{D}},
		\notag \\
		=& \mathcal{J}(\pmb{A}, \pmb{\alpha}, \pmb{\lambda}), \quad \forall (\pmb{A}, \pmb{\alpha}, \pmb{\lambda}) \in \pmb{\mathscr{D}},	
	\end{align*}
	where the second equality is due to \textbf{Step 1} (ii), the third and the fifth equalities are both due to \textbf{Step 1} (i), and the fourth inequality is from $(\pmb{A}^{*,(r_{i})}, \pmb{\alpha}^{*,(r_{i})}, \pmb{\lambda}^{*,(r_{i})}) \in \widetilde{S}^{(r_{i})}$.
	This indicates $(\pmb{A}^{*,*}, \pmb{\alpha}^{*,*}, \pmb{\lambda}^{*,*}) \in S$, and 
	one has, similar to the derivation for Eq.\eqref{eq-perturbed-optimal-value-converge-discrete}, that
	\begin{align}\label{eq-perturbed-optimal-value-converge-continuous}
		\lim\limits_{i \to +\infty} \widetilde{\mathcal{J}}^{(r_{i})}(\pmb{A}^{*,(r_{i})}, \pmb{\alpha}^{*,(r_{i})}, \pmb{\lambda}^{*,(r_{i})})
		= \inf_{(\pmb{A}, \pmb{\alpha}, \pmb{\lambda}) \in \pmb{\mathscr{D}}} \mathcal{J}(\pmb{A}, \pmb{\alpha}, \pmb{\lambda}).
	\end{align}
	
	\textbf{Step 3:} We now prove (2) and (3) by contradiction. 
	To prove (2), we assume that the sequence 
	\begin{align*}
		\inf\limits_{(\pmb{A}^{*}, \pmb{\alpha}^{*}, \pmb{\lambda}^{*}) \in S} \|(\pmb{A}^{*,(r)}, \pmb{\alpha}^{*,(r)}, \pmb{\lambda}^{*,(r)}) - (\pmb{A}^{*}, \pmb{\alpha}^{*}, \pmb{\lambda}^{*})\|_{L^{p}([0,T]; \mathbb{R}^{m \times n} \times \mathbb{R} \times \mathbb{R})} \nrightarrow 0.
	\end{align*}
	Then there exist an $\varepsilon_{0} > 0$ and a subsequence 
	\begin{align*}
		\Big\{ \inf_{(\pmb{A}^{*}, \pmb{\alpha}^{*}, \pmb{\lambda}^{*}) \in S} \|(\pmb{A}^{*,(r_{i})}, \pmb{\alpha}^{*,(r_{i})}, \pmb{\lambda}^{*,(r_{i})}) - (\pmb{A}^{*}, \pmb{\alpha}^{*}, \pmb{\lambda}^{*})\|_{L^{p}([0,T]; \mathbb{R}^{m \times n} \times \mathbb{R} \times \mathbb{R})} \Big\}_{i=1}^{+\infty}
	\end{align*}
	such that for every $i \in \mathbb{N}_{+}$,
	\begin{align}\label{eq-dist-continuous-all}
		\inf_{(\pmb{A}^{*}, \pmb{\alpha}^{*}, \pmb{\lambda}^{*}) \in S} \|(\pmb{A}^{*,(r_{i})}, \pmb{\alpha}^{*,(r_{i})}, \pmb{\lambda}^{*,(r_{i})}) - (\pmb{A}^{*}, \pmb{\alpha}^{*}, \pmb{\lambda}^{*})\|_{L^{p}([0,T]; \mathbb{R}^{m \times n} \times \mathbb{R} \times \mathbb{R})}
		\geq \varepsilon_{0}.
	\end{align} 
	By \textbf{Step 2}, we can further extract a convergent subsequence $\{ (\pmb{A}^{*,(r_{i_{j}})}, \pmb{\alpha}^{*,(r_{i_{j}})}, \pmb{\lambda}^{*,(r_{i_{j}})}) \}_{j=1}^{+\infty}$ from $\{ (\pmb{A}^{*,(r_{i})}, \pmb{\alpha}^{*,(r_{i})}, \pmb{\lambda}^{*,(r_{i})}) \}_{i=1}^{+\infty}$ in $L^{p}$ topology with a limit point $(\pmb{A}^{*,+}, \pmb{\alpha}^{*,+}, \pmb{\lambda}^{*,+}) \in S$, and then 
	\begin{align*}
		& \inf_{(\pmb{A}^{*}, \pmb{\alpha}^{*}, \pmb{\lambda}^{*}) \in S} \|(\pmb{A}^{*,(r_{i_{j}})}, \pmb{\alpha}^{*,(r_{i_{j}})}, \pmb{\lambda}^{*,(r_{i_{j}})}) - (\pmb{A}^{*}, \pmb{\alpha}^{*}, \pmb{\lambda}^{*})\|_{L^{p}([0,T]; \mathbb{R}^{m \times n} \times \mathbb{R} \times \mathbb{R})}  \\
		\leq& \|(\pmb{A}^{*,(r_{i_{j}})}, \pmb{\alpha}^{*,(r_{i_{j}})}, \pmb{\lambda}^{*,(r_{i_{j}})}) - (\pmb{A}^{*,+}, \pmb{\alpha}^{*,+}, \pmb{\lambda}^{*,+})\|_{L^{p}([0,T]; \mathbb{R}^{m \times n} \times \mathbb{R} \times \mathbb{R})}  \\
		\to& 0, \quad j \to +\infty,
	\end{align*} 
	which contradicts with Eq.\eqref{eq-dist-continuous-all}.
	A similar argument and noting Eq.\eqref{eq-perturbed-optimal-value-converge-continuous} give (3).
\end{proof}

\begin{remark}
	According to \textbf{Theorem} \ref{thm-stability-control-problem-discrete-all} and \textbf{Theorem} \ref{thm-stability-control-problem-continuous-all}, if we set $b^{(r)} \equiv b$, $y^{(r)} \equiv y$, $\forall r \in \mathbb{N}_{+}$, we can derive the perturbation stability of the minimums and the solutions to the discrete- and continuous-time control problems w.r.t. the initial value $x^{0}$. 
	Similarly, we can obtain the perturbation stability w.r.t. $b$ and $y$ respectively under the settings fixing $(x^{0})^{(r)} \equiv x^{0}$, $y^{(r)} \equiv y$ or fixing $(x^{0})^{(r)} \equiv x^{0}$, $b^{(r)} \equiv b$.
\end{remark}

\section{Numerical experiment}\label{sec-experiment}

In this section, we conduct a simple experiment to investigate the behavior of the training loss for the basic FBS-network, providing numerical validation for our main general convergence result \textbf{Theorem} \ref{thm-convergence-objective-functions}. 
We examine how the training loss varies with an increasing number of network layers $N$ in sparse signal reconstruction tasks.
Note that our main result focuses on the behavior of the training problem rather than generalization, and we thus observe only the training loss here.
All implementations are carried out in PyTorch and executed on $4$ NVIDIA GeForce RTX 3060 GPUs.

We construct a customized dataset $\{ (b_{j}, y_{j}) \}_{j=1}^{J + \overline{J}}$ with $J = 16384$ and $\overline{J} = 2048$ being the numbers of training and validation samples, respectively. 
The dimensions of the observation vector $b_{j} \in \mathbb{R}^{m}$ and the corresponding ground truth sparse signal vector $y_{j} \in \mathbb{R}^{n}$ of the $j$-th sample are set to $m = 256$ and $n = 1024$, as in Eq.\eqref{eq-system-inclusion-discrete}. 
The observations are generated as $b_{j} = Ay_{j} + \varepsilon_{j}$ for $j = 1, 2, \dots, J+\overline{J}$ using a fixed matrix $A \in \mathbb{R}^{m \times n}$, where $\{ \varepsilon_{j} \}_{j=1}^{J+\overline{J}}$ are small observation errors. 
No additional preprocessing is applied.

We construct an $N$-layer FBS-network unfolded from the FBS structure as in Eq.\eqref{eq-system-inclusion-discrete} with $h_{N} = 1 / N$ for $N \in \mathbb{N}_{+}$.
Each layer comprises a linear transformation followed by a nonlinear activation, for which we choose the soft-thresholding function.
This corresponds to setting the function $\mathcal{R}$ as $\mathcal{R}(\cdot) := \|\cdot\|_{1}$ in Eq.\eqref{eq-system-inclusion-discrete}.
The training loss, introduced in Eq.\eqref{eq-discrete-optimal-control-problem-multi}, consists of a data fitting term and several regularization terms.
The data fitting term $\mathcal{L}$ is the mean squared error (MSE) between the network output and the ground truth:
\begin{align*}
    \mathcal{L}\bigl( x^{N,N}_{j}; y_{j} \bigr) 
    := \frac{1}{2} \bigl\| x^{N,N}_{j} - y_{j} \bigr\|_{2}^{2},
\end{align*}
where $j = 1, 2, \cdots, J$.
The regularization terms are defined as
\begin{align*}
    \mathcal{H}_{N}^{(1)}\bigl( A^{N,:} \bigr) 
    = \frac{1}{N} \sum_{k=0}^{N-1} \bigl\| A^{N,k} \bigr\|_{2}^{2}, \quad
    \mathcal{H}_{N}^{(2)}\bigl( \alpha^{N,:} \bigr) 
    = \frac{1}{N} \sum_{k=0}^{N-1} \bigl| \alpha^{N,k} \bigr|^{2},  \quad
    \mathcal{H}_{N}^{(3)}\bigl( \lambda^{N,:} \bigr) 
    = \frac{1}{N} \sum_{k=0}^{N-1} \bigl| \lambda^{N,k} \bigr|^{2},
\end{align*}
where $A^{N,:} = \{ A^{N,k} \}_{k=0}^{N-1}$, $\alpha^{N,:} = \{ \alpha^{N,k} \}_{k=0}^{N-1}$, and $\lambda^{N,:} = \{ \lambda^{N,k} \}_{k=0}^{N-1}$ denote the collections of trainable parameters across the $N$ layers.
The corresponding regularization coefficients are set to $\beta_1 = \beta_2 = \beta_3 = 1 \times 10^{-7}$.

We now describe the training configuration for the FBS-network models. 
All models are trained for $800$ epochs using stochastic gradient descent (SGD) with a momentum of $0.9$ and an initial learning rate $r_{0} := 8.0 \times 10^{-3}$.
Following the approach in \cite{Ludziejewski2025DecoupledRL}, we assign different learning rates to different parameter groups. 
Specifically, the learning rates for $A^{N,:}$ and $\lambda^{N,:}$ are set to $r_{A} := r_{0} * N$ and $r_{\lambda} := r_{0} * N$, respectively, while the learning rate for $\alpha^{N,:}$ is set to $r_{\alpha} := r_{0} * N^{3}$.
The batch size is set to $256$.
The trainable linear transformation matrices $A^{N,:}$ are all initialized as $((A^\top)_{\text{orth}})^{\top}$, i.e., the transpose of the orthogonal decomposition of $A^{\top}$.
Similarly, the trainable scalar parameters $\alpha^{N,:}$ and $\lambda^{N,:}$ are initialized as $10$ and $0.05$, respectively.

\begin{figure}[htbp]
	\centering
	\includegraphics[width=0.8\textwidth]{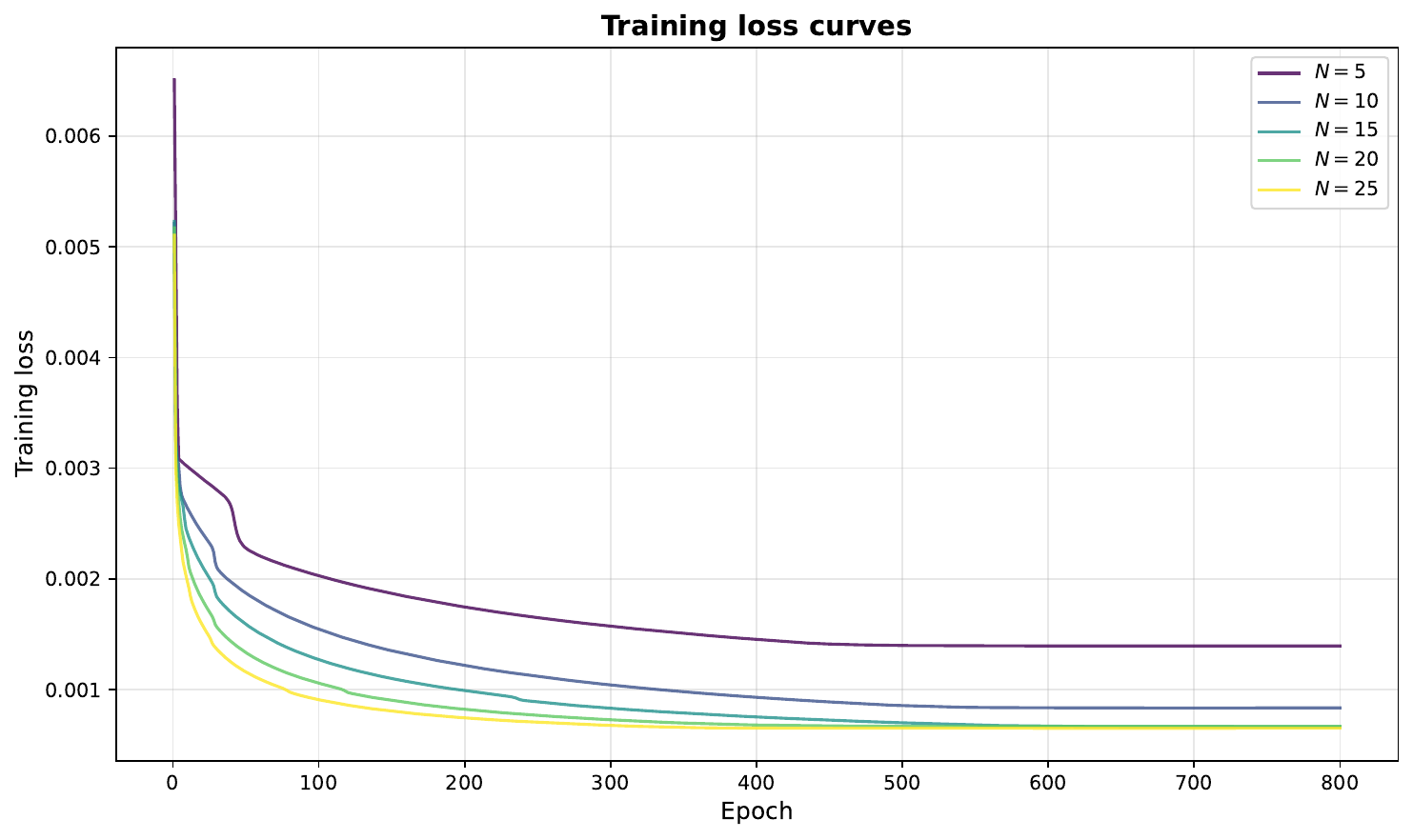}
	\caption{The training loss curves of the FBS-network with increasing network layer numbers $N = 5, 10, 15, 20, 25$. 
	}
	\label{fig:train-loss-compare}
\end{figure}

To verify our main convergence result, we evaluate the performance of the FBS-network with increasing $N \in \{5, 10, 15, 20, 25\}$. 
Figure~\ref{fig:train-loss-compare} presents the training loss curves versus epochs for these networks. 
We can observe two phenomena.
First, the training loss consistently decreases as $N$ increases, with larger $N$ yielding a lower loss. 
Second, the marginal improvement gradually diminishes as the network becomes deeper, indicating that the network performance gradually tends to a limit.
These empirical findings verify our theoretical analysis in \textbf{Theorem} \ref{thm-convergence-objective-functions}.

\section{Conclusion}\label{sec-conclusion}

In the last decade, the FBS algorithm has been unrolled to build several effective neural networks in sparse coding and image restoration.
In this paper, we studied the deep-layer limit and stability behaviors of the learning problems of the basic FBS-network with direct parameter relaxations as given in \cite{lin2025deep}. 
Under the dynamical inclusion forward modeling, we established some convergence properties of the network training problem to the related deep-layer limit learning problem, and presented stability results of the above two learning problems.
Our analysis procedure can be simplified to derive similar results for other unrolling networks from FBS-type algorithms, such as LISTA and ALISTA networks. 
These theoretical results indicate, in some sense, that such FBS-induced learning methods can work stably and consistently with different numbers of layers.

\section*{Acknowledgments}
This work was supported by the National Natural Science Foundation of China (grants 12271273), the Key Program (21JCZDJC00220) of the Natural Science Foundation of Tianjin, China.

\bibliographystyle{elsarticle-num}
\bibliography{FBS_II_ref}

\begin{appendices}
	
\section{Appendix}\label{sec-preliminaries-control}

\subsection{Backgrounds of some convergence} \label{subappendix-set-value-mapping}

\begin{theorem}\label{thm-Lp-imply-a.e.pointwise}
	($L^{p}$ convergence implies a.e. pointwise convergence up to a subsequence)
	Let $(X,\mathcal{A},\mu)$ be a measure space, $\{f_{k}\}_{k=1}^{+\infty} \subseteq L^{p}(X)$, and $f \in L^{p}(X)$. Here, $1 \leq p \leq +\infty$. If $\|f_{k} - f\|_{L^{p}(X)} \to 0$ as $k \to +\infty$, then there exists a subsequence $\{f_{k_{r}}\}_{k=1}^{+\infty}$ such that $|f_{k_{r}}(x) - f(x)| \to 0$ as $r \to +\infty$ for a.e. $x \in X$.
\end{theorem}

\begin{proof}[Proof. ]
	It is just a simple combination of \cite[Prop.1.1.9]{grafakos2008classical} and \cite[Thm.1.1.11]{grafakos2008classical}.
\end{proof}

\begin{theorem}\label{thm-Kolmogorov-Riesz}
	(Kolmogorov-Riesz-Fr{\'e}chet Theorem, a generalization of Arzela-Ascoli Theorem in $L^{p}$ spaces)
	\cite[Thm.4.26]{brezis2010functional}
	Let $\mathcal{F}$ be a bounded set in $L^{p}(\mathbb{R}^{N})$ with $p \in [1,+\infty)$. Assume that
	\begin{align*}
		\lim_{|h| \to 0} \|\tau_{h}f - f\|_{L^{p}(\mathbb{R}^{N})}  = 0 \,\, {\rm uniformly \,\, for} \,\, f \in \mathcal{F},
	\end{align*}
	i.e., $\forall \varepsilon > 0$, $\exists \delta > 0$ such that $\forall h \in \mathbb{R}^{N}$ with $|h| < \delta$, the inequality $\|\tau_{h}f - f\|_{L^{p}(\mathbb{R}^{N})} < \varepsilon$ holds for all $f \in \mathcal{F}$. Then the closure of $F |_{\Omega}$ in $L^{p}(\mathbb{R}^{N})$ is compact for any measurable set $\Omega \subseteq \mathbb{R}^{N}$	with finite measure. Here $F |_{\Omega}$ denotes the restrictions to $\Omega$ of the functions in $\mathcal{F}$.
\end{theorem}

\begin{definition}\label{def-Gamma-convergence}
	($\Gamma$-convergence)
	\cite[Prop.8.1]{dal1993introduction}
	Let $X$ be a topological space satisfying the first axiom of countability, $\{F_{k}\}_{k=1}^{+\infty}$ be a sequence of functionals such that for every $k \in \mathbb{N}_{+}$, $F_{k}: X \to \overline{\mathbb{R}} := \mathbb{R} \cup \{+\infty\}$. 
	Then $\{F_{k}\}_{k=1}^{+\infty}$ $\Gamma$-converges to the functional $F: X \to \mathbb{R}$ if and only if the following conditions hold:
	
	(a) For every $x \in X$ and every sequence $\{ x^{k} \}_{k=1}^{+\infty} \subset X$ converging to $x$ as $k \to +\infty$, there is $F(x) \leq \mathop{\lim\inf} \limits_{k \to +\infty} F_{k}(x_{k})$.
	
	(b) For every $x \in X$, there exists a sequence $\{ x^{k} \}_{k=1}^{+\infty}$ converging to $x$ as $k \to +\infty$ such that $F(x) \geq \mathop{\lim\sup} \limits_{k \to +\infty} F_{k}(x_{k})$.
\end{definition}

\begin{theorem}\label{thm-fundamental-theorem-of-Gamma-convergence}
	(The fundamental theorem of $\Gamma$-convergence)	
	\cite[Cor 7.20]{dal1993introduction} \cite[Thm 3.3]{matias2015homogenization}
	Let $(X,d)$ be a metric space, $\{F_{\varepsilon}\}_{\varepsilon}$ be a sequence of functionals defined on $X$, $F := \Gamma\text{-}\lim\limits_{\varepsilon \to 0^{+}}F_{\varepsilon}$ be a functional defined on $X$. 
	For $\forall \varepsilon>0$, let $(x^{\varepsilon})^{*}$ be a minimizer of $F_{\varepsilon}$ in $X$.
	
	(1) If $x^{*}$ is a cluster point of $\{(x^{\varepsilon})^{*}\}_{\varepsilon}$, then $x^{*}$ is a minimizer of $F$, and $F(x^{*}) = \mathop{\lim\sup}\limits_{\varepsilon\to0^{+}}F_{\varepsilon}((x^{\varepsilon})^{*})$.
	
	(2) If $\{(x^{\varepsilon})^{*}\}_{\varepsilon}$ converges to a point $x^{*}$ in $X$, then $x^{*}$ is a minimizer of $F$, and $F(x^{*}) = \lim\limits_{\varepsilon\to0^{+}}F_{\varepsilon}((x^{\varepsilon})^{*})$.
\end{theorem}

\subsection{Proof of Theorem \ref{thm-stability-control-problem-discrete-all}} \label{subappendix-proof-stability-control-problem-discrete-all}

\begin{proof}[Proof. ]
	For clarity, we denote $\{ x^{N,k}(A^{N,:}, \alpha^{N,:}, \lambda^{N,:}; x^{0}, b) \}_{k=0}^{N}$ as the state of the network \eqref{eq-system-inclusion-discrete} determined by the parameters $(A^{N,:}, \alpha^{N,:}, \lambda^{N,:})$ with the initial value $x^{0}$ and the observed  data $b$.
	
	\textbf{Step 1:} We first prove that for any given $N \in \mathbb{N}_{+}$, 
	\begin{align}\label{eq-lim-sup-discrete}
		\lim_{r \to +\infty} \sup_{(A^{N,:}, \alpha^{N,:}, \lambda^{N,:}) \in \mathscr{D}_{N}} \big| \widetilde{\mathcal{J}}_{N}^{(r)}(A^{N,:}, \alpha^{N,:}, \lambda^{N,:}) - \mathcal{J}_{N}(A^{N,:}, \alpha^{N,:}, \lambda^{N,:}) \big| = 0.
	\end{align}
	
	Since $\mathscr{D}_{N}$ is compact, \textbf{Stability} \textbf{Theorem} 4.1 in \cite{lin2025deep} gives that \begin{align*}
		& \cup_{r=1}^{+\infty} \{ x^{N,N}(A^{N,:}, \alpha^{N,:}, \lambda^{N,:}; (x^{0})^{(r)}, b^{(r)}): (A^{N,:}, \alpha^{N,:}, \lambda^{N,:}) \in \mathscr{D}_{N} \}  \\
		\bigcup& \{ x^{N,N}(A^{N,:}, \alpha^{N,:}, \lambda^{N,:}; x^{0}, b): (A^{N,:}, \alpha^{N,:}, \lambda^{N,:}) \in \mathscr{D}_{N} \}
	\end{align*}
	is a subset of a compact set 
	in $\mathbb{R}^{n}$. 
	It follows from \textbf{Stability} \textbf{Theorem} 4.1 in \cite{lin2025deep} that, there exist two constants $\widetilde{c}_{1} \geq 0, \widetilde{c}_{2} \geq 0$ both independent of $r$, $A^{N,:}$, $\alpha^{N,:}$, $\lambda^{N,:}$, such that
	\begin{align*}
		& \|x^{N,N}(A^{N,:}, \alpha^{N,:}, \lambda^{N,:}; (x^{0})^{(r)}, b^{(r)}) - x^{N,N}(A^{N,:}, \alpha^{N,:}, \lambda^{N,:}; x^{0}, b)\|_{2}  \\ 
		\leq& \widetilde{c}_{1} \|(x^{0})^{(r)} - x^{0}\|_{2}
		+ \widetilde{c}_{2} \|b^{(r)} - b\|_{2}.  
	\end{align*}
	Therefore, for any $\delta > 0$, there exists $r_{1,\delta} \in \mathbb{N}_{+}$ such that for every integer $r \geq r_{1,\delta}$,  
	\begin{align*}
		\sup_{(A^{N,:}, \alpha^{N,:}, \lambda^{N,:}) \in \mathscr{D}_{N}} \|x^{N,N}(A^{N,:}, \alpha^{N,:}, \lambda^{N,:}; (x^{0})^{(r)}, b^{(r)}) - x^{N,N}(A^{N,:}, \alpha^{N,:}, \lambda^{N,:}; x^{0}, b)\|_{2} 
		\leq \frac{1}{3} \delta.
	\end{align*}
	Besides, $\lim_{r \to +\infty} y^{(r)} = y$ implies that for any $\delta > 0$, there exists $r_{2,\delta} \in \mathbb{N}_{+}$ such that for every integer $r \geq r_{2,\delta}$, we have
	\begin{align*}
		\|y^{(r)} - y\|_{2} \leq \frac{1}{3} \delta.
	\end{align*}
	These results indicates that for any $\delta > 0$, there exists $r_{3,\delta} := \max\{ r_{1,\delta}, r_{2,\delta} \}$ such that for every integer $r \geq r_{3,\delta}$, 
	\begin{align*}
		& \sup_{(A^{N,:}, \alpha^{N,:}, \lambda^{N,:}) \in \mathscr{D}_{N}} \|x^{N,N}(A^{N,:}, \alpha^{N,:}, \lambda^{N,:}; (x^{0})^{(r)}, b^{(r)}) - x^{N,N}(A^{N,:}, \alpha^{N,:}, \lambda^{N,:}; x^{0}, b)\|_{2} + \|y^{(r)} - y\|_{2} \\
		&\leq \frac{2}{3} \delta
		< \delta.
	\end{align*}
	By further noting the continuity assumption of $\mathcal{L}$ in (A4) and thus its uniform continuity on compact sets, we obtain that for any $\varepsilon > 0$, there exists an integer $r_{3,\delta(\varepsilon)}$ such that for every integer $r \geq r_{3,\delta(\varepsilon)}$, 
	\begin{align*}
		& \sup_{(A^{N,:}, \alpha^{N,:}, \lambda^{N,:}) \in \mathscr{D}_{N}} |\widetilde{\mathcal{J}}_{N}^{(r)}(A^{N,:}, \alpha^{N,:}, \lambda^{N,:}) - \mathcal{J}_{N}(A^{N,:}, \alpha^{N,:}, \lambda^{N,:})|  \\
		=& \sup_{(A^{N,:}, \alpha^{N,:}, \lambda^{N,:}) \in \mathscr{D}_{N}} \Big| \mathcal{L}(x^{N,N}(A^{N,:}, \alpha^{N,:}, \lambda^{N,:}; (x^{0})^{(r)}, b^{(r)}); y^{(r)})  \\
		&\hspace{75pt}+ \beta_{1} \mathcal{H}_{N}^{(1)}(A^{N,:})
		+ \beta_{2} \mathcal{H}_{N}^{(2)}(\alpha^{N,:})
		+ \beta_{3} \mathcal{H}_{N}^{(3)}(\lambda^{N,:})  \Big.  \\
		&\hspace{75pt}\Big.
		- \mathcal{L}(x^{N,N}(A^{N,:}, \alpha^{N,:}, \lambda^{N,:}; x^{0}, b); y)  \\
		&\hspace{75pt}- \beta_{1} \mathcal{H}_{N}^{(1)}(A^{N,:})
		- \beta_{2} \mathcal{H}_{N}^{(2)}(\alpha^{N,:})
		- \beta_{3} \mathcal{H}_{N}^{(3)}(\lambda^{N,:})
		\Big|  \\
		=& \sup_{(A^{N,:}, \alpha^{N,:}, \lambda^{N,:}) \in \mathscr{D}_{N}} \left| \mathcal{L}(x^{N,N}(A^{N,:}, \alpha^{N,:}, \lambda^{N,:}; (x^{0})^{(r)}, b^{(r)}); y^{(r)})  
		- \mathcal{L}(x^{N,N}(A^{N,:}, \alpha^{N,:}, \lambda^{N,:}; x^{0}, b); y) 
		\right|  \\
		<& \varepsilon,
	\end{align*} 
	which indicates Eq.\eqref{eq-lim-sup-discrete}.
	
	\textbf{Step 2:} In this step, we will prove (1).
	The existence of a cluster point of $\{ ((A^{N,:})^{*,r}, (\alpha^{N,:})^{*,r}, (\lambda^{N,:})^{*,r}) \}_{r=1}^{+\infty}$ is obvious due to the compactness of $\mathscr{D}_{N}$.
	We next show that its every cluster point belongs to $S_{N}$.
	Suppose that $\{ ((A^{N,:})^{*,r_{i}}, (\alpha^{N,:})^{*,r_{i}}, (\lambda^{N,:})^{*,r_{i}}) \}_{i=1}^{+\infty}$ converges to $((A^{N,:})^{*,*}, (\alpha^{N,:})^{*,*}, (\lambda^{N,:})^{*,*})$.
	By using again the continuity of $\mathcal{L}$, we obtain
	\begin{align}\label{eq-study-discrete-all}
		& \mathcal{J}_{N}((A^{N,:})^{*,*}, (\alpha^{N,:})^{*,*}, (\lambda^{N,:})^{*,*})  \notag \\
		=& \mathcal{L}(x^{N,N}((A^{N,:})^{*,*}, (\alpha^{N,:})^{*,*}, (\lambda^{N,:})^{*,*}; x^{0}, b); y)
		+ \beta_{1} \mathcal{H}_{N}^{(1)}((A^{N,:})^{*,*})  \notag \\
		&+ \beta_{2} \mathcal{H}_{N}^{(2)}((\alpha^{N,:})^{*,*})
		+ \beta_{3} \mathcal{H}_{N}^{(3)}((\lambda^{N,:})^{*,*})
		\notag \\
		=& \lim\limits_{i \to +\infty} \Big[ \mathcal{L}(x^{N,N}((A^{N,:})^{*,(r_{i})}, (\alpha^{N,:})^{*,(r_{i})}, (\lambda^{N,:})^{*,(r_{i})}; x^{0}, b); y)   \notag \\
		&\hspace{30pt}+ \beta_{1} \mathcal{H}_{N}^{(1)}((A^{N,:})^{*,(r_{i})})
		+ \beta_{2} \mathcal{H}_{N}^{(2)}((\alpha^{N,:})^{*,(r_{i})})
		+ \beta_{3} \mathcal{H}_{N}^{(3)}((\lambda^{N,:})^{*,(r_{i})})
		\Big]  \notag \\
		=& \lim\limits_{i \to +\infty} \Big[ \mathcal{L}(x^{N,N}((A^{N,:})^{*,(r_{i})}, (\alpha^{N,:})^{*,(r_{i})}, (\lambda^{N,:})^{*,(r_{i})}; (x^{0})^{(r_{i})}, b^{(r_{i})}); y^{(r_{i})})   \notag \\
		&\hspace{30pt}+ \beta_{1} \mathcal{H}_{N}^{(1)}((A^{N,:})^{*,(r_{i})})
		+ \beta_{2} \mathcal{H}_{N}^{(2)}((\alpha^{N,:})^{*,(r_{i})})
		+ \beta_{3} \mathcal{H}_{N}^{(3)}((\lambda^{N,:})^{*,(r_{i})})
		\Big]  \notag \\
		\leq& \lim\limits_{i \to +\infty} \Big[ \mathcal{L}(x^{N,N}(A^{N,:}, \alpha^{N,:}, \lambda^{N,:}; (x^{0})^{(r_{i})}, b^{(r_{i})}); y^{(r_{i})})   
		+ \beta_{1} \mathcal{H}_{N}^{(1)}(A^{N,:})  \notag \\
		&\hspace{30pt}+ \beta_{2} \mathcal{H}_{N}^{(2)}(\alpha^{N,:})
		+ \beta_{3} \mathcal{H}_{N}^{(3)}(\lambda^{N,:})
		\Big]  \notag \\
		=& \mathcal{L}(x^{N,N}(A^{N,:}, \alpha^{N,:}, \lambda^{N,:}; x^{0}, b); y)   
		+ \beta_{1} \mathcal{H}_{N}^{(1)}(A^{N,:})
		+ \beta_{2} \mathcal{H}_{N}^{(2)}(\alpha^{N,:})
		+ \beta_{3} \mathcal{H}_{N}^{(3)}(\lambda^{N,:}), \notag \\ &\hspace{20pt} \forall (A^{N,:}, \alpha^{N,:}, \lambda^{N,:}) \in \mathscr{D}_{N}   \notag \\
		=& \mathcal{J}_{N}(A^{N,:}, \alpha^{N,:}, \lambda^{N,:}), \quad \forall (A^{N,:}, \alpha^{N,:}, \lambda^{N,:}) \in \mathscr{D}_{N},	
	\end{align}
	where the third and the fifth equalities are due to the result in \textbf{Step 1}, and the fourth inequality is from $((A^{N,:})^{*,(r_{i})}, (\alpha^{N,:})^{*,(r_{i})}, (\lambda^{N,:})^{*,(r_{i})}) \in \widetilde{S}_{N}^{(r_{i})}$. 
	Obviously, Eq.\eqref{eq-study-discrete-all} indicates $((A^{N,:})^{*,*}, (\alpha^{N,:})^{*,*}, (\lambda^{N,:})^{*,*}) \in S_{N}$, and 
	\begin{align*}
		\mathcal{J}_{N}((A^{N,:})^{*,*}, (\alpha^{N,:})^{*,*}, (\lambda^{N,:})^{*,*}) 
		=& \lim\limits_{i \to +\infty} \widetilde{\mathcal{J}}_{N}^{(r_{i})} ( (A^{N,:})^{*,(r_{i})}, (\alpha^{N,:})^{*,(r_{i})}, (\lambda^{N,:})^{*,(r_{i})} )  \\
		\leq& \inf_{(A^{N,:}, \alpha^{N,:}, \lambda^{N,:}) \in \mathscr{D}_{N}} \mathcal{J}_{N}(A^{N,:}, \alpha^{N,:}, \lambda^{N,:})  \\
		=& \mathcal{J}_{N}((A^{N,:})^{*,*}, (\alpha^{N,:})^{*,*}, (\lambda^{N,:})^{*,*}).
	\end{align*}
	That is, 
	\begin{align}\label{eq-perturbed-optimal-value-converge-discrete}
		\lim\limits_{i \to +\infty} \widetilde{\mathcal{J}}_{N}^{(r_{i})} ( (A^{N,:})^{*,(r_{i})}, (\alpha^{N,:})^{*,(r_{i})}, (\lambda^{N,:})^{*,(r_{i})} )
		= \inf_{(A^{N,:}, \alpha^{N,:}, \lambda^{N,:}) \in \mathscr{D}_{N}} \mathcal{J}_{N}(A^{N,:}, \alpha^{N,:}, \lambda^{N,:}).
	\end{align}
	
	\textbf{Step 3:} Now we prove (2) and (3). 
	For (2), we assume by contradiction that 
	\begin{align*}
		& \inf\limits_{ ( (A^{N,:})^{*}, (\alpha^{N,:})^{*}, (\lambda^{N,:})^{*} ) \in S_{N}} \Big\| ( (A^{N,:})^{*,(r)}, (\alpha^{N,:})^{*,(r)}, (\lambda^{N,:})^{*,(r)} ) \\
		&\hspace{100pt} - ( (A^{N,:})^{*}, (\alpha^{N,:})^{*}, (\lambda^{N,:})^{*} ) \Big\|_{\ell^{p}((\mathbb{R}^{m \times n})^{N} \times \mathbb{R}^{N} \times \mathbb{R}^{N})} \nrightarrow 0.
	\end{align*} 
	Then there exist an $\varepsilon_{0} > 0$ and a subsequence
	\begin{align*}
		& \Big\{ \inf\limits_{ ( (A^{N,:})^{*}, (\alpha^{N,:})^{*}, (\lambda^{N,:})^{*} ) \in S_{N}} \Big\| ( (A^{N,:})^{*,(r_{i})}, (\alpha^{N,:})^{*,(r_{i})}, (\lambda^{N,:})^{*,(r_{i})} )  \\
		&\hspace{105pt} - ( (A^{N,:})^{*}, (\alpha^{N,:})^{*}, (\lambda^{N,:})^{*} ) \Big\|_{\ell^{p}((\mathbb{R}^{m \times n})^{N} \times \mathbb{R}^{N} \times \mathbb{R}^{N})} \Big\}_{i=1}^{+\infty}
	\end{align*}
	such that for every $i \in \mathbb{N}_{+}$, 
	\begin{align}\label{eq-approx-all-0}
		& \inf_{ ( (A^{N,:})^{*}, (\alpha^{N,:})^{*}, (\lambda^{N,:})^{*} ) \in S_{N}} \Big\| ( (A^{N,:})^{*,(r_{i})}, (\alpha^{N,:})^{*,(r_{i})}, (\lambda^{N,:})^{*,(r_{i})} )  \notag \\
		&\hspace{100pt}- ( (A^{N,:})^{*}, (\alpha^{N,:})^{*}, (\lambda^{N,:})^{*} ) \Big\|_{\ell^{p}((\mathbb{R}^{m \times n})^{N} \times \mathbb{R}^{N} \times \mathbb{R}^{N})} 
		\geq \varepsilon_{0}.
	\end{align} 
	By \textbf{Step 2}, we can further extract a convergent subsequence $\left\{ \left( (A^{N,:})^{*,(r_{i_{j}})}, (\alpha^{N,:})^{*,(r_{i_{j}})}, (\lambda^{N,:})^{*,(r_{i_{j}})} \right) \right\}_{j=1}^{+\infty}$ from $\left\{ \left( (A^{N,:})^{*,(r_{i})}, (\alpha^{N,:})^{*,(r_{i})}, (\lambda^{N,:})^{*,(r_{i})} \right) \right\}_{i=1}^{+\infty}$ with the limit point belonging to $S_{N}$ to get a contradiction with Eq.\eqref{eq-approx-all-0}. 
	A similar argument and noting Eq.\eqref{eq-perturbed-optimal-value-converge-discrete} give (3).	
\end{proof}

\end{appendices}

\end{sloppypar}
\end{document}